%% file: manuscript_report.tex
\documentclass[letterpaper, 10 pt, conference]{ieeeconf}  

\IEEEoverridecommandlockouts                              

\input{packages}
\input{commands}

\input{notation}

\title{\Large \bf
Adaptive Dual-Headway Unicycle Pose Control and Motion Prediction \\ for Optimal Sampling-Based Feedback Motion Planning
}

\author{Aykut \.{I}\c{s}leyen, Abhidnya Kadu, Ren\'{e} van de Molengraft, \"{O}m\"{u}r Arslan
\thanks{The authors are with the Department of Mechanical Engineering, Eindhoven University of Technology, P.O. Box 513, 5600 MB Eindhoven, The Netherlands, and also affiliated with the Eindhoven AI Systems Institute. Emails:  \{a.isleyen, a.a.kadu, m.j.g.v.d.molengraft, o.arslan\}@tue.nl} %
}

\begin{document}

\maketitle
\thispagestyle{empty}
\pagestyle{empty}

\begin{abstract}
Safe, smooth, and optimal motion planning for nonholonomically constrained mobile robots and autonomous vehicles is essential for achieving reliable, seamless, and efficient autonomy in logistics, mobility, and service industries.
In many such application settings, nonholonomic robots, like unicycles with restricted motion, require precise planning and control of both translational and orientational motion to approach specific locations in a designated orientation, such as for approaching changing, parking, and loading areas.
In this paper, we introduce a new dual-headway unicycle pose control method by leveraging an adaptively placed headway point in front of the unicycle pose and a tailway point behind the goal pose. 
In summary, the unicycle robot continuously follows its headway point, which chases the tailway point of the goal pose and the asymptotic motion of the tailway point towards the goal position guides the unicycle robot to approach the goal location with the correct orientation.  
The simple and intuitive geometric construction of dual-headway unicycle pose control enables an explicit convex feedback motion prediction bound on the closed-loop unicycle motion trajectory for fast and accurate safety verification. 
We present an application of dual-headway unicycle control for optimal sampling-based motion planning around obstacles.
In numerical simulations, we show that optimal unicycle motion planning using dual-headway translation and orientation distances significantly outperforms Euclidean translation and cosine orientation distances in generating smooth motion with minimal travel and turning~effort.    
\end{abstract}

\section{Introduction}
\label{sec.introduction}

Autonomous mobile robots provide flexible automation solutions in various application settings, ranging from assisting people with daily activities (e.g., service robots \cite{kim_etal_RAM2009, jones_MRA2006}) to enhancing transportation and mobility systems (e.g., warehouse robots \cite{ackerman_Spectrum2022, renan_nascimento_RAS2021} and self-driving vehicles \cite{gonzalez_perez_milanes_mashashibi_TITS2016, paden_cap_yong_yershov_frazzoli_TIV2016}).
To safely and smoothly perform diverse tasks in complex environments, including people \cite{philippsen_siegwart_ICRA2003} and other mobile robots \cite{snape_etal_IROS2010}, nonholonomic mobile robots, such as unicycles with restricted turning or no sideways movement, require effective planning and control of both translational and orientational motion to approach specific locations in a designated orientation, for example, when approaching charging and loading areas. 
Control-aware planning with feedback motion prediction play a key role in safe and smooth autonomous motion design around obstacles \cite{chakravarthy_debasish_TSM1998, fiorini_shiller_IJRR1998, arslan_koditschek_IJRR2019, isleyen_arslan_RAL2022}.

In this paper, we design a new geometric unicycle pose control approach using headway and tailway points of unicycle poses in the front and back, respectively, to continuously guide a nonholonomically constrained mobile robot to approach a given goal position with a desired orientation, as shown in \reffig{fig.dual_headway_control_example} (left). 
The simple and intuitive geometric construction of dual-headway unicycle control allows for accurately and explicitly estimating the required spatial region for executing feedback motion using the convex hull of the start and goal positions and their respective headway and tailway points, as well as designing heuristics for measuring turning and travel costs informatively. 
Accordingly, we apply the newly developed dual-headway unicycle motion control, prediction, and translational and orientational distances for optimal sampling-based feedback motion planning, minimizing total travel and turning effort, as shown in \reffig{fig.dual_headway_control_example} (right).


\begin{figure}[t]
\centering
\begin{tabular}{@{}c@{\hspace{1mm}}c@{}}
\includegraphics[width = 0.5\columnwidth]{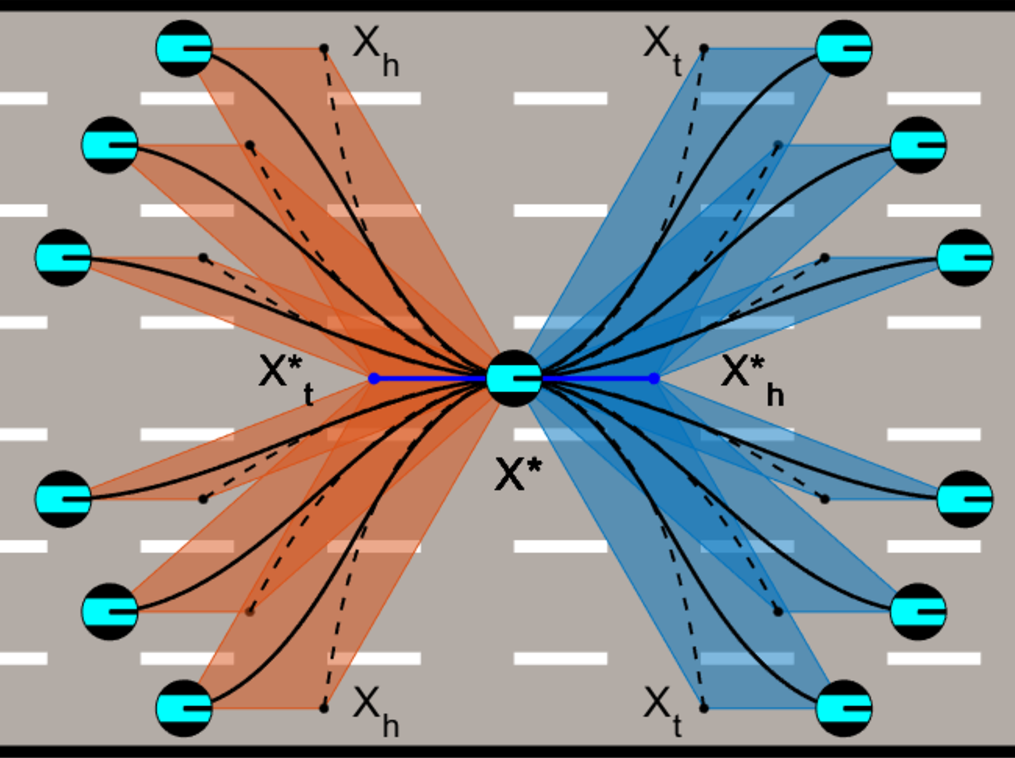} & \includegraphics[width = 0.5\columnwidth]{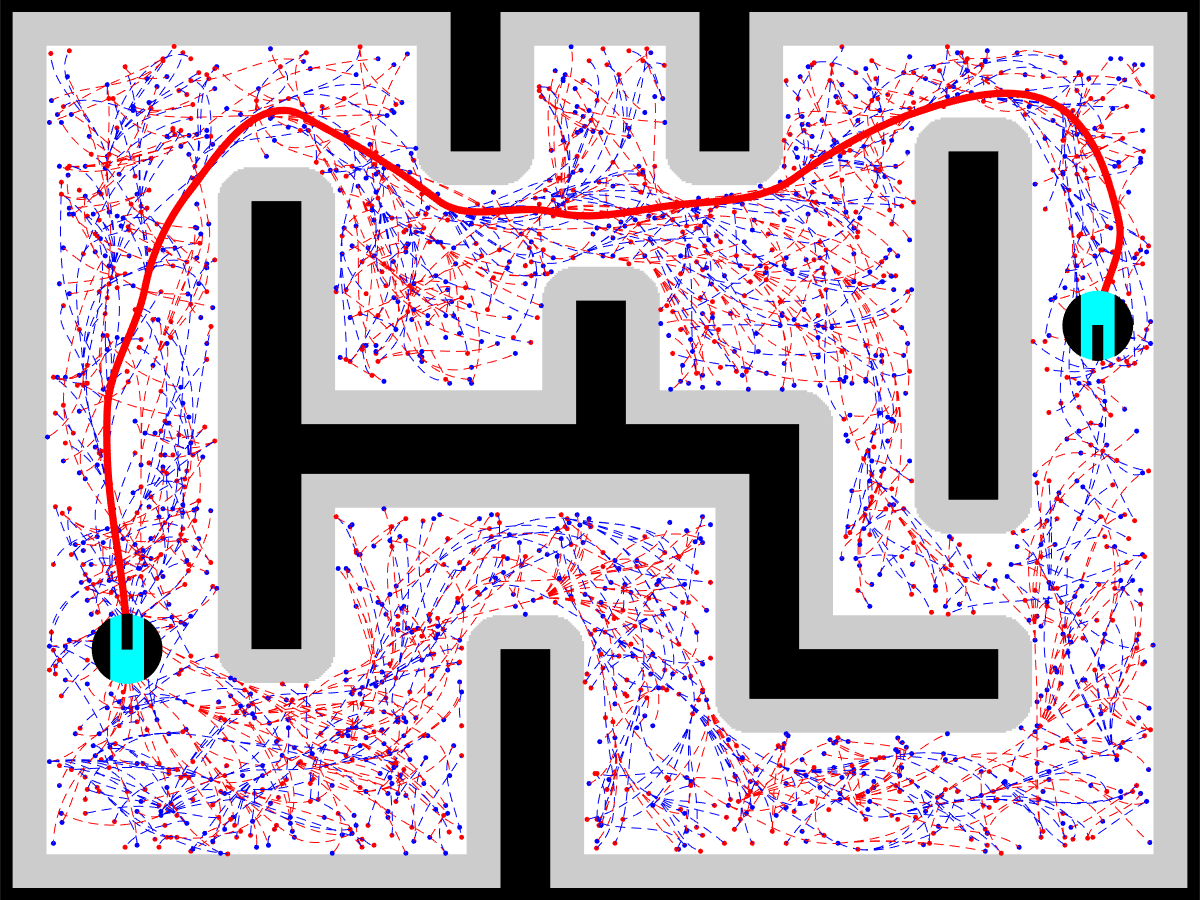}
\end{tabular}
\vspace{-3mm}
\caption{Adaptive dual-headway unicycle pose control continuously moves the headway point of the unicycle pose, which is in front of the unicycle, toward the tailway point of the goal pose, located behind the goal. 
(left) Example closed-loop trajectories of dual-headway unicycle control for a smooth lane-changing scenario with associated motion prediction bounds. 
(right) Example application of dual-headway unicycle control and motion prediction for optimal sampling-based motion planning around obstacles, minimizing the total dual-headway translation and orientation distance.
}
\label{fig.dual_headway_control_example}
\vspace{-3mm}
\end{figure}

\subsection{Motivation and Related Literature}

\subsubsection{Unicycle Motion Control}

Due to its nonholonomic constraint of no sideways motion, it is well known that global smooth, continuous, and autonomous (i.e., time-invariant) control of unicycle dynamics to reach any given position is not possible \cite{brockett_DGCT1983}. 
This limitation makes unicycle pose control even more challenging when attempting to continuously reach a specific position and orientation.
Fortunately, non-smooth, discontinuous control is only needed to break the symmetry, for example, when the goal position is directly behind the unicycle with an opposite orientation.
Accordingly, various unicycle control methods with minor discontinuities, such as inner-outer loop approaches \cite{astolfi_JDSMC1999}, angular feedback linearization \cite{lee_etal_IROS2000}, and full state feedback linearization \cite{dandrea-novel_campion_bastin_IJRR1995}, have been developed to asymptotically move (almost) all unicycle poses to a given destination position.
A common feature of these standard unicycle motion control methods is that they all allow explicit, positively inclusive convex feedback motion prediction bounds on the resulting closed-loop unicycle motion trajectory, providing more accurate and less conservative safety verification around obstacles than invariant Lyapunov level sets \cite{isleyen_vandewouw_arslan_IROS2023, isleyen_vandewouw_arslan_CDC2023, tarshahani_isleyen_arslan_ECC2024}.
Polar coordinates enable Lyapunov-based unicycle pose control methods to guide a unicycle to a target position and orientation in an empty environment without obstacles \cite{aicardi_etal_RAM1995, astolfi_JDSMC1999}; however, the nonlinear polar transformation and the complex shapes of Lyapunov level sets make it challenging to derive an accurate and simple geometric trajectory bound for safety verification and safe motion control around obstacles.  
Time-varying nonlinear control, e.g., using sinusoids, has also been applied for unicycle pose control, but, systematically ensuring safety around obstacles can be more complicated due to the time-varying, open-loop nature of such approaches \cite{murray_sastry_TAC1993, samson_IJRR1993, samson_TAC1995}.
In this paper, we propose a new time-invariant geometric unicycle pose control approach that asymptotically moves the headway point of the unicycle pose to the tailway point of the goal pose using state feedback linearization. 
Its simple geometric construction enables an explicit, positively inclusive convex motion prediction bound on the  closed-loop unicycle motion trajectory for efficient safety assessment, using the convex hull of the unicycle position, goal position, and their respective headway and tailway points.    

\subsubsection{Unicycle Motion Planning}

Motion planning for nonholonomic and kinodynamic nonlinear systems is known to be computationally challenging \cite{canny_ComplexityMotionPlanning1988}. 
Consequently, sampling-based motion planning algorithms, using random control searches or motion primitives, are often applied to find safe and smooth paths for unicycle-like nonholonomic systems around obstacles \cite{lavalle_kuffner_IJRR2001, karaman_frazzoli_IJRR2011}. 
However, solving the two-point boundary problem, e.g., bringing any unicycle position and orientation to a specific goal position and orientation, through random control and motion primitive search is usually computationally difficult and inefficient.
This challenge is generally addressed by tightly integrating control and planning, enabling sample unicycle poses to be connected via closed-loop trajectory predictions \cite{kuwata_teo_fiore_karaman_frazzoli_how_TCST2009, palmieri_arras_IROS2014, park_kuipers_IROS2015, arslan_berntorp_tsiotras_ICRA2017}, based on forward simulation of unicycle dynamics under a unicycle pose control policy \cite{aicardi_etal_RAM1995, astolfi_JDSMC1999}. 
Although forward-simulated closed-loop predictions leverage control for effectively generating local steering strategies, safety verification still remains computationally expensive due to the nonparametric, dense multi-point representation of closed-loop trajectories.
To effectively utilize control for both establishing local connectivity and ensuring safety in motion planning, geometrically simple over-approximations of positively invariant Lyapunov function sets are often used for conservative yet efficient safety and collision checking \cite{danielson_berntorp_cairano_weiss_ACC2020}. 
Such control Lyapunov functions also serve as local cost measures in motion planning to quantify connection difficulty between unicycle poses \cite{park_kuipers_IROS2015, danielson_berntorp_cairano_weiss_ACC2020}, as an alternative to standard additively weighted Euclidean translation and cosine orientation distances \cite{palmieri_arras_IROS2014}.
In this paper, aligned with such uses of closed-loop prediction and Lyapunov functions in the literature, we demonstrate an application of the newly proposed dual-headway pose control, its explicit feedback motion prediction, and the associated unicycle pose distance for optimal sampling-based feedback motion planning.

\subsection{Contributions and Organization of the Paper}

This paper introduces a new dual-headway unicycle pose control and motion prediction approach for optimal sampling-based feedback motion planning, enabling safe and smooth motion around obstacles while reaching a desired unicycle pose with minimal travel and turning effort.
In summary, the major contributions of our paper are as follows:

\indent $\bullet$ We propose a dual-headway unicycle pose control approach that asymptotically brings almost all unicycle poses to any given goal pose using state feedback linearization.

\indent $\bullet$ We show the closed-loop motion under dual-headway control is bounded by the convex hull of the unicycle position, goal position, and their headway and tailway points.  

\indent $\bullet$ We apply dual-headway unicycle control and motion prediction for optimal sampling-based feedback motion planning with minimal travel and turning effort.

\noindent We demonstrate the effectiveness of our integrated planning and control approach in numerical simulations by showing that minimizing dual-headway translation and orientation distances significantly outperforms standard Euclidean translation and cosine orientation distances, generating smoother motion with minimal travel and turning effort.

The rest of the paper is organized as follows. 
\refsec{sec.adaptive_dual_headway_unicycle_control} presents the dual-headway unicycle pose control and motion prediction.
\refsec{sec.optimal_unicycle_feedback_motion_planning} describes how to perform optimal unicycle feedback motion planning using forward and backward dual-headway motion control primitives and unicycle pose distances.
\refsec{sec.numerical_simulations} demonstrates numerical simulation examples. 
\refsec{sec.conclusions} concludes with a summary of our work and future research directions.

\section{Adaptive Dual-Headway Unicycle Control}
\label{sec.adaptive_dual_headway_unicycle_control}

\subsection{Kinematic Unicycle Robot Model}

We consider a kinematic unicycle robot whose state is represented by its 2D position $\pos \in \R^2$ and forward orientation angle $\ori \in [ -\pi, \pi )$ that is measured in radians counterclockwise from the horizontal axis.
The equations of motion of the kinematic unicycle robot model are given by
\begin{align} \label{eq.UnicycleDynamics}
\dot{\pos} = \linvel \ovect{\ori} \quad \text{and} \quad \dot{\ori} = \angvel 
\end{align}
where  $\linvel \in \R$ and $\angvel \in \R$ are scalar control inputs specifying the linear and angular velocities, respectively.
Note that the unicycle dynamics are underactuated, with the nonholonomic constraint of no sideways motion, i.e., $\nvecTsmall{\ori}\dot{\pos} = 0$.

\subsection{Unicycle Pose Control via Headway and Tailway Points}
\label{sec.dual_headway_unicycle_pose_control}

Unicycle headway control is a standard full feedback linearization approach used to bring a unicycle robot to a desired goal position by utilizing a headway (also known as offset) point that is located a certain (e.g., fixed or varying) distance in front of the robot and follows simple linear reference dynamics to asymptotically converge to the goal position \cite{isleyen_vandewouw_arslan_CDC2023}.
Although global convergence to the goal position is ensured, standard unicycle headway position control methods pay little or no attention to controlling the final approach orientation of the unicycle. 
Inspired by the standard unicycle headway control methods, to reach a goal position $\goalpos \in \R^{2}$ with a specified goal approach orientation $\goalori \in [-\pi, \pi)$, we consider dual headway points at varying distances: one headway point in front of the current unicycle pose and one tailway point behind the goal pose, as
\begin{subequations}\label{eq.headway_tailway_points}
\begin{align}
\headpos &:= \pos + \headcoef \norm{\pos - \goalpos} \ovectsmall{\ori}
\\
\goaltailpos &:= \goalpos - \goaltailcoef \norm{\pos - \goalpos} \ovectsmall{\goalori}
\end{align}
\end{subequations}
where $\headcoef > 0$ and $\goaltailcoef > 0$ are constant positive headway and tailway coefficients, respectively.
One can observe that the time rate of change of the headway and tailway points under the unicycle control dynamics in \refeq{eq.UnicycleDynamics} are given by
{
\begin{align*}
\headposdot &=   \!\plist{\!1\! + \headcoef \tfrac{\tr{(\pos - \goalpos)\!}}{\norm{\pos-\goalpos}} \ovectsmall{\ori}\!} \! \ovectsmall{\ori} \linvel   + \headcoef \norm{\pos \!-\! \goalpos\!}  \nvectsmall{\ori} \angvel,
\\
\goaltailposdot &= -  \goaltailcoef \tfrac{\tr{(\pos - \goalpos)}}{\norm{\pos-\goalpos}} \ovectsmall{\ori} \ovectsmall{\goalori} \linvel.
\end{align*}
}%
Using the first-order proportional error feedback as a reference headway-point dynamics towards the tailway point as
\begin{align}\label{eq.headway_reference_dynamics}
\headposdot = - \refcoef \plist{\headpos - \goaltailpos}
\end{align}
we design a dual-headway unicycle motion controller, denoted by $\ctrl_{\goalpos, \goalori}(\pos, \ori) \!=\! \plist{\linvel_{\goalpos, \goalori}(\pos, \ori), \angvel_{\goalpos, \goalori}(\pos, \ori)}$ that determines the linear velocity $\linvel_{\goalpos,\goalori}(\pos, \ori)$ and the angular velocity $\angvel_{\goalpos, \goalori}(\pos, \ori)$ for the unicycle robot model in \refeq{eq.UnicycleDynamics} as
\begin{subequations}\label{eq.dual_headway_control}
\begin{align}
\linvel_{\goalpos, \goalori}(\pos, \ori) & = \frac{-\refcoef\tr{\plist{\headpos - \goaltailpos}} \ovectsmall{\ori}}{1 + \headcoef\tfrac{\tr{\pos - \goalpos}}{\norm{\pos - \goalpos}} \ovectsmall{\ori}},
\\
\angvel_{\goalpos, \goalori}(\pos, \ori) &= \frac{-\refcoef\tr{\plist{\headpos - \goaltailpos}} \nvectsmall{\ori}}{ \headcoef\norm{\pos - \goalpos}},
\end{align}
\end{subequations}
where $\refcoef \!>\! 0$ is a fixed positive control gain for the reference dynamics in \refeq{eq.headway_reference_dynamics}, and $\headcoef > 0$ is the headway coefficient in~\refeq{eq.headway_tailway_points}.

\begin{lemma}\label{lem.headway_tailway_distance}
\emph{(Headway-Tailway Distance)} For any $\headcoef + \goaltailcoef < 1$, the distance between the headway and tailway points bounds the distance between the current and goal positions as 
\begin{align}
\norm{\pos - \goalpos} \leq \tfrac{1}{1 - \headcoef - \goaltailcoef} \norm{\headpos - \goaltailpos}.
\end{align}
\end{lemma}
\begin{proof}
See \refapp{app.lem.headway_tailway_distance}.
\end{proof}

Hence, if \mbox{$\headcoef + \goaltailcoef < 1$}, an asymptotically decreasing headway–tailway distance implies an asymptotically decreasing distance to the goal, as observed below.

\begin{proposition}\label{prop.headway_tailway_distance_decay}
\emph{(Headway-Tailway Distance Decay)} For $\headcoef + \goaltailcoef < 1$,  the dual-headway unicycle controller in \refeq{eq.dual_headway_control} ensures that the distance between the headway and tailway points are strictly decreasing away from the goal, i.e., for any $(\pos, \ori), (\goalpos, \goalori) \in \R^2 \! \times \! [-\pi, \pi)$ with $\pos \neq \goalpos$, one has 
{\small
\begin{align*}
\frac{d}{dt}\norm{\headpos - \goaltailpos}^2 &\leq -2 \refcoef \norm{\headpos - \goaltailpos}^2 \plist{\!1 - \tfrac{\goaltailcoef \absval{\tfrac{\tr{(\goalpos - \pos)}}{\norm{\goalpos - \pos}}\ovectsmall{\ori}}}{1 - \headcoef \tfrac{\tr{(\goalpos - \pos)}}{\norm{\goalpos - \pos}}\ovectsmall{\ori} }\!} 
\end{align*}
}
where the upper bound is strictly negative for $\pos \neq \goalpos$. 
\end{proposition}
\begin{proof}
See \refapp{app.headway_tailway_distance_decay}.
\end{proof}

\noindent Another characteristic of dual-headway control is achieving forward motion in finite time and maintaining it over time.

\begin{proposition}\label{prop.forward_motion_in_finite_time}
\emph{\!(Forward Motion in Finite Time)}
For $\headcoef + \goaltailcoef \!< 1$, if the unicycle pose satisfies \mbox{$\tfrac{\tr{(\goaltailpos - \headpos)}}{\norm{\goaltailpos - \headpos}} \ovectsmall{\ori}\! > \!-1$}, the dual-headway unicycle controller in \refeq{eq.dual_headway_control} switches to forward motion with $\linvel_{\goalpos, \goalori}(\pos, \ori) > 0$ in finite time, because of  the comparison lemma \cite{khalil_NonlinearSystems2001} and
{\footnotesize
\begin{align*}
\scalebox{0.92}{$
\tfrac{\diff }{\diff t} \tr{(\goaltailpos \! - \!\headpos)\!}\! \ovectsmall{\ori} \!\geq\!  - \refcoef(1 \! -\! \headcoef \!-\! \goaltailcoef) \tr{(\goaltailpos \! - \!\headpos)\!} \!\ovectsmall{\ori} \!+\! \tfrac{\refcoef\! \plist{\!\!\tr{(\goaltailpos \! - \headpos)\!} \!\nvectsmall{\ori}\!}^{\!2}\!\!}{\headcoef \norm{\pos - \goalpos}}$}
\end{align*}
}%
for $-1 < \tfrac{\tr{(\goaltailpos - \headpos)}}{\norm{\goaltailpos - \headpos}} \ovectsmall{\ori} \leq 0$.
\end{proposition}
\begin{proof}
See \refapp{app.forward_motion_in_finite_time}.
\end{proof}

\begin{proposition}\label{prop.persistent_forward_motion}
\emph{(Persistent Forward Motion)} For $\headcoef < 1$, once the unicycle starts moving forward with%
\footnote{For $\headcoef < 1$, the following forward motion equivalences hold:
\begin{align*}
\linvel_{\goalpos, \goalori}(\pos, \ori) \geq 0 &\Longleftrightarrow\tr{\plist{\goaltailpos - \headpos}} \ovectsmall{\ori} \geq 0 \\
& \Longleftrightarrow\tfrac{\tr{\plist{\goalpos - \pos}}}{\norm{\goalpos - \pos}}\ovectsmall{\ori} \geq \headcoef + \goaltailcoef \tr{\ovectsmall{\goalori}\!}\!\ovectsmall{\ori}
\end{align*}
where these equivalences are also valid with strict inequalities. 
} 
$\linvel_{\goalpos, \goalori}(\pos, \ori) \geq 0$, the dual-headway unicycle controller in \refeq{eq.dual_headway_control} persistently generates positive linear velocity away from the goal, i.e.,
\begin{align*}
\linvel_{\goalpos, \goalori}(\pos, \ori) = 0 \Longrightarrow \tfrac{\diff}{\diff t} \linvel_{\goalpos, \goalori}(\pos, \ori) > 0 \quad  \forall \pos \neq \goalpos.
\end{align*}
\end{proposition}
\begin{proof}
See \refapp{app.prop.persistent_forward_motion}.
\end{proof}

Once in forward motion, dual-headway control reduces the distance to the goal and increases alignment with the goal.

\begin{proposition}\label{prop.distance_to_goal_decay}
\emph{(Distance-to-Goal Decay under Forward Motion)}
For $\goaltailcoef \! \leq \! \headcoef \! <\! 1$, the dual-headway unicycle controller \refeq{eq.dual_headway_control} decreases the distance-to-goal under forward motion, i.e.,
\begin{align*}
\linvel_{\goalpos, \goalori}(\pos, \ori) > 0 \Longrightarrow \tfrac{\diff}{\diff t} \norm{\pos - \goalpos}^2 < 0 \quad \quad \forall \pos  \neq \goalpos.
\end{align*}
\end{proposition}
\begin{proof}
See \refapp{app.distance_to_goal_decay}.
\end{proof}

\begin{proposition}\label{prop.goal_alignment_under_forward_motion}
(Goal Alignment under Forward Motion)
For \mbox{$\goaltailcoef + 2 \headcoef < 1$}, if the unicycle is under forward motion, i.e., \mbox{$\tr{(\goaltailpos\! -\! \headpos)\!} \ovectsmall{\ori} \!\geq 0$}, then the dual-headway unicycle controller in \refeq{eq.dual_headway_control} asymptotically aligns the unicycle orientation $\ori$ with the goal orientation $\goalori$, because it continuously improves the alignment of the unicycle headway-to-tailway vector  with both the goal orientation and the average of the current and goal orientations as follows:
{\small
\begin{align*}
\begin{array}{l}
\tfrac{\tr{(\goaltailpos\! - \headpos)}}{\norm{\goaltailpos \!- \headpos}} \ovectsmall{\ori} \!\geq 0 ,\\
\tfrac{\tr{\plist{\goaltailpos - \headpos}}}{\norm{\goaltailpos - \headpos}} \ovectsmall{\goalori} \! \neq \pm 1 
\end{array} 
\!\! \Longrightarrow  \!\!
\begin{array}{l}
\frac{\diff}{\diff t} \tfrac{\tr{\plist{\goaltailpos - \headpos}}}{\norm{\goaltailpos - \headpos}} \ovectsmall{\goalori} \!> 0,\\
\frac{\diff}{\diff t} \tfrac{\tr{\plist{\goaltailpos - \headpos}}}{\norm{\goaltailpos - \headpos}} \plist{\!\ovectsmall{\ori} \!+\! \ovectsmall{\goalori}\!}\! > 0.
\end{array}
\end{align*}
}%
\end{proposition}
\begin{proof}
See \refapp{app.goal_alignment_under_forward_motion}.
\end{proof}

\begin{figure}[t]
\centering
\begin{tabular}{@{}c@{\hspace{0.2mm}}c@{\hspace{0mm}}c@{}}
\begin{tabular}{@{}c@{}}
\includegraphics[width = 0.32\columnwidth]{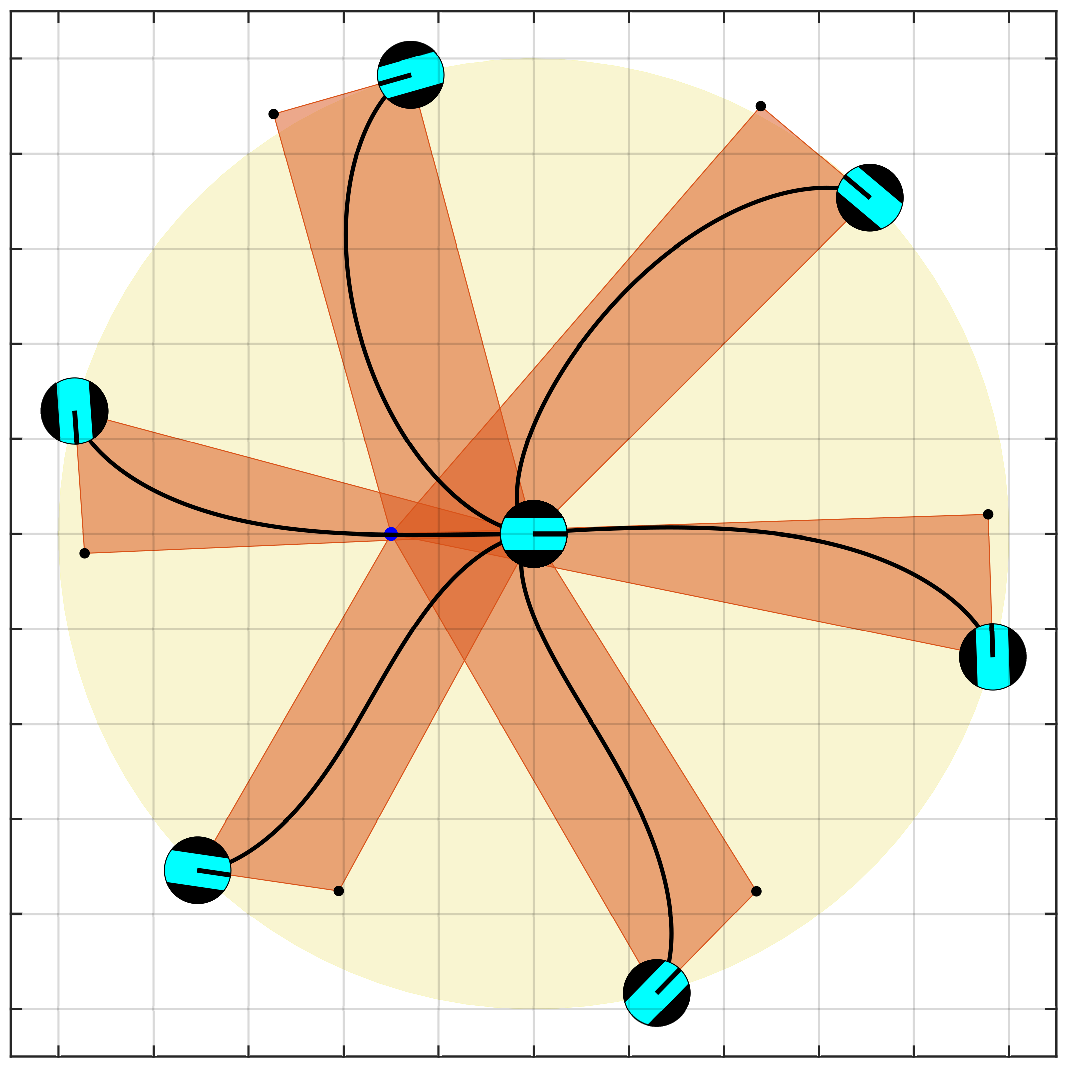}
\end{tabular} &
\begin{tabular}{@{}c@{}}
\includegraphics[width = 0.32\columnwidth]{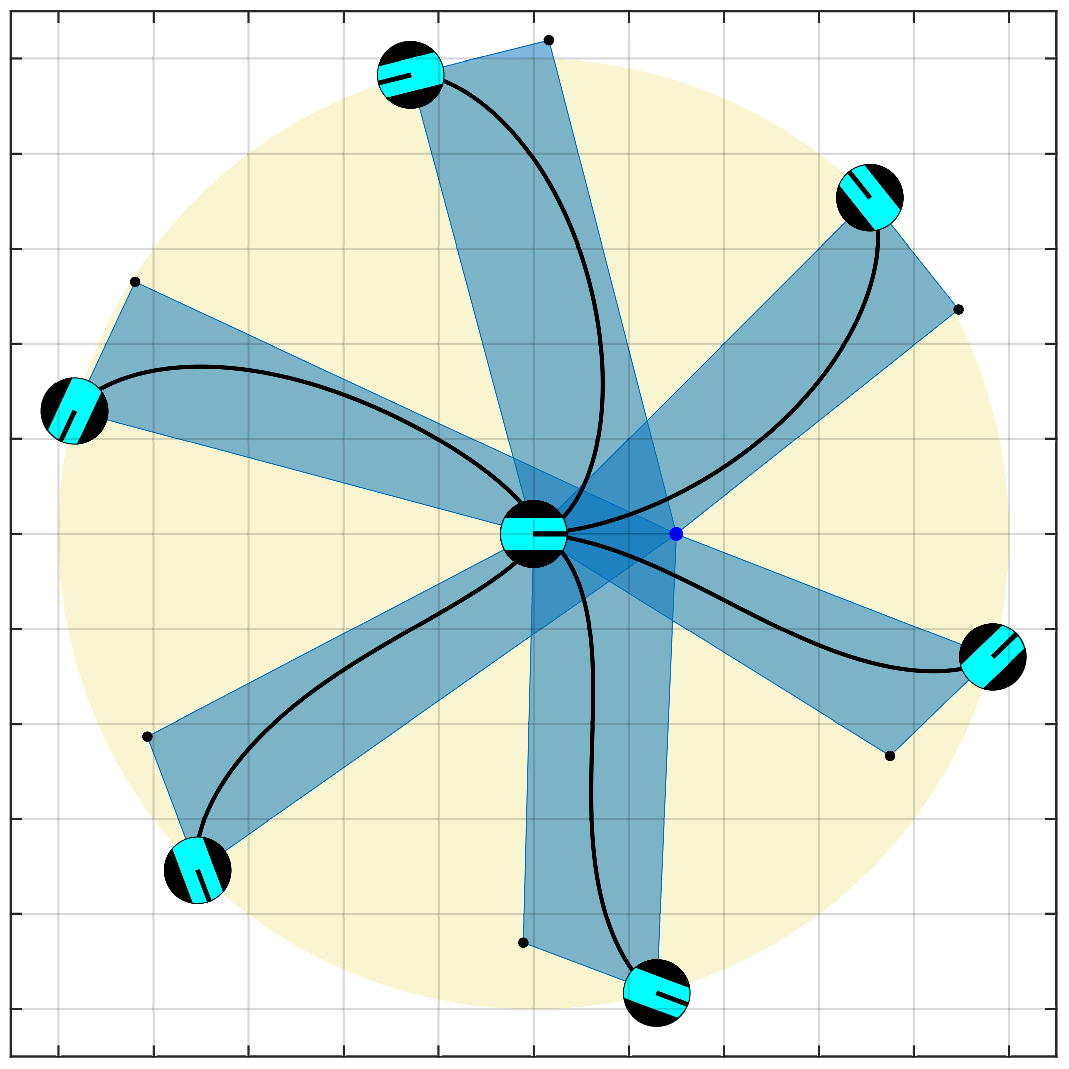} 
\end{tabular}
&
\begin{tabular}{@{\hspace{0mm}}c@{}}
\includegraphics[width = 0.39\columnwidth]{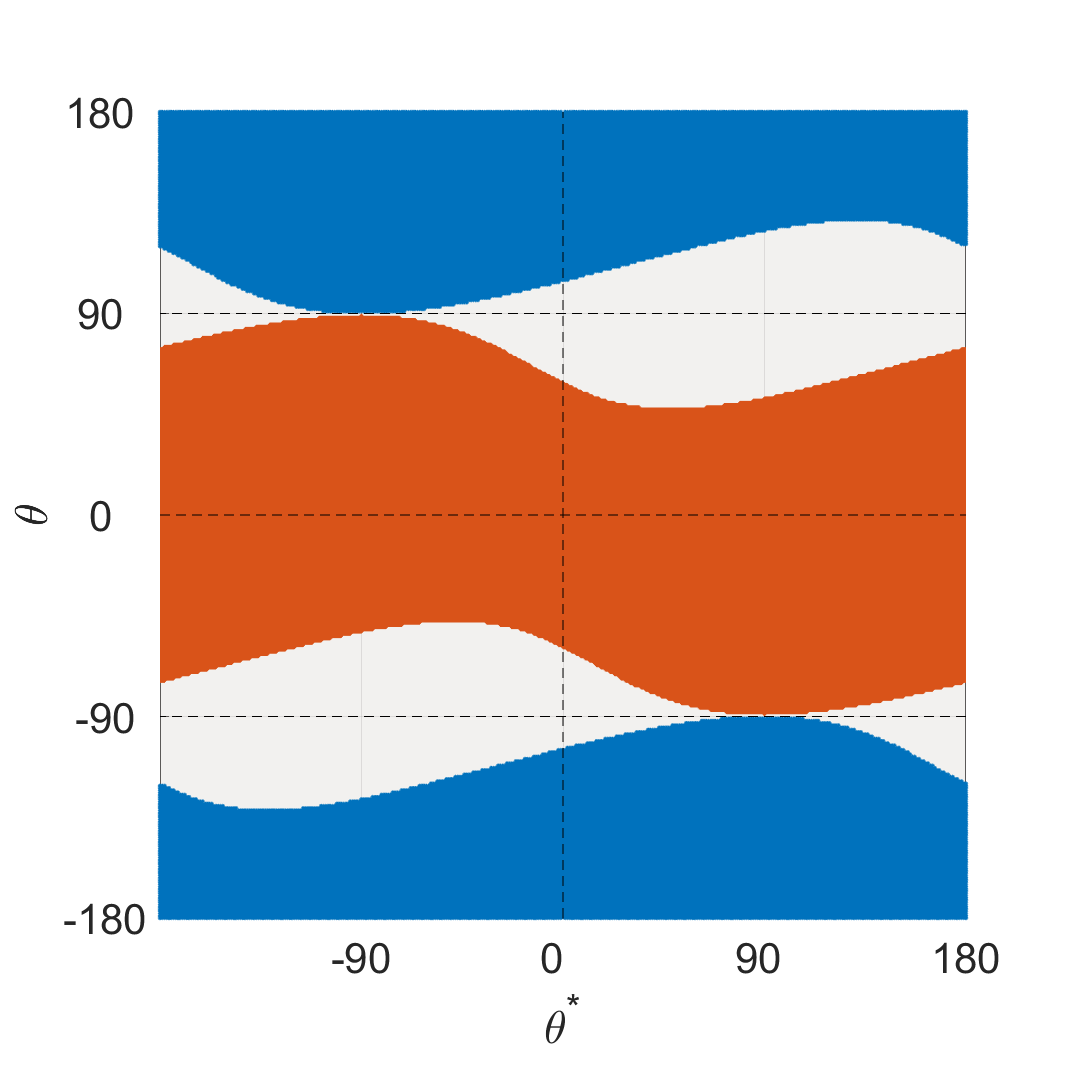}
\end{tabular}
\end{tabular}
\vspace{-6mm}
\caption{Convex feedback motion prediction bound (colored patch) on the closed-loop unicycle motion trajectory (black line) under forward (left, red) and backward (middle, blue) dual-headway unicycle control. (Right) The disjoint quotient space of the forward and backward dual-headway control domains, $\fwddomain_{\goalpos, \goalori}$ (red) and $\bwddomain_{\goalpos, \goalori}$ (blue), with $\pos \sim \scalebox{0.8}{$\begin{bmatrix}0 \\ 0 \end{bmatrix}$}$ and  $\goalpos \sim \scalebox{0.8}{$\begin{bmatrix}1 \\ 0 \end{bmatrix}$}$.   
}
\label{fig.dual_headway_motion_prediction}
\vspace{-3mm}
\end{figure}

\subsection{Forward and Backward Unicycle Motion Primitives}

As observed in \refprop{prop.forward_motion_in_finite_time}, dual-headway unicycle control might yield backward motion for some finite time before starting to approach the goal position forward. 
This initial backward motion often complicates the estimation of the spatial region needed to execute the motion, for example, as a regional bound on the closed-loop unicycle motion as feedback motion prediction \cite{isleyen_vandewouw_arslan_CDC2023}. 
To avoid such intricacies for motion planning, we consider restricting the domain of the dual-headway unicycle controller to persistently perform forward-approaching forward motion and backward-approaching backward motion with associated feedback motion predictions, as illustrated in \reffig{fig.dual_headway_motion_prediction} and \mbox{\reffig{fig.dual_headway_control_positive_inclusion}}.

\subsubsection{Forward Dual-Headway Unicycle Controller}

Based on the geometric properties of dual-headway unicycle pose control in \refsec{sec.dual_headway_unicycle_pose_control}, we define the forward dual-headway unicycle controller $\fwdctrl_{\goalpos, \goalori}=(\fwdlinvel_{\goalpos, \goalori}, \fwdangvel_{\goalpos, \goalori})$ as  in \refeq{eq.dual_headway_control} as
\begin{subequations}\label{eq.forward_dual_headway_unicycle_control}
\begin{align}
\fwdlinvel_{\goalpos, \goalori}(\pos, \ori) & := \tfrac{-\refcoef\tr{\plist{\headpos - \goaltailpos}} \ovectsmall{\ori}}{1 + \headcoef\tfrac{\tr{\pos - \goalpos}}{\norm{\pos - \goalpos}} \ovectsmall{\ori}},
\\
\fwdangvel_{\goalpos, \goalori}(\pos, \ori) &:= \tfrac{-\refcoef\tr{\plist{\headpos - \goaltailpos}} \nvectsmall{\ori}}{ \headcoef\norm{\pos - \goalpos}},
\end{align}
\end{subequations}
for any unicycle pose $(\pos, \ori)$ in the forward motion control domain  $\fwddomain_{\goalpos, \goalori}$ that is defined as
{\footnotesize
\begin{align}\label{eq.forward_unicycle_control_domain}
\fwddomain_{\goalpos, \goalori} \!\!:=\! \!\clist{\!(\pos, \ori)\! \!\in\! \R^2 \!\!\times\!\! [-\pi, \pi) \Big |  \tfrac{\tr{(\goaltailpos\! -\! \headpos)\!}\!}{\norm{\goaltailpos \!-\! \headpos}}\! \ovectsmall{\ori} \!\!\geq\! 0,\! \tfrac{\tr{(\goaltailpos \!-\! \headpos)\!}\!}{\norm{\goaltailpos \!-\! \headpos}}\! \ovectsmall{\goalori\!} \!\!>\!\! -1\! }
\end{align} 
}%
where the headway point \mbox{$\headpos = \pos \!+\! \headcoef \norm{\pos\! -\!\goalpos}\!\ovectsmall{\ori}$} of the current unicycle pose $(\pos, \ori)$ and the tailway point $\goaltailpos = \goalpos \!-\! \goaltailpos \norm{\pos \!-\! \goalpos} \! \ovectsmall{\goalori\!}$ of the goal pose $(\goalpos, \goalori)$ are defined as in \refeq{eq.headway_tailway_points} and $\headcoef, \goaltailcoef, \refcoef$ are positive headway, tailway, reference coefficients with $2\headcoef + \goaltailcoef < 1$.
Note that the first condition in \refeq{eq.forward_unicycle_control_domain} captures forward motion with positive velocity, and the second condition ensures forward alignment at the goal position, both of which increases under forward dual-headway control in \refeq{eq.forward_dual_headway_unicycle_control} (see \refprop{prop.persistent_forward_motion} and \refprop{prop.goal_alignment_under_forward_motion}).

Since the forward dual-headway unicycle controller decreases the distance to the goal (\refprop{prop.distance_to_goal_decay}), by leveraging feedback linearized unicycle dynamics, we determine a simple but accurate convex motion bound on the closed-loop unicycle trajectory that can be used for safety verification.

\begin{proposition} \label{prop.forward_unicycle_motion_prediction}
\emph{(Forward Unicycle Motion Prediction)}
Starting from any initial unicycle state $(\pos_0, \ori_0) \in \fwddomain_{\goalpos, \goalori}$ towards any goal pose $(\goalpos, \goalori)$, the closed-loop unicycle pose trajectory $(\pos(t), \ori(t))$ under the forward dual-headway unicycle controller  in \refeq{eq.forward_dual_headway_unicycle_control} is bounded for any  $t' \geq t $ as 
\begin{align*}
\pos(t') &\in \conv\plist{\pos(t), \headpos(t), \goaltailpos(t), \goalpos},
\\
\pos(t') &\in \ball(\goalpos, \norm{\pos(t) - \goalpos}),
\end{align*}  
where the convex trajectory bounds shrink over time as
{\small
\begin{align*}
\conv\plist{\pos(t'), \headpos(t'), \goaltailpos(t'), \goalpos} &\subseteq \conv\plist{\pos(t), \headpos(t), \goaltailpos(t), \goalpos}, 
\\
\ball(\goalpos, \norm{\pos(t') - \goalpos}) &\subseteq \ball(\goalpos, \norm{\pos(t) - \goalpos}).
\end{align*}
}%
\end{proposition}
\begin{proof}
See \refapp{app.forward_unicycle_motion_prediction}.
\end{proof}

\begin{figure}[t]
\centering
\vspace{3mm}
\begin{tabular}{@{}c@{\hspace{0.01\columnwidth}}c@{}}
\includegraphics[width = 0.49\columnwidth]{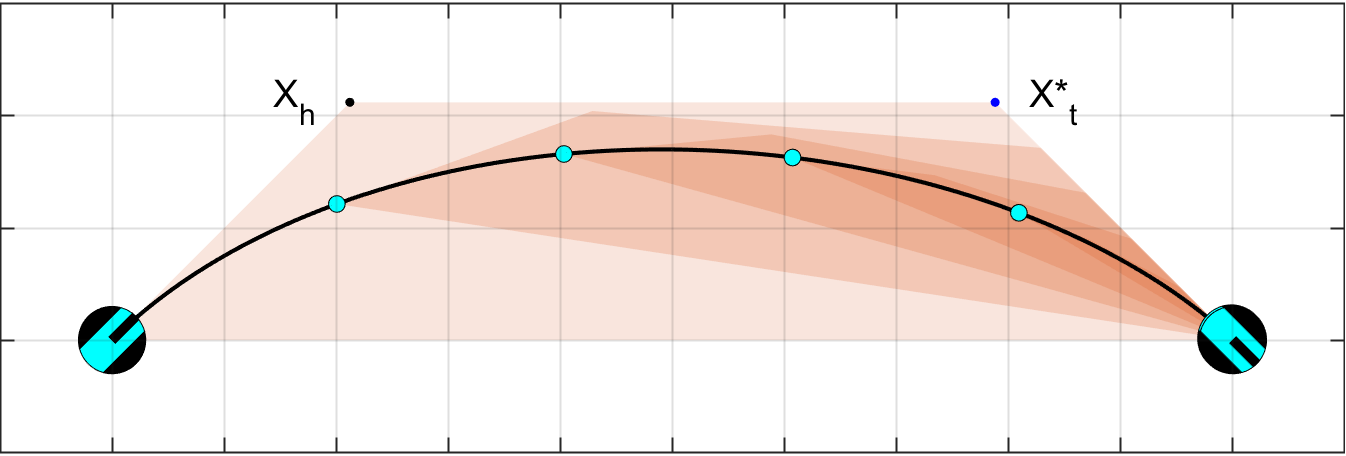} &
\includegraphics[width = 0.49\columnwidth]{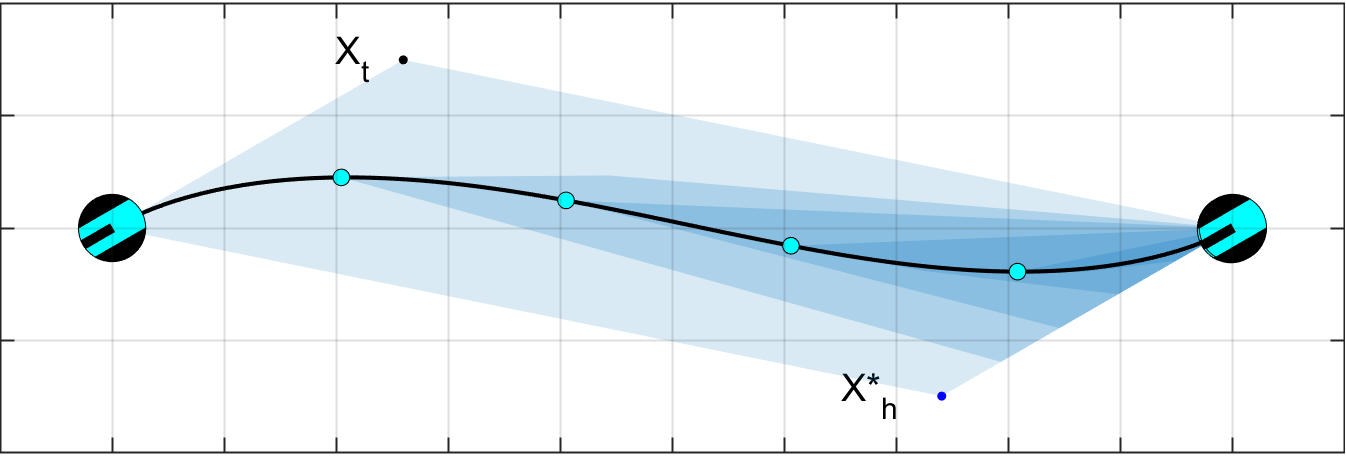}
\\[0mm]
\includegraphics[width = 0.49\columnwidth]{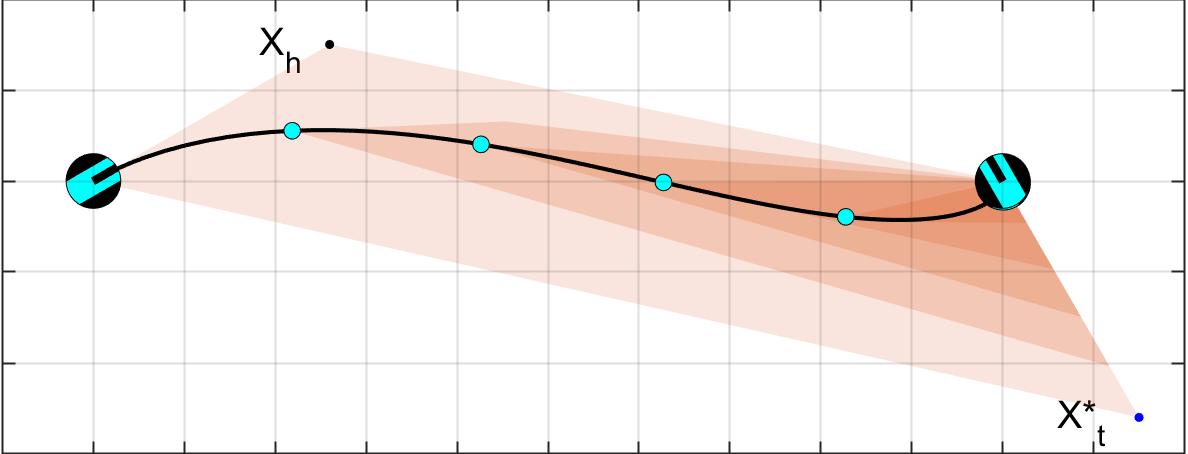} &
\includegraphics[width = 0.49\columnwidth]{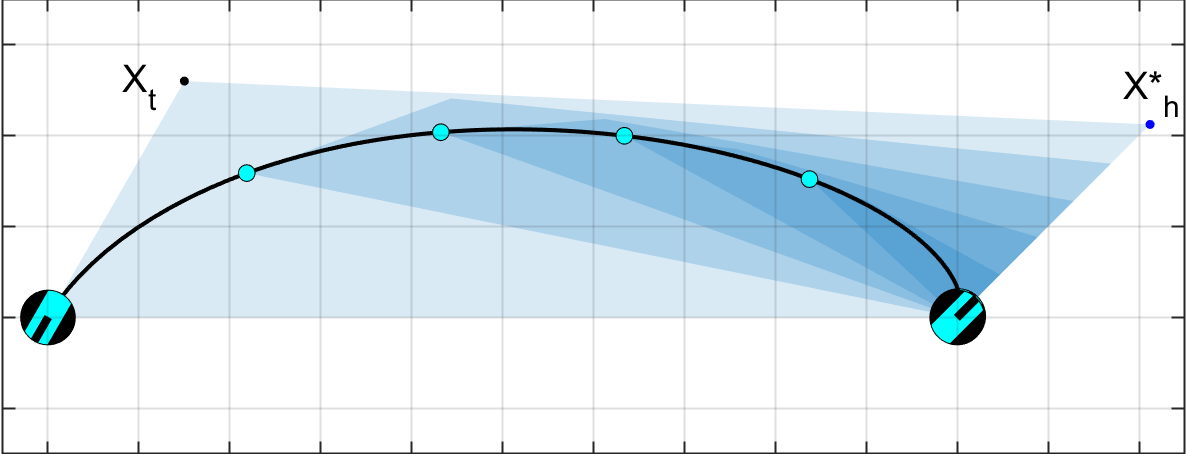}
\\[-1mm]
\end{tabular}
\vspace{-2mm}
\caption{
Positive inclusion of convex feedback motion prediction for (left) forward  and (right) backward dual-headway unicycle control.
}
\label{fig.dual_headway_control_positive_inclusion}
\vspace{-3mm}
\end{figure}

\subsubsection{Backward Dual-Headway Unicycle Controller}     
\label{sec.backward_dual_headway_unicycle_control}

Similar to the construction of dual-headway unicycle motion control in \refsec{sec.dual_headway_unicycle_pose_control}, using the tailway point $\tailpos$ of the current unicycle pose $(\pos, \ori)$ and the headway point $\goalheadpos$ of the goal pose $(\goalpos, \goalori)$ that are defined as%
\footnote{Note that the tailway point $\tailpos$ and the headway point $\goalheadpos$ under the unicycle dynamics in \refeq{eq.UnicycleDynamics} continuously evolve as
\begin{align*}
\tailposdot & = \plist{1 - \tailcoef \tfrac{\tr{(\pos - \goalpos)}}{\norm{\pos - \goalpos}}\ovectsmall{\ori}} \ovectsmall{\ori} \linvel  - \tailcoef \norm{\pos - \goalpos} \nvectsmall{\ori} \angvel, 
\\
\goalheadposdot & = \goalheadcoef  \tfrac{\tr{(\pos - \goalpos)}}{\norm{\pos - \goalpos}}\ovectsmall{\ori} \ovectsmall{\goalori} \linvel.
\end{align*}
}
\begin{subequations}
\begin{align}
\tailpos &:= \pos - \tailcoef \norm{\pos - \goalpos}\ovectsmall{\ori}, 
\\
\goalheadpos&:= \goalpos + \goalheadcoef \norm{\pos - \goalpos}\ovectsmall{\goalori},
\end{align}
\end{subequations}
and the tailway reference dynamics
\begin{align}
\tailposdot = -\refcoef\plist{\tailpos - \goalheadpos}
\end{align}
we design the backward dual-headway unicycle controller, denoted by $\bwdctrl_{\goalpos, \goalori}(\pos, \ori) = (\bwdlinvel_{\goalpos, \goalori}(\pos,\ori), \bwdangvel_{\goalpos, \goalori}(\pos, \ori))$, as 
\begin{subequations}\label{eq.backward_dual_headway_unicycle_control}%
\begin{align}
\overleftarrow{\linvel}_{\goalpos, \goalori}(\pos, \ori) & :=\frac{-\refcoef\tr{\plist{\tailpos - \goalheadpos}}\ovectsmall{\ori}}{1 - \tailcoef\tfrac{\tr{(\pos - \goalpos)}}{\norm{\pos - \goalpos}}\ovectsmall{\ori}},
\\
\overleftarrow{\angvel}_{\goalpos, \goalori}(\pos, \ori) & := \frac{\refcoef\tr{\plist{\tailpos - \goalheadpos}}\nvectsmall{\ori}}{\tailcoef \norm{\pos - \goalpos}},
\end{align}
\end{subequations}
for any unicycle pose $(\pos, \ori)$ in the the backward motion domain $\bwddomain_{\goalpos, \goalori}$ that is defined to be
{\footnotesize
\begin{align}\label{eq.backward_unicycle_control_domain}
\bwddomain_{\goalpos, \goalori} \!\!:=\!\! \clist{\!(\pos, \ori) \!\!\in \!\R^2 \!\!\times\! \![-\pi, \pi) \Big |  \tfrac{\tr{(\goalheadpos \!-\! \tailpos)\!}\!}{\norm{\goalheadpos \!-\! \tailpos}} \!\ovectsmall{\ori} \!\!\leq\! 0,\! \tfrac{\tr{(\goalheadpos - \tailpos)\!}\!}{\norm{\goalheadpos \!-\! \tailpos}}\! \ovectsmall{\goalori\!} \!\!<\!\! +1 \!}
\end{align} 
}%
where $\tailcoef, \goalheadcoef, \refcoef >0$ are constant positive tailway, headway, and reference coefficients, respectively, with $2 \tailcoef + \goalheadcoef < 1$.
The first condition in \refeq{eq.backward_unicycle_control_domain} corresponds to the backward motion with negative linear velocity, and the second condition guarantees backward alignment at the goal position, both of which are decreasing under the backward dual-headway unicycle controller in \refeq{eq.backward_dual_headway_unicycle_control} (due to its similarity with forward motion control in  \refprop{prop.persistent_forward_motion} and \refprop{prop.goal_alignment_under_forward_motion}).
Hence, due to their symmetry, the forward and backward dual-headway motion control share similar geometric properties.

\begin{proposition}\label{eq.backward_unicycle_motion_prediction}
\emph{(Backward Unicycle Motion Prediction)}
For any goal pose $(\goalpos, \goalori)$ and $2\tailcoef+\goalheadcoef < 1$, the backward dual-headway unicycle controller in \refeq{eq.backward_dual_headway_unicycle_control} ensures that the closed-loop unicycle pose trajectory $(\pos(t), \ori(t))$, starting from any $(\pos(0), \ori(0)) \in \bwddomain_{\goalpos, \goalori}$, is bounded for any $t' \geq t \geq 0$ as 
\begin{align*}
&\pos(t') \in \conv\plist{\pos(t), \tailpos(t), \goalheadpos(t), \goalpos},  
\\
&\pos(t') \in \ball(\goalpos, \norm{\pos(t) - \goalpos}),
\end{align*}
where the trajectory bounds are positively inclusive, i.e.,
{\small
\begin{align*}
\conv\plist{\pos(t'), \tailpos(t'), \goalheadpos(t'), \goalpos} &\subseteq \conv\plist{\pos(t), \tailpos(t), \goalheadpos(t), \goalpos},\\
\ball(\goalpos, \norm{\pos(t') - \goalpos}) &\subseteq \ball(\goalpos, \norm{\pos(t) - \goalpos}).
\end{align*}
}%
\end{proposition}
\begin{proof}
The result follows from a similar argument as in the proof of \refprop{prop.forward_unicycle_motion_prediction} in \refapp{app.forward_unicycle_motion_prediction}.
\end{proof}     


\section{Optimal Unicycle Feedback Motion Planning}
\label{sec.optimal_unicycle_feedback_motion_planning}

In this section, we present an example application of dual-headway unicycle pose control for optimal sampling-based feedback motion planning around obstacles using rapidly exploring random trees (RRTs) \cite{lavalle_kuffner_IJRR2001, karaman_frazzoli_IJRR2011}. 
For ease of exposition, we consider a unicycle robot with a circular robot body shape with radius of $\radius$ moving in a planar environment $\workspace \subseteq \R^{2}$ with known obstacles $\obstspace$ where the space of collision-free robot position is given by
\begin{align*}
\freespace : = \clist{\pos \in \workspace\big | \ball(\pos, \radius) \cap \obstspace = \varnothing}
\end{align*}  
where $\ball(\pos, \radius):= \clist{\vect{y} \in \R^{2} \big | \norm{\vect{y} - \pos} \leq \radius}$ denotes the 2D Euclidean ball centered at $\pos \in \R^{2}$ with radius $\radius > 0$.

\subsection{Elements of Sampling Based Planning for Unicycles}
\label{sec.elements_of_sampling_based_planning}

To perform optimal sampling-based unicycle motion planning, we outline essential components below, including sampling, distance, neighborhood, projection, local cost, and safety verification of unicycle poses and their connections.

\subsubsection{Random Sampling of Unicycle Poses}

Uniform sampling of collision-free unicycle poses can be performed using rejection sampling (i.e., repeatedly drawing sample points from a box-shaped subset of $\R^2$ containing $\freespace$ until a collision-free position in $\freespace$ is found) combined with uniform orientation sampling over $[-\pi, \pi)$. 
We denote such a uniform unicycle pose sampler by $\unisample(\freespace \times [-\pi, \pi)) \mapsto (\pos, \ori)$, which returns independently and identically distributed unicycle pose samples $(\pos, \ori)$ from $\freespace \times [-\pi, \pi)$.
For goal-biased sampling toward a given global goal pose $(\goalpos, \goalori) \in \freespace \times [-\pi, \pi]$, denoted by $\unisample(\freespace \times [-\pi, \pi), (\goalpos, \goalori))$, we generate a sample unicycle pose exactly at the goal pose with a fixed ( small) probability $p^* \in [0,1]$, and perform uniform unicycle pose sampling $\unisample(\freespace \times [-\pi, \pi))$ otherwise (with probability $(1 - p^*)$), until establishing a connection to the goal pose.  
When enabled, in informed sampling, we also reject sample unicycle poses whose travel cost, via their nearest neighbor to the start pose, exceeds the travel cost from the start pose to the goal pose, which becomes effective once a path from the start to the goal is found.

\subsubsection{Distance of Unicycle Poses}
The distance, denoted by $\unidist((\pos, \ori), (\hat{\pos}, \hat{\ori}))$, between two unicycle poses $(\pos, \ori)$ and $(\hat{\pos}, \hat{\ori})$ in $\mathbb{R}^{2} \times [-\pi, \pi)$ is a nonnegative symmetric function that quantifies the dissimilarity between the two unicycle poses, see \reftab{tab.unidist}.%
\footnote{Note that $\unidist$ does not need to define a true metric in $\mathbb{R}^2 \times [-\pi, \pi)$, which must also be diminish to zero for identical unicycle poses and satisfy the triangle inequality.}
A common choice of unicycle distance is an additive weighted combination of unicycle translation and orientation distances \cite{palmieri_arras_IROS2014}, defined as follows: 
\begin{align*}
\unidist := \alpha \, \unidist_{\mathrm{translate}} + \beta \, \unidist_{\mathrm{orient}}  
\end{align*}
where $\alpha \geq 0$ and $\beta \geq 0$ are nonnegative weights for the unicycle translation and orientation distances.
For example, a standard choice for unicycle translation distance is the Euclidean distance, i.e.,
\begin{align*}
\unidist_{\euclidean}((\pos, \ori), (\hat{\pos}, \hat{\ori})) := \norm{\pos - \hat{\pos}},
\end{align*}
while a standard choice for unicycle orientation distance is the cosine distance,\footnote{Note that the cosine distance, $1 - \tr{\ovectsmall{\ori}}\ovectsmall{\hat{\ori}}$, is not a metric on the unit circle, but one can easily define a geodesic distance as $\left | \mathrm{arctan2}\plist{\tr{\nvectsmall{\ori}}\ovectsmall{\hat{\ori}}, \tr{\ovectsmall{\ori}}\ovectsmall{\hat{\ori}}}\right|$ which measures the absolute angular different between $\ovectsmall{\ori}$ and $\ovectsmall{\hat{\ori}}$.} i.e.,
\begin{align*}
\unidist_{\cosine}((\pos, \ori), (\hat{\pos}, \hat{\ori})) := 1 - \tr{\ovectsmall{\ori}}\ovectsmall{\hat{\ori}} \in \blist{0,2}.
\end{align*}
We also find it useful to define another unicycle translation distance as the product of the Euclidean distance and the cosine distance, plus one, as follows:
\begin{align*}
\unidist_{\euclideancosine}((\pos, \ori), (\hat{\pos}, \hat{\ori})):= \norm{\pos - \hat{\pos}} \plist{\!2 \!- \!\scalebox{0.95}{$\tr{\ovectsmall{\ori}\!}\!\ovectsmall{\hat{\ori}}$}\!}
\end{align*}
which is tightly lower and upper bounded as 
\begin{align*}
\norm{\pos - \hat{\pos}} \leq \unidist_{\euclideancosine}((\pos, \ori), (\hat{\pos}, \hat{\ori})) \leq 3 \norm{\pos - \hat{\pos}}
\end{align*}
where the lower and upper bounds are realized when the poses are perfectly aligned or opposite.

\begin{table}[t]
\caption{Distance Between Unicycle Poses}
\label{tab.unidist}
\centering
\vspace{-3mm}
\begin{tabular}{@{}c@{\hspace{2mm}}c@{}}
\hline
\hline 
\\[-2mm]

Distance Type & $\unidist((\pos, \ori), (\hat{\pos}, \hat{\ori}))$
\\
\hline
\\[-2mm]
\begin{tabular}{@{}c@{}}
Euclidean  Translation
\end{tabular}
 & $\norm{\pos - \hat{\pos}}$
\\
\begin{tabular}{@{}c@{}}
Cosine  Orientation
\end{tabular}
& $1 - \tr{\ovectsmall{\ori}}\ovectsmall{\hat{\ori}}$ 
\\
\begin{tabular}{@{}c@{}}
Euclidean-Cosine \\ Translation
\end{tabular}
& $\norm{\pos - \hat{\pos}} \plist{2 - \tr{\ovectsmall{\ori}}\ovectsmall{\hat{\ori}}}$
\\
\begin{tabular}{@{}c@{}}
Dual-Headway \\  Translation
\end{tabular}
& \!\!\!\!\!\! \hspace{-4mm}$\norm{\pos \!-\! \hat{\pos}} \plist{\!2\kappa \!+\! \min\plist{\nlist{\!\tfrac{\pos - \hat{\pos}}{\norm{\pos - \hat{\pos}}} \! \pm\! \kappa \ovectsmall{\ori} \!\pm\! \kappa \ovectsmall{\hat{\ori}}\!}}\!\!}$
\\
\begin{tabular}{@{}c@{}}
Dual-Headway \\  Orientation
\end{tabular} 
& \!\!\!\! $\min\plist{\nlist{\tfrac{\pos - \hat{\pos}}{\norm{\pos - \hat{\pos}}} \pm \kappa \ovectsmall{\ori} \pm \kappa \ovectsmall{\hat{\ori}}}} \!-\! 1 \! + \!2\kappa$
\\
\begin{tabular}{@{}c@{}}
Additively Weighted \\ Translation \& Orientation 
\end{tabular}
&
$\alpha \, \unidist_{\mathrm{translate}} + \beta \, \unidist_{\mathrm{orient}}$
\\
\hline
\\[-2.5mm]
\hline
\end{tabular}
\vspace{-3mm}
\end{table}

Inspired from the dual-headway unicycle pose control in \refsec{sec.adaptive_dual_headway_unicycle_control}, we define the dual-headway translation distance as the minimum travel distance through headway and tailway points of unicycle poses as%
\footnote{Inspired by the convexity property of B\'ezier curves \cite{arslan_tiemessen_TRO2022}, we consider the headway and tailway points as the intermediate control points of a third-order Bézier curve joining the start and goal positions since B\'ezier curves are also bounded by the convex hull of their control points and the length of a B\'ezier curves is bound above by B\'ezier polygon length.}
\footnote{One can also consider the minimum distance between the headway and tailway points of unicycle poses as a unicycle pose distance measure as 
\begin{align*}
& \unidist_{\headtail}((\pos, \ori), (\hat{\pos}, \hat{\ori})) := \min\plist{
\begin{array}{@{}c@{}}
\norm{\head{\pos} - \tail{\hat{\pos}}}, \norm{\tail{\pos} - \head{\hat{\pos}}}
\end{array}
}  
\\
& \hspace{10mm} = \underbrace{\norm{\pos - \hat{\pos}}}_{\substack{\text{straight-line} \\ \text{distance}}} \underbrace{\min\plist{\nlist{\frac{\pos - \hat{\pos}}{\norm{\pos - \hat{\pos}}} \pm \kappa \ovectsmall{\ori} \pm \kappa \ovectsmall{\hat{\ori}} }}}_{\substack{\text{a measure of orientational mismatch}\\ \in [1 - 2\kappa, 1 + 2\kappa]}} 
\end{align*}
whose ratio to the Euclidean distance, respectively, defines the headway-tailway orientation distance as 
\begin{align*}
& \unidist_{\headtailori} (\!(\pos, \ori), (\hat{\pos}, \hat{\ori})\!) \! := \! \frac{\unidist_{\headtail}(\!(\pos, \ori), (\hat{\pos}, \hat{\ori})\!)}{\norm{\pos - \hat{\pos}}} \!-\! 1 \!+\! 2\kappa
\\
& \hspace{15mm} = \min\plist{\nlist{\frac{\pos - \hat{\pos}}{\norm{\pos - \hat{\pos}}} \pm \kappa \ovectsmall{\ori} \pm \kappa \ovectsmall{\hat{\ori}} }} -1 + 2\kappa
\\
& \hspace{15mm} = \unidist_{\dualheadori} (\!(\pos, \ori), (\hat{\pos}, \hat{\ori})\!)
\end{align*} 
which is exactly equivalent to the dual-headway orientation distance.
} 
{\small
\begin{align*}
&\unidist_{\dualhead}((\pos, \ori), (\hat{\pos}, \hat{\ori})) \\
&\hspace{10mm}  := \min\plist{
\begin{array}{@{}c@{}}
\norm{\pos \!-\! \head{\pos}} + \norm{\head{\pos} \!-\! \tail{\hat{\pos}}} + \norm{\tail{\hat{\pos}} \!-\! \hat{\pos}}, \\ 
\norm{\pos \!-\! \tail{\pos}} + \norm{\tail{\pos} \!-\! \head{\hat{\pos}}} + \norm{\head{\hat{\pos}} \!-\! \hat{\pos}} 
\end{array}
}  
\\
& \hspace{10mm} = \underbrace{\norm{\pos \!-\! \hat{\pos}}}_{\substack{\text{straight-line} \\ \text{distance}}} \underbrace{\min\plist{2 \headtailcoef + \nlist{\frac{\pos \!-\! \hat{\pos}}{\norm{\pos \!-\! \hat{\pos}}} \!\pm\! \headtailcoef \ovectsmall{\ori} \!\pm\! \headtailcoef \ovectsmall{\hat{\ori}} }}}_{\substack{\text{a measure of orientational mismatch}\\ \in [1, 1 + 4\headtailcoef]}} 
\end{align*}
}%
where the headway and tailway points of unicycle pose $(\pos, \ori)$ and $(\hat{\pos}, \hat{\ori})$, associated with a shared identical headway and tailway coefficient $\headtailcoef \in (0, \tfrac{1}{2})$, are defined as in \refeq{eq.headway_tailway_points} as
\begin{align}\label{eq.all_headway_tailway_points}
\begin{array}{@{}l@{}}
\head{\pos}\!:=\! \pos + \headtailcoef \norm{\pos - \hat{\pos}} \ovectsmall{\ori},
\quad
\head{\hat{\pos}}\!:=\! \hat{\pos} + \headtailcoef \norm{\pos - \hat{\pos}} \ovectsmall{\hat{\ori}},
\\
\tail{\pos}\!:=\! \pos - \headtailcoef \norm{\pos - \hat{\pos}}\ovectsmall{\ori},
\quad
\tail{\hat{\pos}}\!:=\! \hat{\pos} - \headtailcoef \norm{\pos - \hat{\pos}}\ovectsmall{\hat{\ori}}.
\end{array} \!\!\!\!
\end{align}
As highlighted above, the dual-headway unicycle distance behaves as the product of translation and orientation distances.
Accordingly, we define the dual-headway orientation distance as the ratio of the dual-headway translation distance to the Euclidean translation distance as
{\small
\begin{align*}
& \unidist_{\dualheadori} ((\pos, \ori), (\hat{\pos}, \hat{\ori}))  
\\
& \hspace{10mm}:= \frac{\unidist_{\dualhead}((\pos, \ori), (\hat{\pos}, \hat{\ori}))}{\norm{\pos - \hat{\pos}}} - 1
\\
& \hspace{10mm} = \min\underbrace{\plist{\nlist{\frac{\pos - \hat{\pos}}{\norm{\pos - \hat{\pos}}} \pm \kappa \ovectsmall{\ori} \pm \kappa \ovectsmall{\hat{\ori}} }}}_{\in \blist{1 - 2\headtailcoef, 1 + \headtailcoef}} - 1 + 2\headtailcoef
\end{align*} 
}%
for $\pos \neq \hat{\pos}$; and $2\kappa - \kappa \nlist{\ovectsmall{\ori} + \ovectsmall{\hat{\ori}}}$ otherwise.
Note that under both the forward and backward dual-headway unicycle control policies in \refsec{sec.adaptive_dual_headway_unicycle_control}, the Euclidean distance to the goal, the distance between headway and tail points, and consequently the dual-headway translation distance all decrease (see \refprop{prop.headway_tailway_distance_decay} and \refprop{prop.distance_to_goal_decay}).
Moreover, the dual-headway orientation distance is also better aligned with the actual total turning during closed-loop motion compared to the cosine orientation distance, as seen in \reffig{fig.dual_headway_total_turning}.


\begin{figure}[t]
\centering
\begin{tabular}{@{}c@{\hspace{1mm}}c@{\hspace{1mm}}c@{}}
\begin{tabular}{@{\hspace{-1mm}}c@{}}
\includegraphics[width = 0.355\columnwidth]{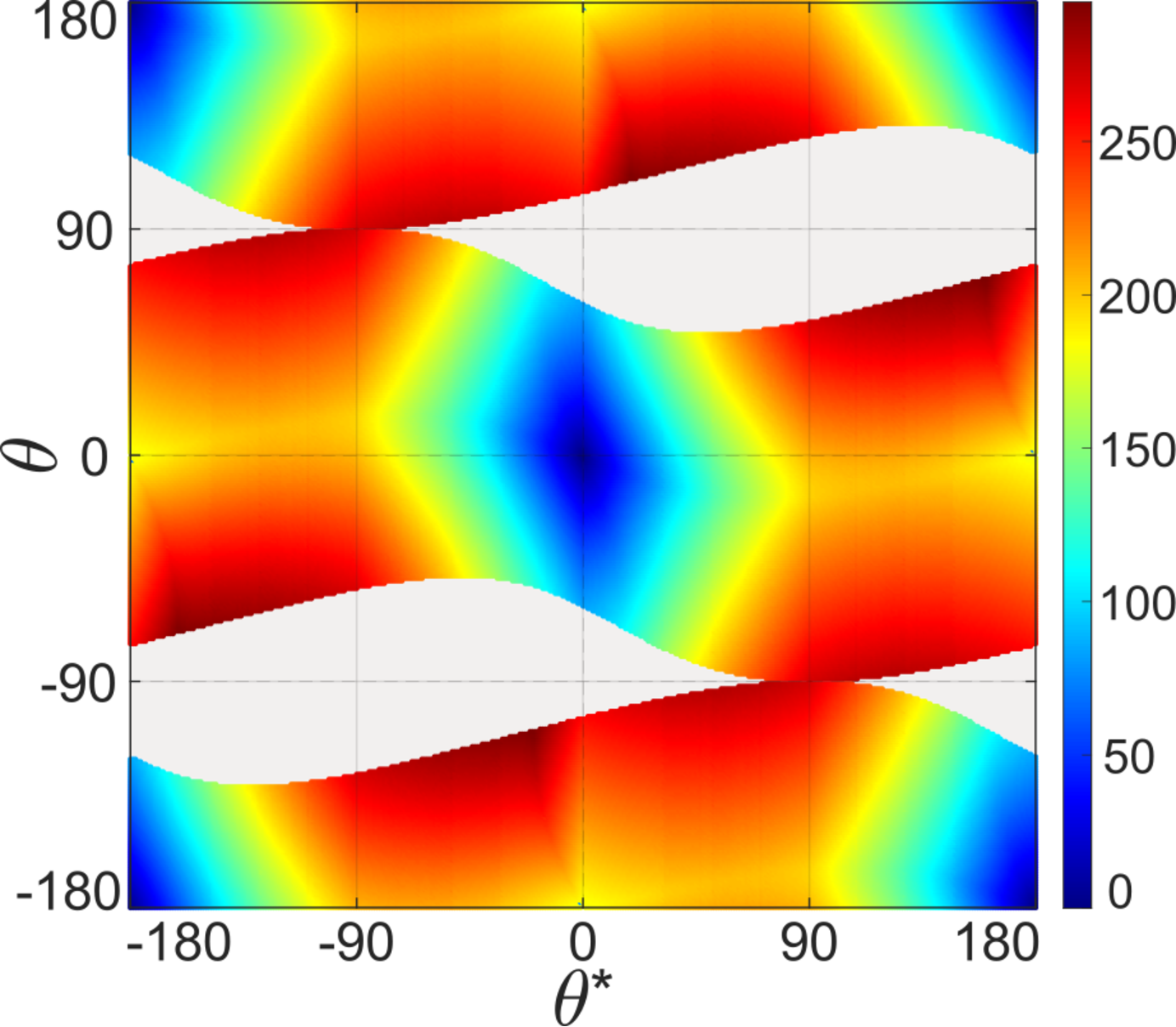} 
\end{tabular}
&
\begin{tabular}{@{}c@{}}
\includegraphics[width = 0.315\columnwidth]{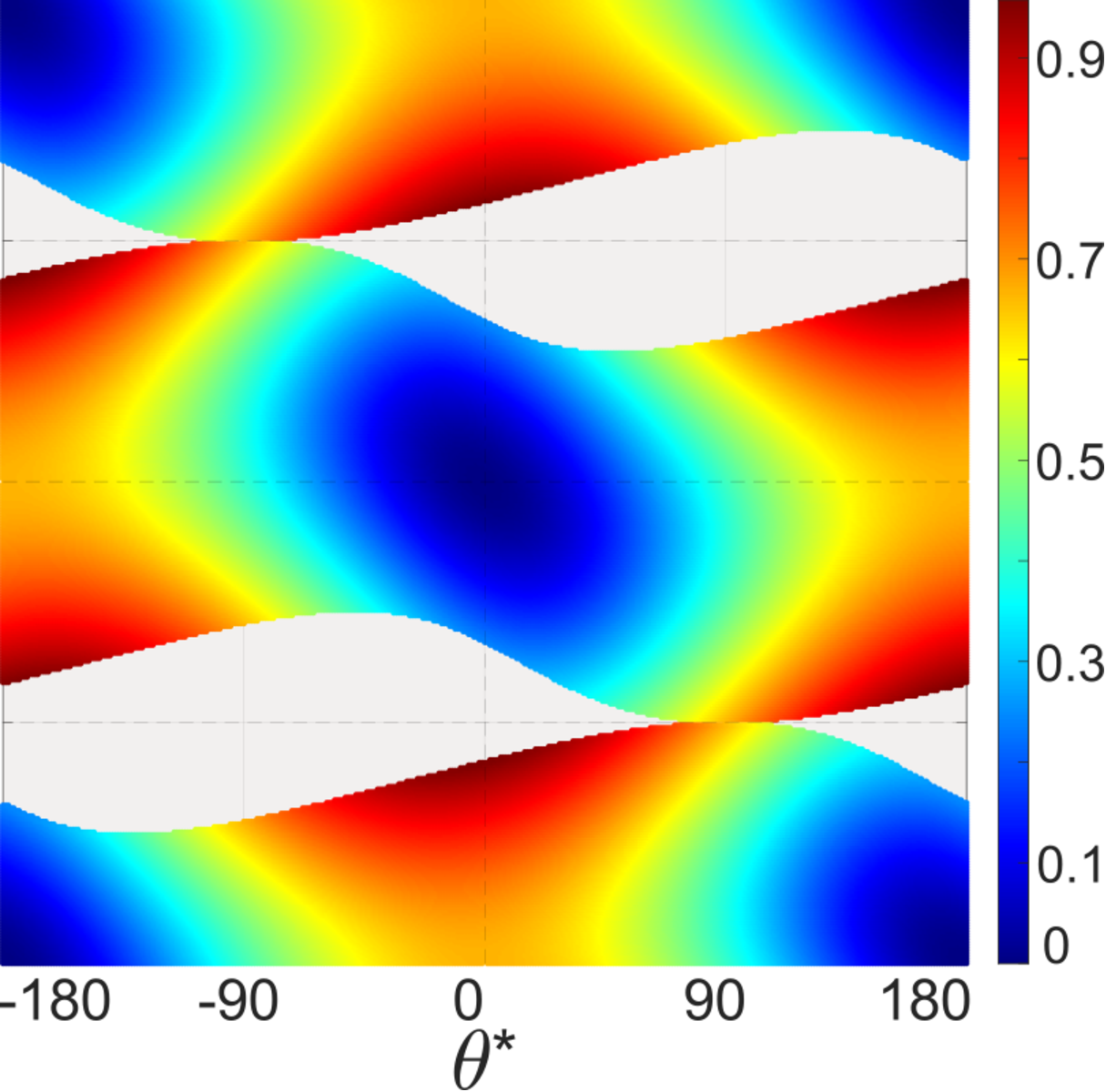} 
\end{tabular}&
\begin{tabular}{@{}c@{}}
\includegraphics[width = 0.315\columnwidth]{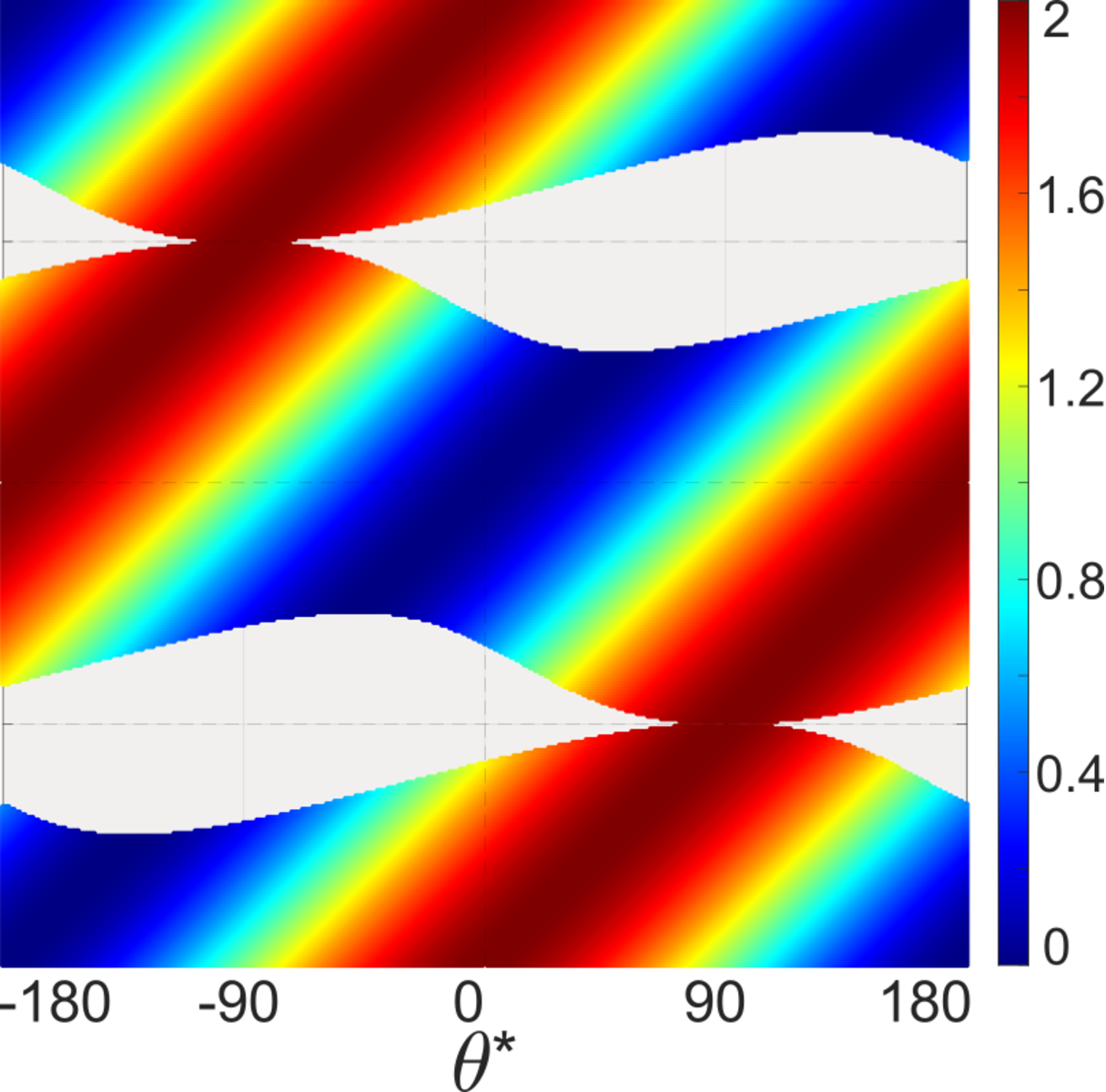} 
\end{tabular}
\end{tabular}
\vspace{-5mm}
\caption{Total turning effort estimation of dual-headway unicycle control. (left) Numerically computed total absolute turning along the closed-loop motion trajectory, (middle) dual-headway orientation distance, and (right) cosine orientation distance over the quotient space of the forward and backward dual-headway unicycle control domains, shown in \reffig{fig.dual_headway_motion_prediction}.
}
\label{fig.dual_headway_total_turning}
\vspace{-3mm}
\end{figure}

\subsubsection{Neighborhood of Unicycle Poses}
Given a set of ordered sample unicycle poses $\unisampleset:= \plist{\big.(\pos_1, \ori_1), \ldots, (\pos_m,\ori_m)}$, the nearest neighbor of $(\pos, \ori)$ from the sample unicycle poses $\unisampleset$ can be selected as the sample unicycle pose with the minimum distance,  based on a weighted translational and orientational unicycle distance, $\unidist= \alpha \unidist_{\translate} + \beta \unidist_{\orient}$, as  
\begin{align*}
&\uninearest_{\unisampleset}(\pos, \ori):= \argmin_{(\hat{\pos}, \hat{\ori}) \in \unisampleset} \unidist((\pos, \ori), (\hat{\pos}, \hat{\ori})). 
\end{align*}   
We define the coupled and decoupled translational and orientational neighborhoods of a unicycle pose $(\pos, \ori)$ from $\unisampleset$~as 
\begin{align*}
\neighbor_{\unisampleset, \Deltaradius} (\pos, \ori) &\!:=\! \clist{\!(\hat{\pos}, \hat{\ori}) \!\in\! \unisampleset \big | \unidist((\pos,\ori),(\hat{\pos}, \hat{\ori})) \!\leq\! \Deltaradius\!},
\\
\neighbor_{\unisampleset, (\Deltapos, \Deltaori)}(\pos, \ori)&\!:=\! \!\clist{\!(\hat{\pos}, \hat{\ori}) \!\!\in\! \unisampleset \Big | 
\scalebox{0.85}{$
\begin{array}{@{}l@{}}
\unidist_{\translate}((\pos,\ori),(\hat{\pos}, \hat{\ori})) \!\leq\! \Deltapos, 
\\
\unidist_{\orient}((\pos,\ori),(\hat{\pos}, \hat{\ori})) \!\leq\! \Deltaori
\end{array}
$}
\!},
\end{align*}
where $\Deltaradius \geq 0$ is a combined translational and orientational neighborhood tolerance, while $\Deltapos \geq 0$ and $\Deltaori \geq 0$ are separate translational and orientational neighborhood tolerances.
Due to the additive form of the unicycle distance, we have
{\small
\begin{align*}
\Deltaradius \geq  \alpha \Deltapos + \beta \Deltaori \Longrightarrow \neighbor_{\unisampleset, \Deltaradius} (\pos, \ori) \supseteq \neighbor_{\unisampleset, (\Deltapos, \Deltaori)}(\pos, \ori),
\\
\Deltaradius \leq \min(\alpha \Deltapos, \beta\Deltaori) \Longleftrightarrow \neighbor_{\unisampleset, \Deltaradius} (\pos, \ori) \subseteq \neighbor_{\unisampleset, (\Deltapos, \Deltaori)}(\pos, \ori).
\end{align*}
}%
We find the decoupled sample neighborhood $\neighbor_{\unisampleset, (\Deltapos, \Deltaori)}(\pos, \ori)$ to be more effective and precise in practice for determining translational and orientational similarity, whereas the coupled version is more effective for selecting the nearest neighbor using $\uninearest_{\unisampleset}(\pos,\ori)$.
Accordingly, we define the sample neighbor set as follows:%
\footnote{As a discrete neighborhood notion, one can also  further restrict the neighbor set to the $k$-nearest neighbors of $(\pos, \ori)$ from a set of unicycle poses $\unisampleset$, for some positive maximum number of neighbors $0 < k \leq |\unisampleset|$, defined as follows:
%
{
\begin{align*}
\neighbor_{\unisampleset, k}(\pos, \ori) := \clist{(\pos_{\mathrm{sort}_{\unisampleset, (\pos, \ori)}(i)}, \ori_{\mathrm{sort}_{\mathcal{U}, (\pos, \ori)}(i)}) \big | i = 1, \ldots, k} 
\end{align*}
}%
where $\mathrm{sort}_{\unisampleset, (\pos, \ori)}: \clist{1, \ldots, m} \rightarrow \clist{1, \ldots, m}$ is a bijective function that represents the sorting permutation of the unicycle poses in $\mathcal{U}$ based on their distance to the unicycle pose $(\pos, \ori)$ such that 
\begin{align*}
& \unidist((\pos_{\mathrm{sort}_{\mathcal{U}, (\pos, \ori)}(i)}, \ori_{\mathrm{sort}_{\mathcal{U}, (\pos, \ori)}(i)}, (\pos, \ori)) \\
& \quad\quad \leq \unidist((\pos_{\mathrm{sort}_{\mathcal{U}, (\pos, \ori)}(i+1)}, \ori_{\mathrm{sort}_{\mathcal{U}, (\pos, \ori)}(i+1)}, (\pos, \ori)) 
\end{align*}
for all $i = 1, \ldots, m$.
}
\begin{align*}
\unineighbor_{\unisampleset}(\pos, \ori):= \neighbor_{\unisampleset, (\Deltapos, \Deltaori)} (\pos, \ori).
\end{align*}

\subsubsection{Projection of Unicycle Poses for Local Steering}


To expand a random motion planning graph with a bounded step size during the exploration of collision-free space, one often needs to project a sample unicycle pose onto a continuous local neighborhood of its nearest neighbor among the sample unicycle poses.
A unicycle pose $(\pos, \ori)$ can be projected onto a continuous local neighborhood of another unicycle pose $(\hat{\pos}, \hat{\ori})$ using a unicycle distance measure $\unidist$ as 
\begin{align*}
\uniproj_{\deltaradius} ((\hat{\pos}, \hat{\ori}), (\pos, \ori))\!:=\! \hspace{-6mm} \argmin_{\substack{(\bar{\pos}, \bar{\ori}) \in  \R^{2}  \times  [-\pi, \pi) \\ \unidist((\bar{\pos}, \bar{\ori}), (\hat{\pos}, \hat{\ori})) \leq \deltaradius}}
 \hspace{-6mm} \unidist((\bar{\pos}, \bar{\ori}), (\pos, \ori)) 
\end{align*}  
where $\deltaradius \geq 0$ is a positive neighborhood radius.
Since finding an explicit expression for unicycle pose projection based on a generic distance measure is generally difficult, we alternatively consider a decoupled unicycle position and orientation projection based on the cross-product form of the local neighborhood using Euclidean and cosine distances as
\begin{align*}
&\uniproj_{(\deltapos, \deltaori)} ((\hat{\pos}, \hat{\ori}), (\pos, \ori))\\
& \quad := \plist{
\begin{array}{@{}c@{}}
\argmin\limits_{\substack{\bar{\pos} \in \R^2 \\ \norm{\bar{\pos} - \hat{\pos}} \leq \deltapos}} \norm{\bar{\pos} - \pos}, \argmin\limits_{\substack{\bar{\ori} \in [-\pi, \pi) \\ 1 - \scalebox{0.6}{$\tr{\ovectsmall{\bar{\ori}}} \ovectsmall{\hat{\ori}}$} \leq \deltaori}} \!\!\!\!\! \plist{\!1 \!-\! \tr{\ovectsmall{\bar{\ori}}\!}\! \ovectsmall{\ori}\!}
\end{array}
\!\!}
\end{align*}
where  $\deltapos \geq 0$ and $\deltaori \geq 0$ are positive maximum step sizes for translation and orientation changes, and 
the explicit form of Euclidean translation projection is given by
{\small
\begin{align*}
&\argmin_{\substack{\bar{\pos} \, \in\, \R^{2} \\ \norm{\bar{\pos} - \hat{\pos}} \leq \deltapos}} \norm{\bar{\pos} - \pos} = \left \{ \begin{array}{@{}l@{\,}l@{}}
\pos & \text{, if } \norm{\pos - \hat{\pos}} \leq \deltapos,
\\
\deltapos \frac{\pos - \hat{\pos}}{\norm{\pos - \hat{\pos}}} + \hat{\pos} & \text{, otherwise,}
\end{array}
\right.
\end{align*}
}%
and the cosine orientation projection can be obtained as
{\small
\begin{align*}
&\argmin_{\substack{\bar{\ori} \in [-\pi, \pi) \\  1- \scalebox{0.6}{$\tr{\ovectsmall{\bar{\ori}}\!}\! \ovectsmall{\hat{\ori}}$} \leq \deltaori}} \hspace{-4mm}  \plist{1 - \scalebox{0.85}{$\tr{\ovectsmall{\bar{\ori}}\!}\! \ovectsmall{\ori}$}}
\\ 
&\hspace{8mm} =\! \left \{
\begin{array}{@{}l@{}l@{}}
\ori & \text{, if } 1 - \scalebox{0.65}{$\tr{\ovectsmall{\ori}} \ovectsmall{\hat{\ori}}$} \leq \deltaori,\\
\ori \! + \! \cos^{-1}(\deltaori) & \text{, elseif } \scalebox{0.65}{$\tr{\ovectsmall{(\ori \! + \!\cos^{-1}(\deltaori))}\!}\! \ovectsmall{\hat{\ori}} \!\geq\! \tr{\ovectsmall{(\ori \!-\! \cos^{-1}(\deltaori))} \!} \! \ovectsmall{\hat{\ori}}$} ,
\\
\ori \!-\! \cos^{-1}(\deltaori) & \text{, else.}
\end{array}
\right .  
\end{align*}
}%
\subsubsection{Safety Verification via Unicycle Motion Prediction}

Safety verification of closed-loop motion under dual-headway control can be performed using the distance from feedback motion prediction to obstacles, which can be efficiently computed for occupancy grid maps using distance transforms and for environments with convex obstacles using convex optimization \cite{arslan_pacelli_koditschek_IROS2017}.
Hence, we check the safe reachability of a unicycle pose $(\hat{\pos}, \hat{\ori})$ from another unicycle pose $(\pos, \ori)$ using forward and backward dual-headway control and their associated control domains and motion prediction as   
{
\begin{align*}
\mathrm{issafe((\pos, \ori), (\hat{\pos}, \hat{\ori}))} 
= \left \{ 
\begin{array}{@{}r@{\hspace{1mm}}l@{}}
1, & \text{if }  \begin{array}{@{}c@{}}
(\pos, \ori) \in \fwddomain_{\hat{\pos}, \hat{\ori}}, \\ \conv(\pos, \head{\pos}, \tail{\hat{\pos}}, \hat{\ori}), \hat{\pos}) \subseteq \freespace
\end{array} 
\\
1, & \text{if } \begin{array}{@{}c@{}}
(\pos, \ori) \in \bwddomain_{\hat{\pos}, \hat{\ori}}
\\
\conv(\pos, \tail{\pos}, \head{\hat{\pos}}, \hat{\pos}) \subseteq \freespace
\end{array} 
\\
0, & \text{otherwise}
\end{array}
\right.
\end{align*}
}%
where the related headway and tailway points are  defined as in \refeq{eq.all_headway_tailway_points}. 
Note that the domains $\fwddomain_{\hat{\pos}, \hat{\ori}}$ and $\bwddomain_{\hat{\pos}, \hat{\ori}}$ of the forward and backward dual-headway unicycle controllers in \refeq{eq.forward_unicycle_control_domain} and \refeq{eq.backward_unicycle_control_domain}, respectively, are almost disjoint, with an intersection of measure zero, see \reffig{fig.dual_headway_motion_prediction}. 

\subsubsection{Local Cost of Unicycle Pose Control for Planning}


A local cost, denoted by $\localcost((\pos, \ori), (\hat{\pos}, \hat{\ori}))$, is a heuristic measure of the cost of unicycle steering control for moving from $(\pos, \ori)$ to $(\hat{\pos}, \hat{\ori})$ under the dual-headway unicycle control.     
Hence, any unicycle pose distance $\unidist$ in \reftab{tab.unidist} can be selected as a local cost heuristic, i.e., $\localcost = \unidist$, or one may simply use a uniform local cost for all valid motion, i.e., $\localcost=1$. 
As one might expect, to better align control and planning, it is technically more suitable to use the dual-headway unicycle distance associated with dual-headway unicycle control, as also demonstrated in numerical simulations in \refsec{sec.numerical_simulations}.

\subsection{Optimal Rapidly Exploring Random Trees for Unicycles}

In this section, we simply adapt the optimal rapidly-exploring random trees algorithm from \cite{karaman_frazzoli_IJRR2011}, as described in \refalg{alg.optimal_unicycle_motion_planning}. 
For ease of notation, we denote a unicycle pose more compactly as $\unipose = (\pos, \ori)$ and introduce below a few graph-theoretical tools for the algorithmic description.

Let $\graph = \plist{\vertexset, \edgeset}$ denote an undirected, connected graph of unicycle poses $\vertexset$, where a pair of unicycle poses $\unipose, \hat{\unipose} \in \vertexset$ are connected if there is an edge $(\unipose, \hat{\unipose}) \in \edgeset$ with an edge cost of $\localcost(\unipose, \hat{\unipose}) > 0$. 
Let $\cost_{\graph}\plist{\uniposestart, \uniposegoal}$ denote the minimum total cost of traveling from $\uniposestart$ to $\uniposegoal$ over the graph $\graph$, which can be computed using standard graph search methods such as Dijkstra's algorithm.
Hence, we define the parent of a unicycle pose $\unipose \in \vertexset$ as its neighbor with the minimum total travel cost to a given start pose $\uniposestart \in \vertexset$ over the motion graph $\graph = (\vertexset, \edgeset)$ as 
\begin{align*}
\parent_{\graph}(\uniposestart,\unipose):= \argmin_{\substack{\hat{\unipose} \, \in\, \vertexset \\ (\hat{\unipose}, \unipose) \, \in\, \edgeset}} \cost_{\graph}(\uniposestart, \hat{\unipose}). 
\end{align*} 
Accordingly, using the essential building blocks described in \refsec{sec.elements_of_sampling_based_planning}, we present in \refalg{alg.optimal_unicycle_motion_planning} how to build a connected optimal rapidly-exploring random tree for sampling-based unicycle feedback motion planning.

\begin{algorithm}[h]
\caption{Optimal Rapidly-Exploring Random Tree\!\!}
\label{alg.optimal_unicycle_motion_planning}
\begin{footnotesize}
\SetKwInOut{Input}{Input}  
\Input{ $\freespace$ -- Collision-Free Space \\
$\uniposestart = (\pos, \ori) \in \freespace \times [-\pi,\pi)$ -- Start Unicycle Pose \\
$\uniposegoal  = (\goalpos, \goalori) \in \freespace \times [-\pi,\pi)$ -- Goal Unicycle Pose \\
$\NumSample$ -- Number of Iterations\\
}
\KwOut{$\graph=\plist{\vertexset, \edgeset}$ -- \mbox{Random Motion Planning~Graph\hspace{-10mm}} 
\\
\vspace{-2mm}
\hrulefill
}

$\vertexset \leftarrow \clist{\uniposestart}$; $\edgeset \leftarrow \varnothing$\\
\For{$i = 1, \ldots, \NumSample$}{
$\uniposerand \leftarrow \unisample\plist{\freespace \! \times \! [\pi, \pi), \uniposegoal}$\\
$\uniposenearest \leftarrow \uninearest_{\vertexset}\plist{\uniposerand}$\\
$\uniposenew \leftarrow \uniproj\plist{\uniposenearest, \uniposerand}$\\
\If{$\issafe\plist{\uniposenearest, \uniposenew}$}{
$\uniposemin \leftarrow \uniposenearest$\\
$\mathrm{mincost} \leftarrow \mathrm{cost}_{\graph}(\uniposestart, \uniposenearest) + \localcost(\uniposenearest, \uniposenew)$\\
$\uniposeneighbor \leftarrow \unineighbor_{\vertexset}(\uniposenew) $\\
\For{$\uniposenear \in \uniposeneighbor$}{
\mbox{$\mathrm{tempcost} \leftarrow \mathrm{cost}_{\graph}(\uniposestart, \uniposenear) + \localcost(\uniposenear, \uniposenew)$}\\
\If{$(\mathrm{tempcost} < \mathrm{mincost}) \wedge \issafe(\uniposenear, \uniposenew)$}{
$\uniposemin \leftarrow \uniposenear$\\
$\mathrm{mincost} \leftarrow \mathrm{tempcost}$\\
}
}
$\vertexset \leftarrow \vertexset \cup \clist{\uniposenew}$; $\edgeset \leftarrow \edgeset \cup \clist{\plist{\uniposemin, \uniposenew}}$\; 
\For{$\uniposenear \in \uniposeneighbor$}{
$\mathrm{mincost} \leftarrow \mathrm{cost}_{\graph}(\uniposestart, \uniposenear)$\\
\mbox{$\mathrm{tempcost} \leftarrow \mathrm{cost}_{\graph}\plist{\uniposestart, \uniposenew} + \localcost(\uniposenew, \uniposenear)$}\\
\If{\mbox{$(\mathrm{tempcost} < \mathrm{mincost}) \wedge  \issafe(\uniposenew, \uniposenear)$}}{
$\uniposeparent \leftarrow \parent_{\graph}(\uniposestart,\uniposenear)$ \\
\mbox{$\edgeset \leftarrow \edgeset   \cup  \clist{\plist{\uniposenew, \uniposenear}} \! \setminus \! \clist{\plist{\uniposeparent, \uniposenear}}$}\\
}
}
} 
}
\Return{$\graph=\plist{\vertexset, \edgeset}$}\\
\end{footnotesize}
\end{algorithm}

\subsection{Unicycle Feedback Control for Planning Execution}

One known advantage of feedback motion planning, which tightly couples planning and control, is that executing a plan is almost effortless.
More specifically, given a connected motion graph  $\graph=(\vertexset, \edgeset)$, constructed, for example, by \refalg{alg.optimal_unicycle_motion_planning}, we select a safely reachable local goal $(\hat{\pos}, \hat{\ori})$ at any given collision-free unicycle pose $\unipose = (\pos, \ori) \in \freespace \times [-\pi, \pi)$ by minimizing the total travel cost to reach a global goal $\uniposegoal = (\goalpos, \goalori) \!\in\! \vertexset$ as 
{\small
\begin{align*}
(\hat{\pos}, \hat{\ori})\! & = \mathrm{localgoal}_{\graph}(\unipose)  \!:=\! \!\!\!\argmin_{\substack{\hat{\unipose} \in \vertexset \\ \issafe(\unipose, \hat{\unipose})=1}} \! \!\!\localcost(\unipose, \hat{\unipose}) \!+\! \cost_{\graph}(\hat{\unipose}, \uniposegoal)
\end{align*}
}%
and then continuously navigate toward the local goal using forward or backward dual-headway unicycle control as
{\small
\begin{align*}
\ctrl_{\graph}(\pos, \ori) = \left \{ \begin{array}{@{}c@{}c@{}}
\fwdctrl_{\hat{\pos}, \hat{\ori}}(\pos, \ori) & \text{, if } (\pos, \ori) \in \fwddomain_{\hat{\pos}, \hat{\ori}}, \\
\bwdctrl_{\hat{\pos}, \hat{\ori}}(\pos, \ori) & \text{, if } (\pos, \ori) \in \bwddomain_{\hat{\pos}, \hat{\ori}}. 
\end{array}
\right.
\end{align*}
}%
It is important to observe that if any unicycle pose in the motion graph $\graph$ is safely reachable from the current unicycle pose $(\pos, \ori)$, then the unicycle remains connected to the motion graph for all future times and reaches the global goal pose $\uniposegoal = (\goalpos, \goalori)$ by visiting a finite number of intermediate unicycle poses, each with a decreasing travel cost to the global goal.
In other words, the unicycle motion graph corresponds to a collection of adjacent local control policies, associated with and abstracted by unicycle poses as control goals; and the minimum-cost local goal selection serves as a sequential composition strategy to systematically combine local navigation policies to achieve global navigation \cite{burridge_rizzi_koditschek_IJRR1999}.


\section{Numerical Simulations}
\label{sec.numerical_simulations}

In this section, we present numerical simulations demonstrating the influence of optimization objective, neighborhood size, and number of samples on the quality of the resulting unicycle motion. 
We also present the use of informed sampling and pruning for efficient and effective search and exploration in complex environments by balancing the trade-off between exploration and exploitation.\footnote{For all simulations, unless otherwise specified, we set the headway and tailway coefficients identically as  $\headtailcoef = \tfrac{1}{3}$; $(\Deltapos,\Deltaori)$-neighborhood is defined using the Euclidean translation and cosine orientation distances with $\Deltapos = 1.5$ and $\Deltaori=1 - \cos(\tfrac{\pi}{3})$; the $(\deltapos, \deltaori)$-projection parameters are set as $\deltapos = 1$ and $\deltaori = 1 - \cos(\tfrac{\pi}{6})$; the additive weights of dual-headway translation and orientation distances are selected as $\alpha =1$ and $\beta= 10$.} 

\begin{figure}[t]
\begin{tabular}{@{}c@{\hspace{0.5mm}}c@{\hspace{0.5mm}}c@{}}
\includegraphics[width = 0.33\columnwidth]{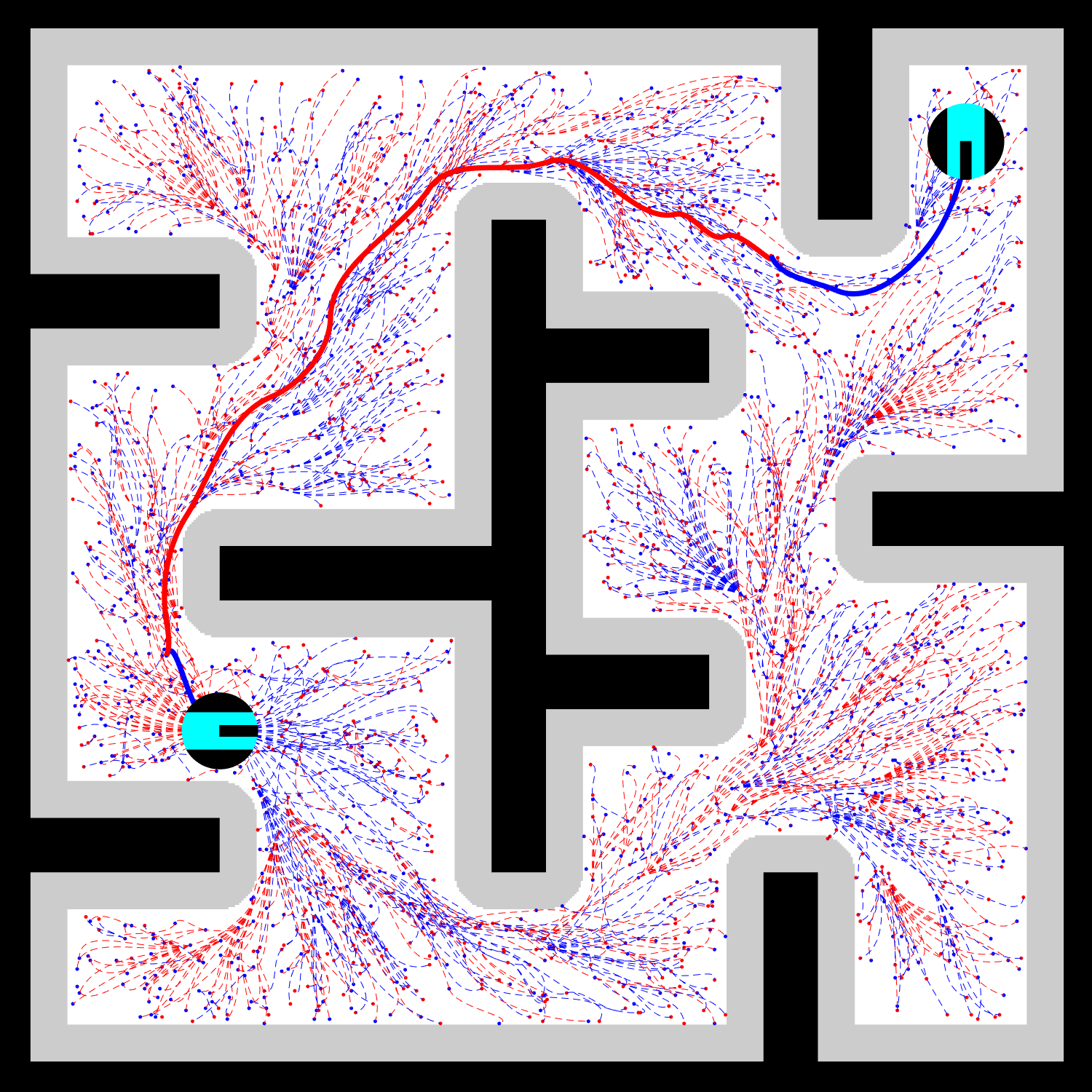} & \includegraphics[width = 0.33\columnwidth]{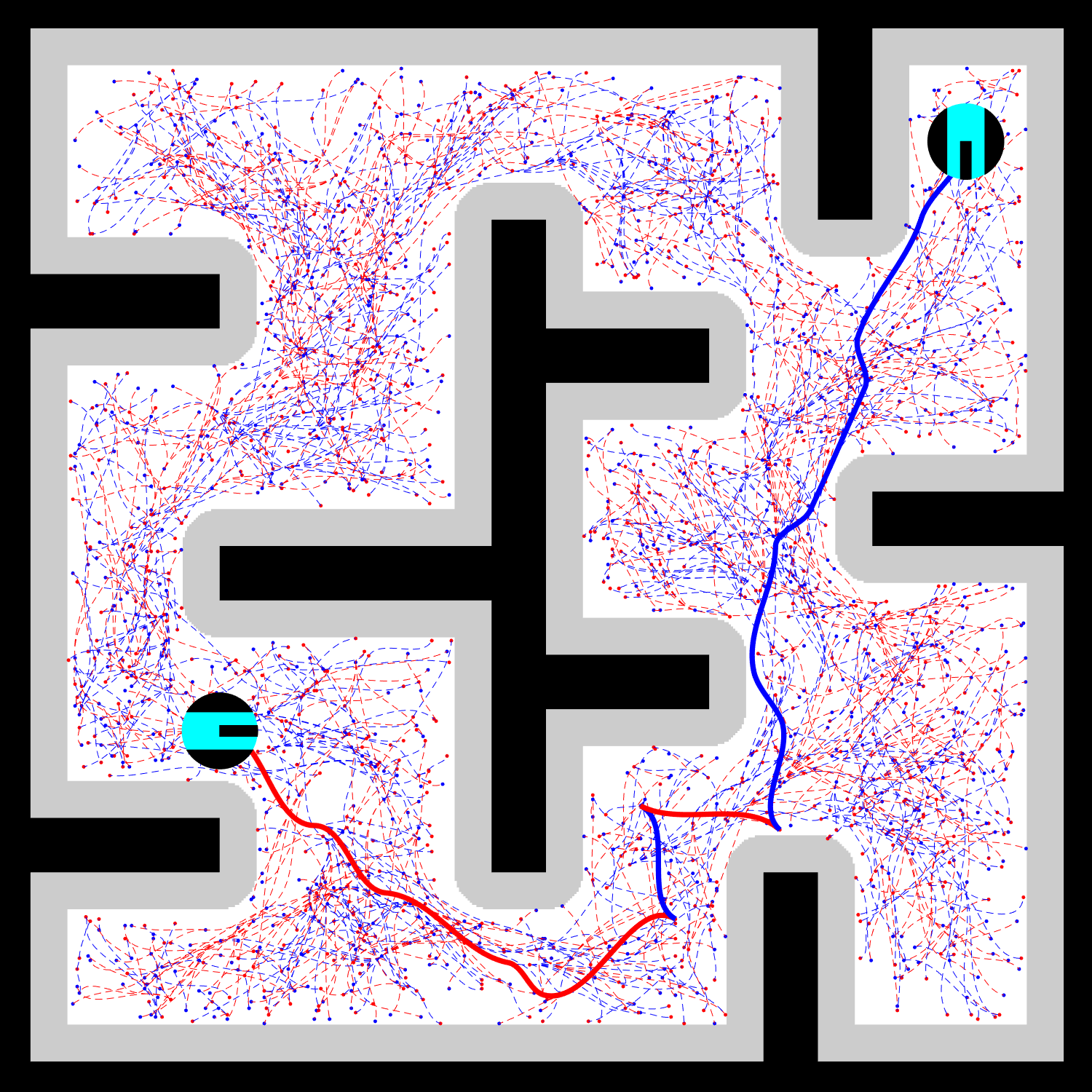} & \includegraphics[width = 0.33\columnwidth]{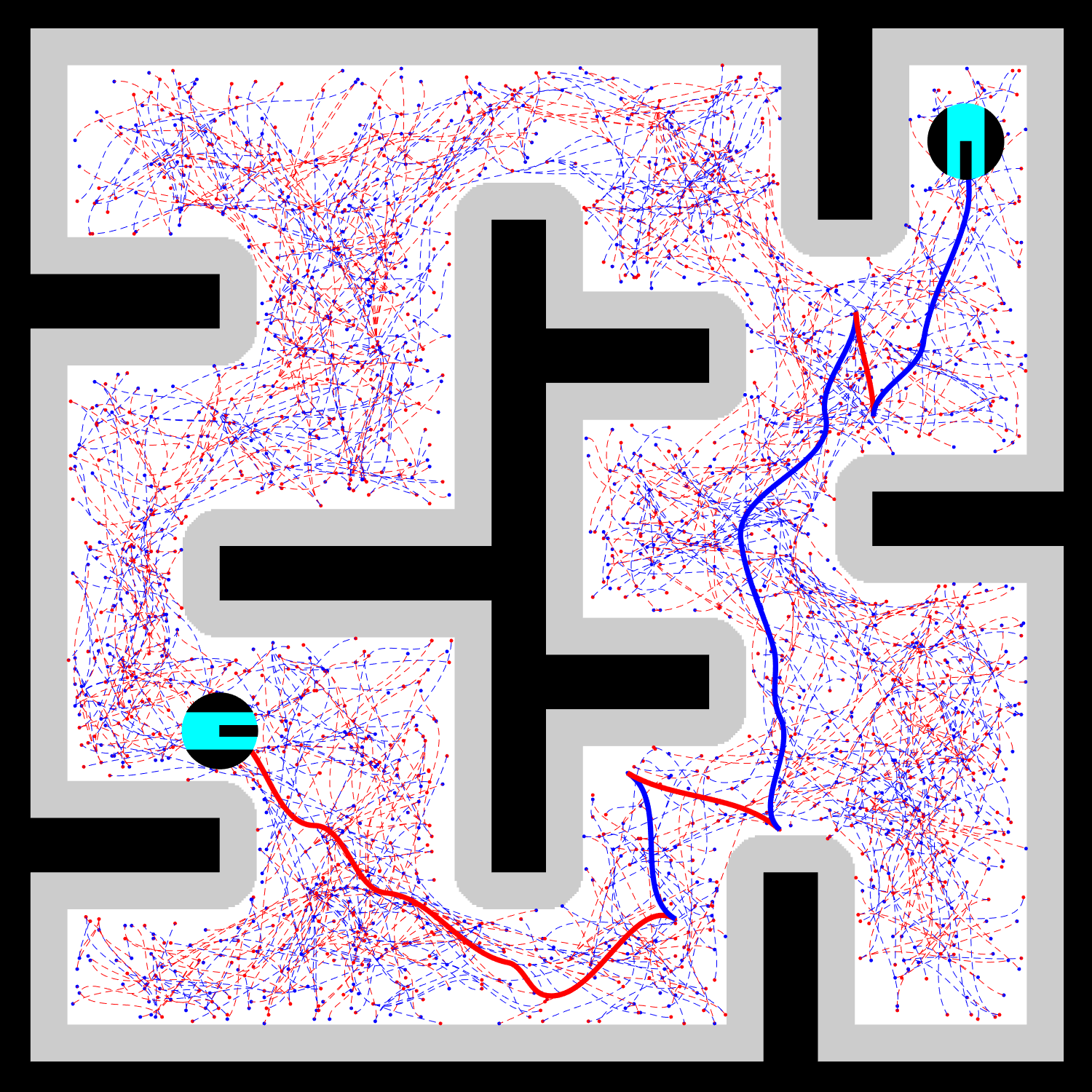}
\\
\includegraphics[width = 0.33\columnwidth]{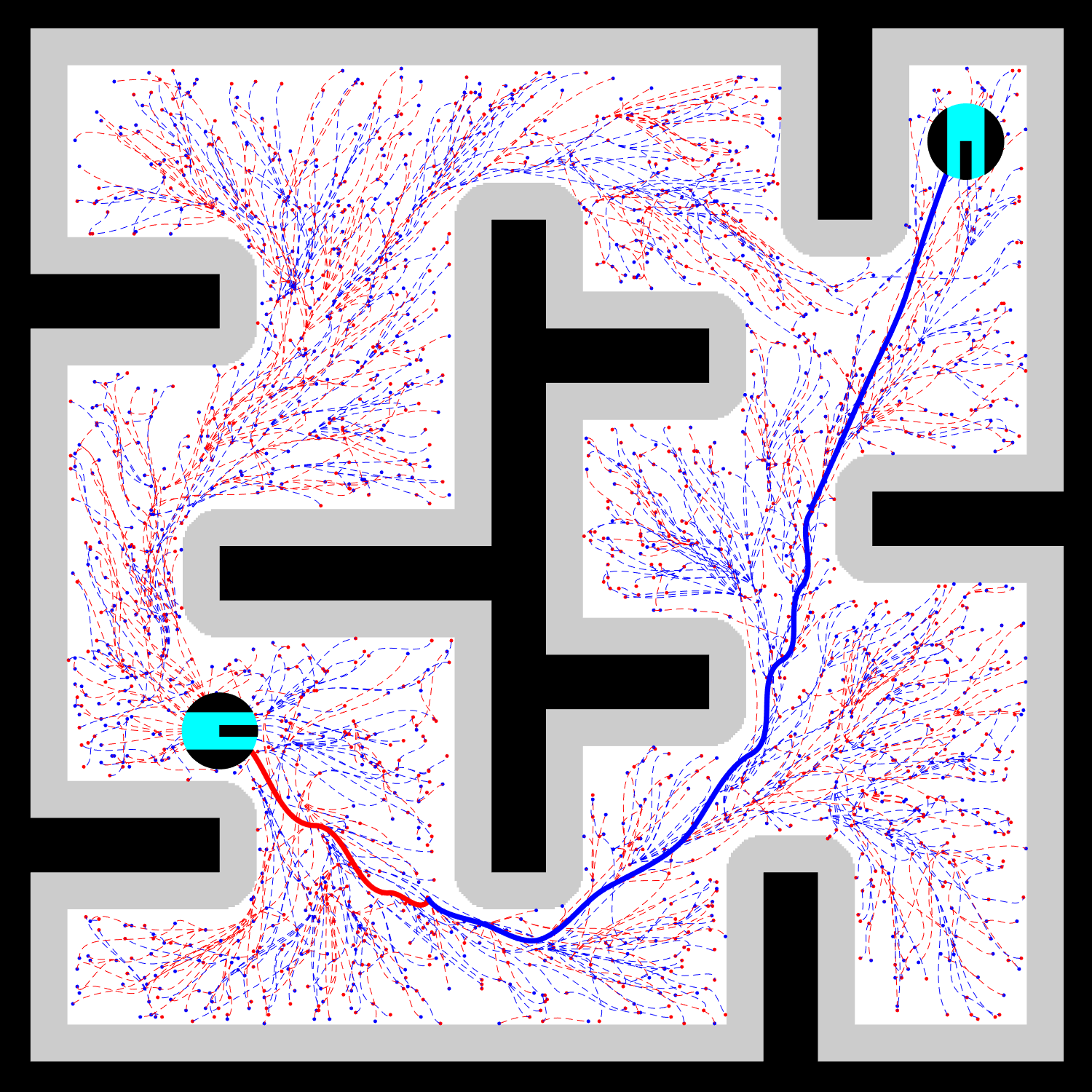} & \includegraphics[width = 0.33\columnwidth]{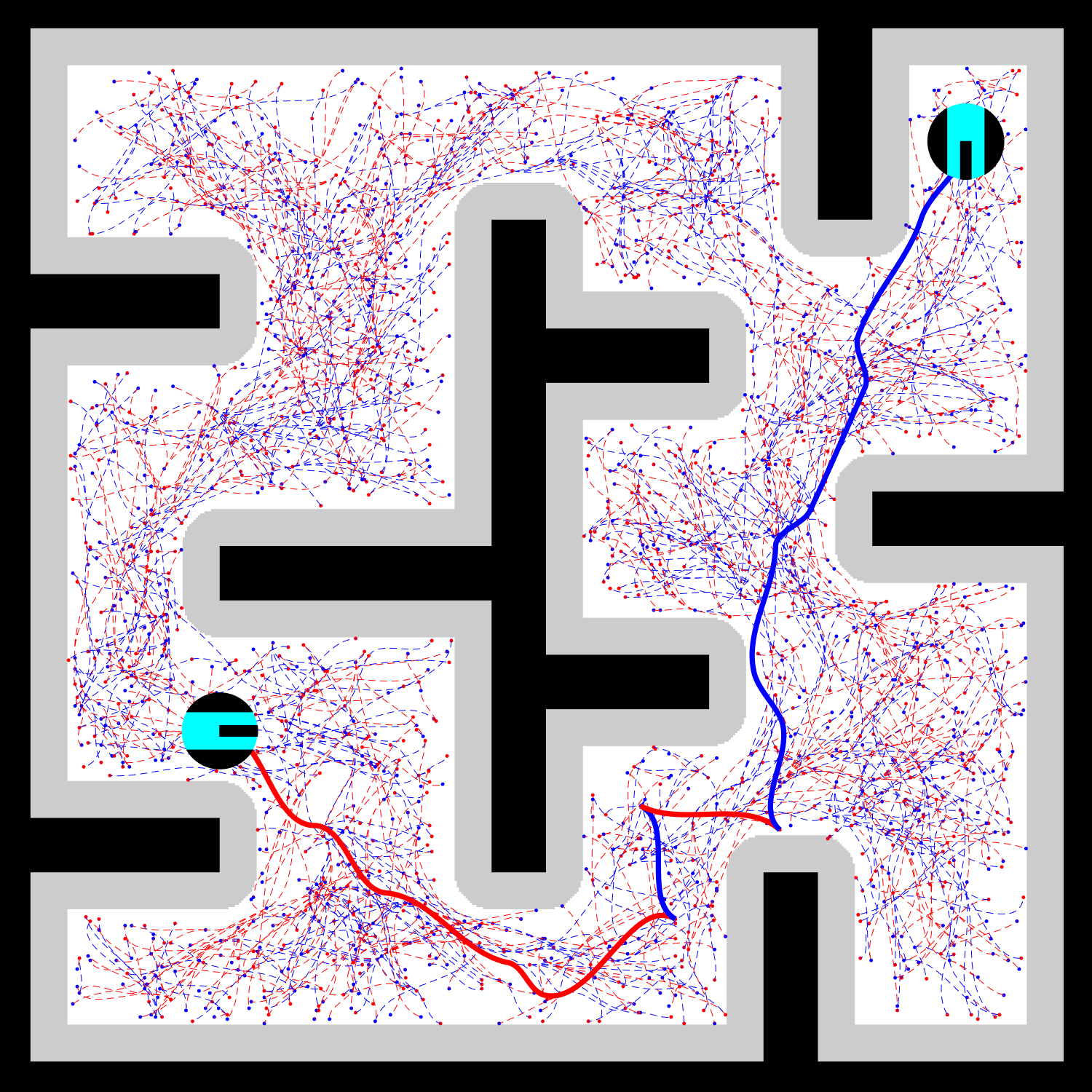} & \includegraphics[width = 0.33\columnwidth]{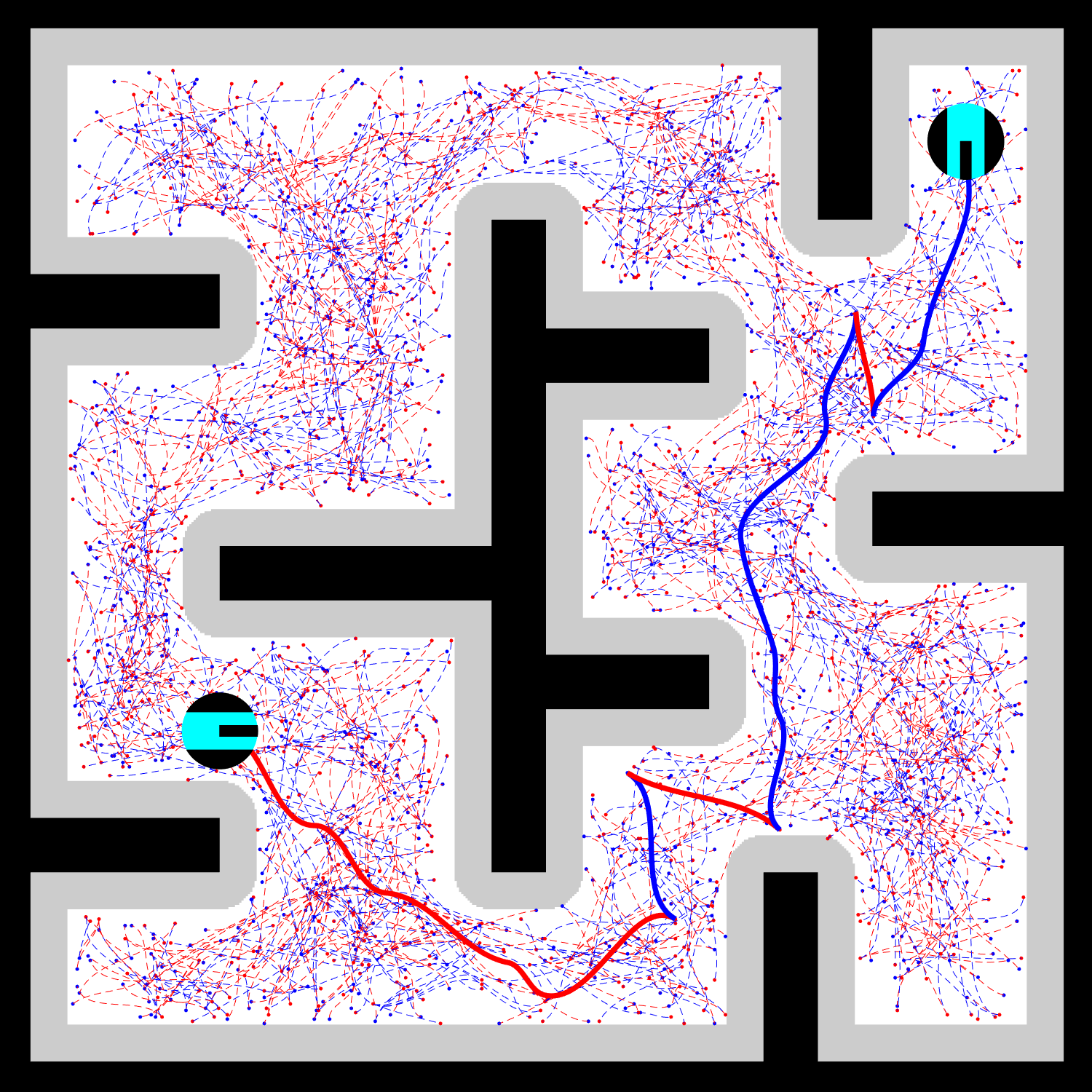}
\\
\includegraphics[width = 0.33\columnwidth]{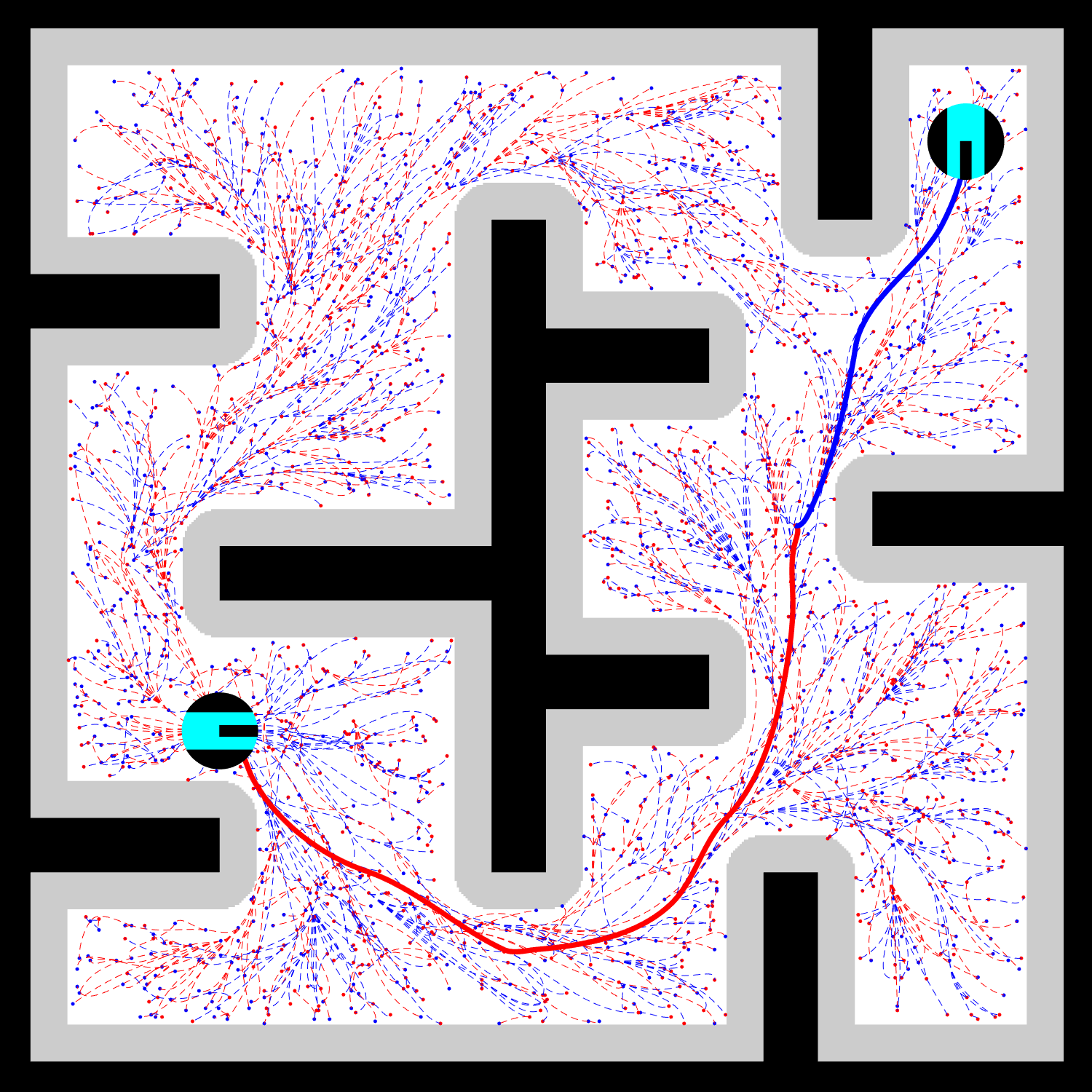} & \includegraphics[width = 0.33\columnwidth]{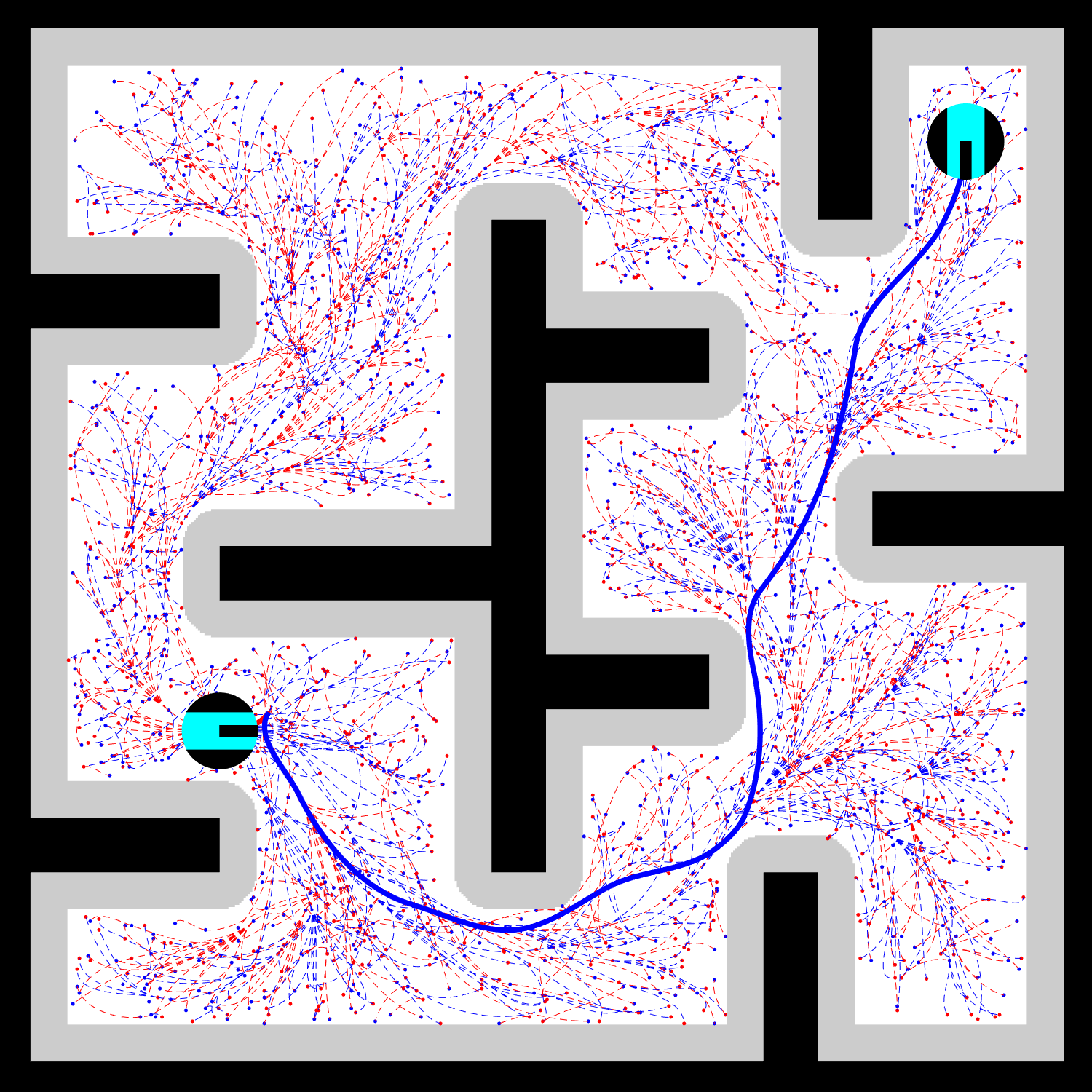} & \includegraphics[width = 0.33\columnwidth]{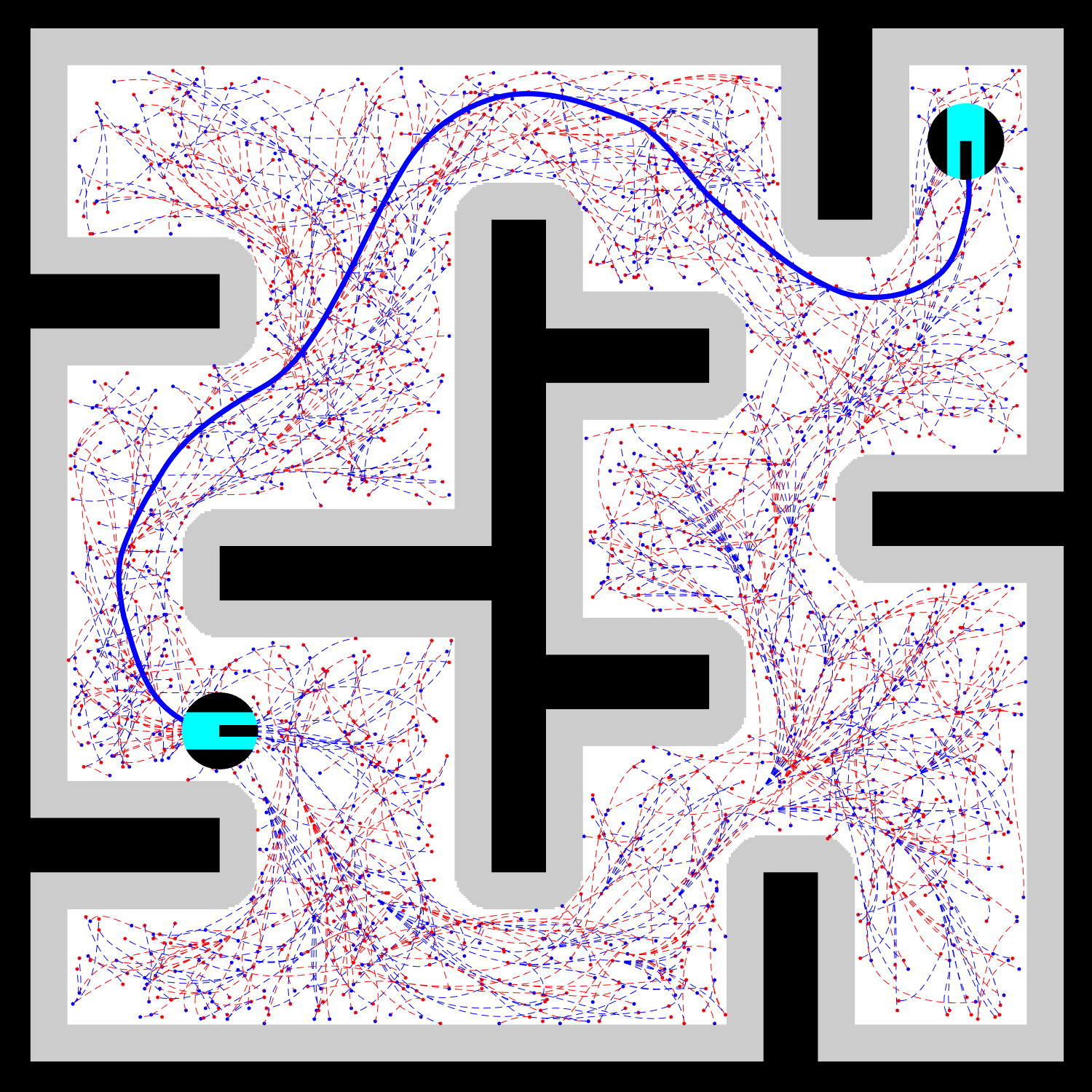}
\end{tabular}
\vspace{-4mm}
\caption{Optimal sampling-based unicycle feedback motion planning examples using forward and backward dual-headway motion primitives for different orientation distance weight and different translation and orientation distances: (top) weighted combination of Euclidean translation and cosine orientation distance, (middle) Euclidean-Cosine translation and Cosine orientation, (bottom) dual-headway translation and orientation distances; (left) $\alpha = 1.0, \beta=0.0$, (center) $\alpha = 1.0, \beta=2.0$, (right) $\alpha=1.0, \beta=10.0$.}
\label{fig.influence_of_optimization_objective}
\vspace{-3mm}
\end{figure}

\subsection{Influence of Optimization Objective on Unicycle Motion}

The optimal sampling-based unicycle motion planning approach in \refalg{alg.optimal_unicycle_motion_planning} uses randomized exploration with rewiring of local connections to minimize the total cost of a path joining the start and goal poses.
Thus, selecting an appropriate unicycle pose distance as the local connectivity cost plays a key role in the quality of the resulting optimal motion plan.
In \reffig{fig.influence_of_optimization_objective}, we present the optimal unicycle motion paths for three different additively weighted unicycle translation and rotation distances  under three different weight settings, using the same number of samples and the same neighborhood.
When there is no turning penalty (i.e., $\alpha = 1, \beta = 0$), as shown in \reffig{fig.influence_of_optimization_objective} (left), we consistently observe in numerical simulations that all translation distances, Euclidean, Euclidean-Cosine, and dual-headway translation distances, result in similar travel distances; however, smoothness slightly increases from Euclidean to Euclidean-Cosine to dual-headway distance, although all may exhibit sharp turns (i.e., turning in place) along the resulting motion.
With an increasing turning penalty (i.e., $\beta = 0$, $\beta = 2$, and $\beta = 10$ from left to right in \reffig{fig.influence_of_optimization_objective} while keeping $\alpha = 1$), we observe that optimization using dual-headway translation and orientation distances, as shown in \reffig{fig.influence_of_optimization_objective} (bottom), produces increasingly smoother unicycle paths, fully avoiding sharp turns and unnecessary changes in motion direction. 
In contrast, optimization based on cosine orientation distance in \reffig{fig.influence_of_optimization_objective} (top, middle) tends to favor similar orientations between adjacent unicycle poses and so results in a zigzag-like motion pattern between parallel unicycle poses at different positions.
In summary, as demonstrated in practice \cite{kuwata_teo_fiore_karaman_frazzoli_how_TCST2009, park_kuipers_IROS2015}, selecting an appropriate optimization objective using Lyapunov-like control-aware measures enhances planning effectiveness by bridging the gap between control and planning.   

\begin{figure}[t]
\begin{tabular}{@{}c@{\hspace{1mm}}c@{\hspace{1mm}}c@{}}
 \includegraphics[width = 0.33\columnwidth]{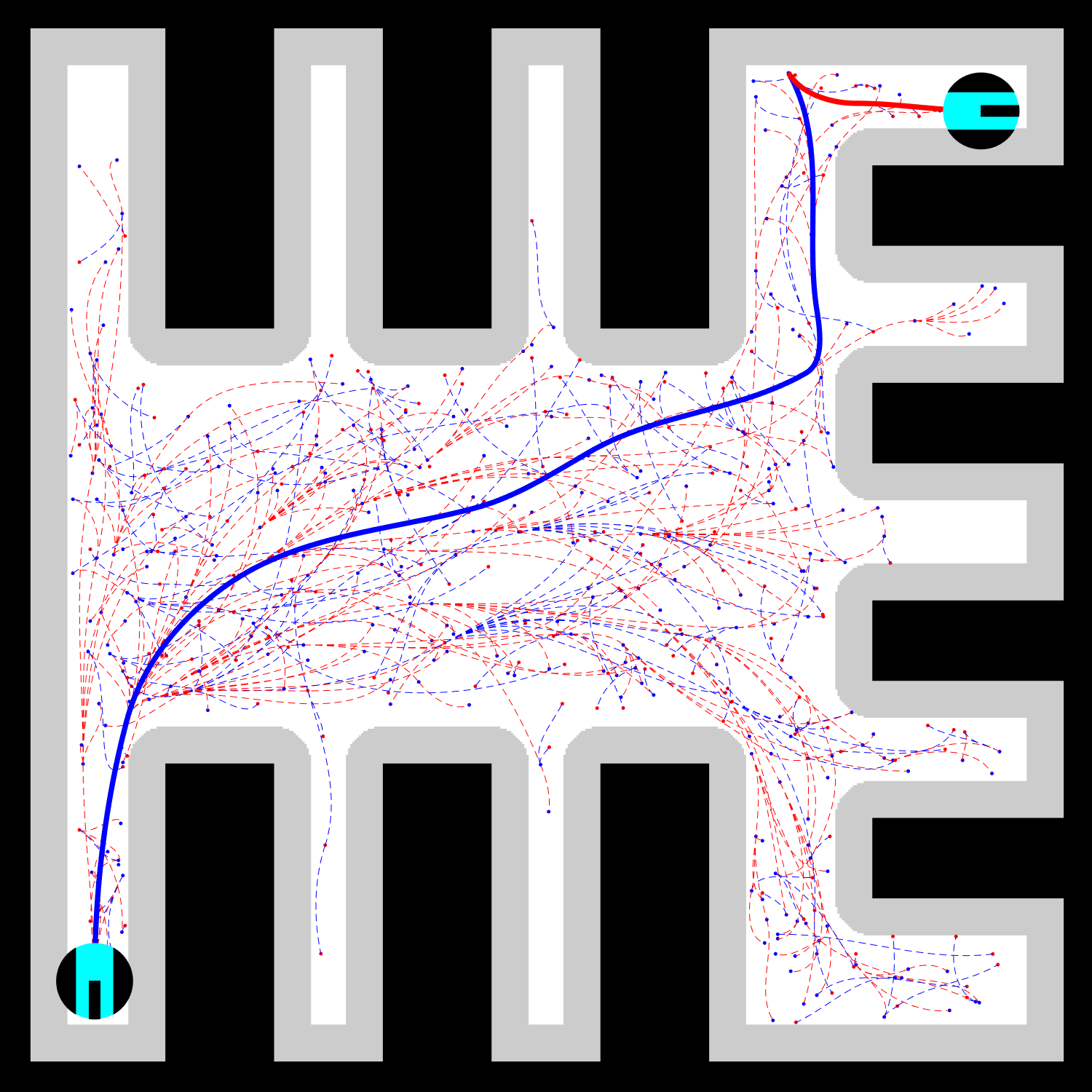} & \includegraphics[width = 0.33\columnwidth]{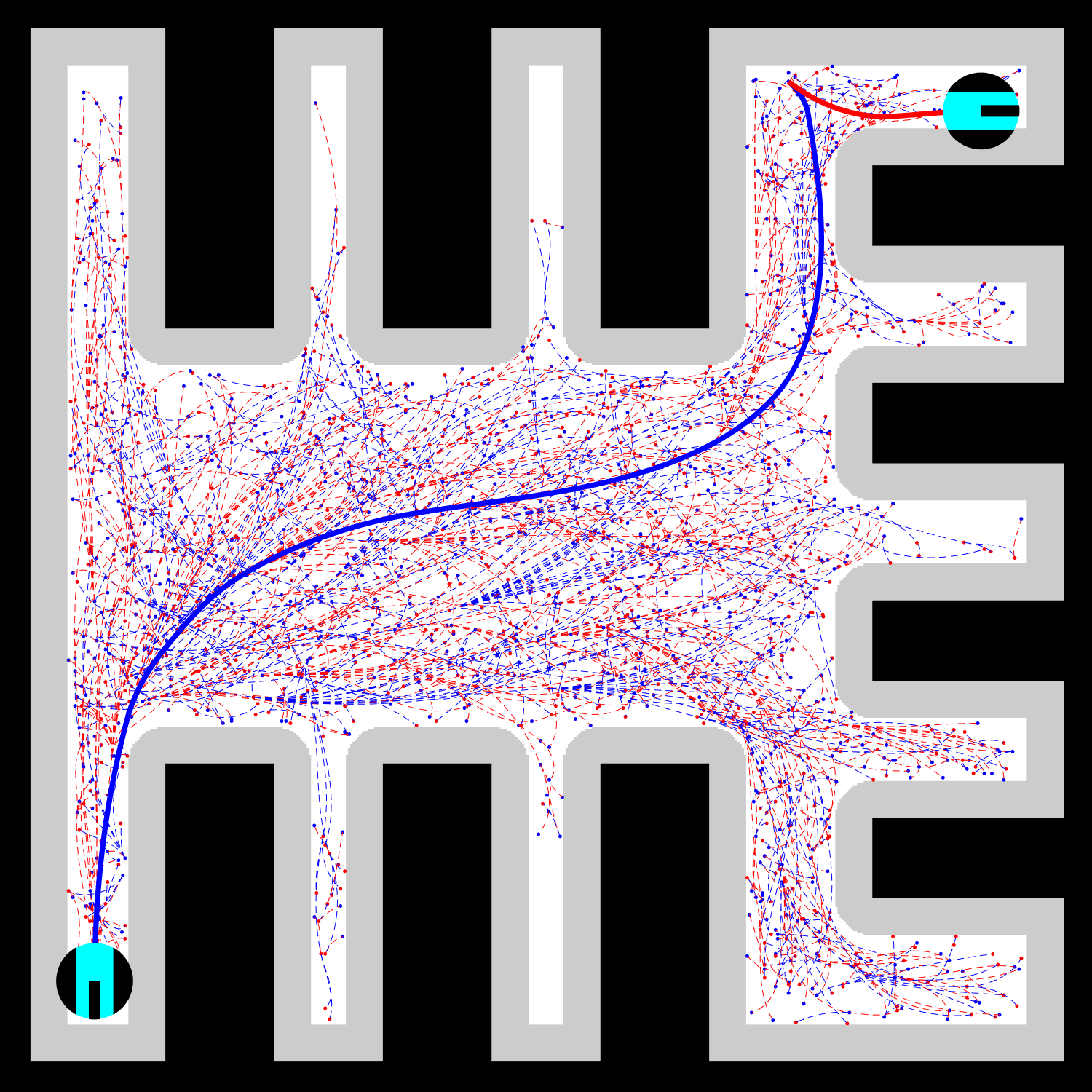} &  \includegraphics[width = 0.33\columnwidth]{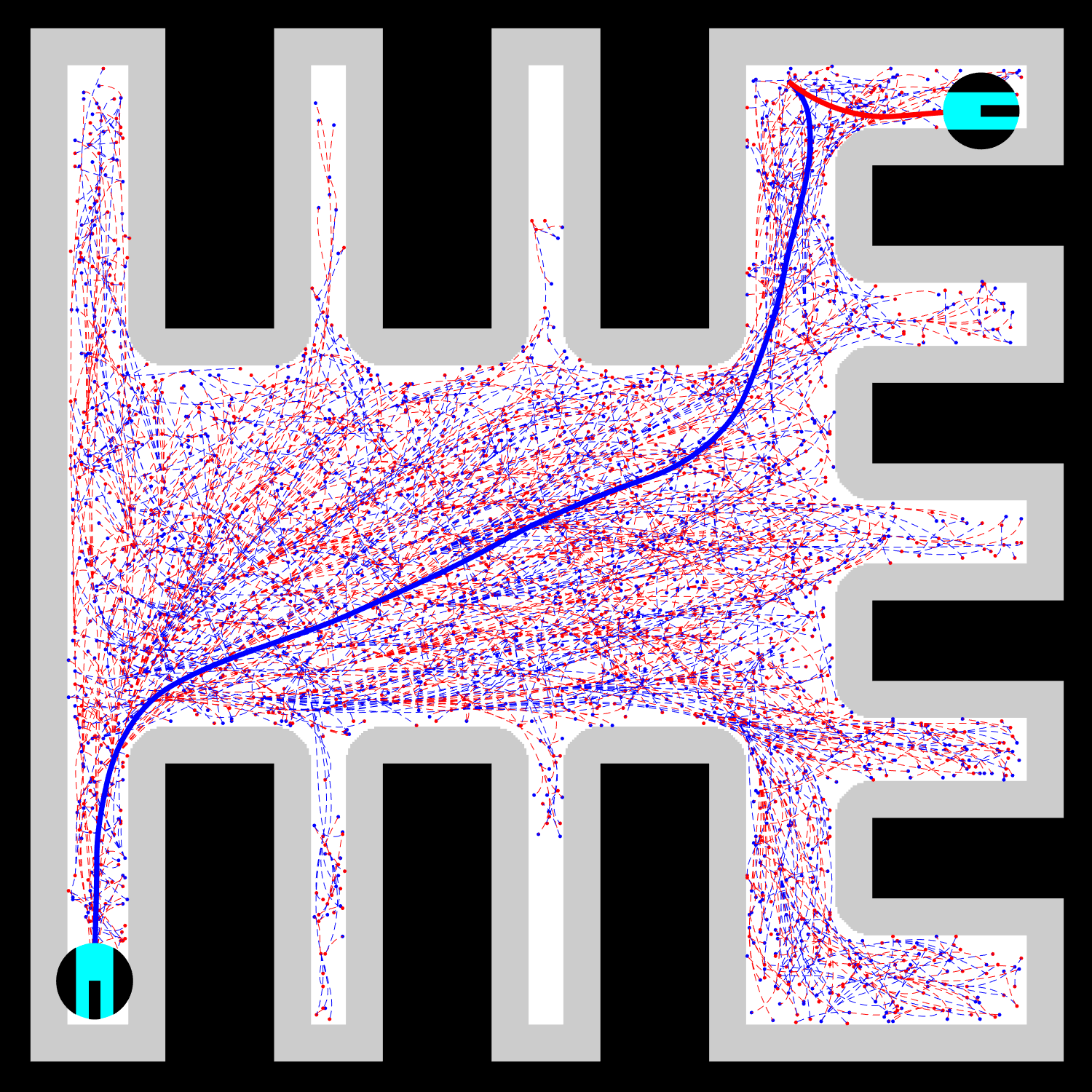} 
\\
\includegraphics[width = 0.33\columnwidth]{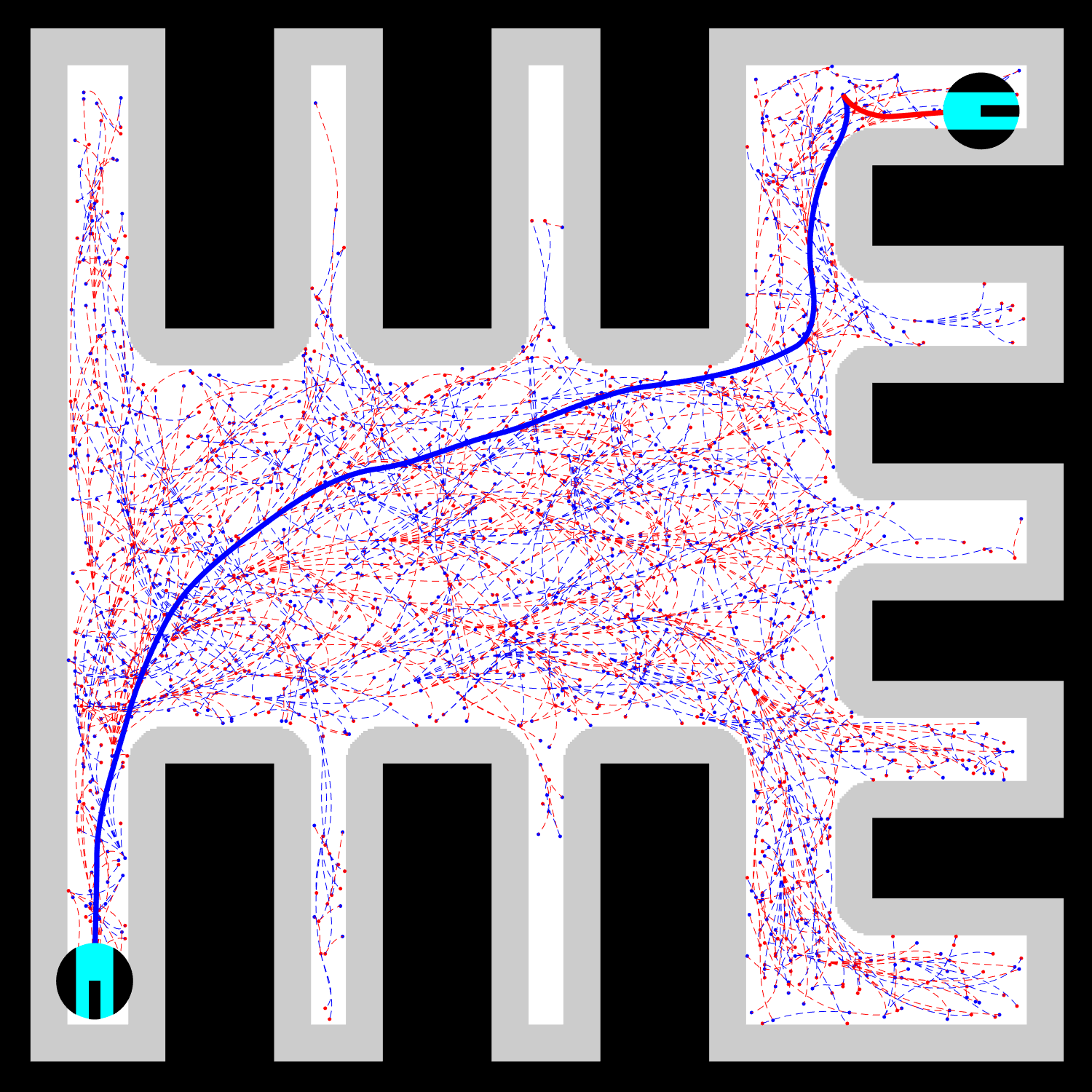} & \includegraphics[width = 0.33\columnwidth]{figures/RRT_env_parking-objective_dualheadway-beta_10-iteration_3000-nhood_medium.eps} & \includegraphics[width = 0.33\columnwidth]{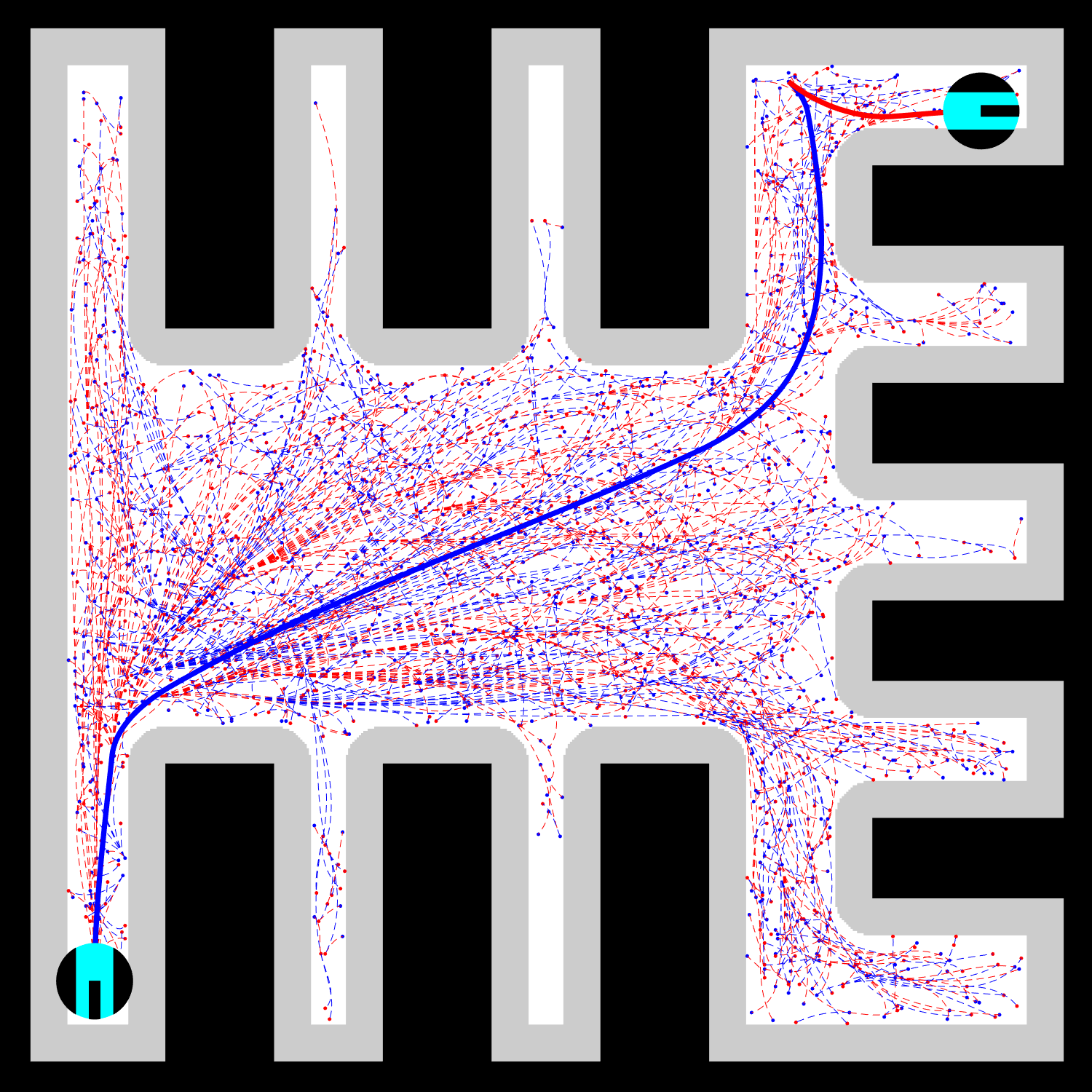}
\end{tabular}
\vspace{-3.5mm}
\caption{Optimal sampling-based unicycle feedback motion planning examples using forward and backward dual-headway motion primitives, minimizing additive dual-headway translation and orientation distance for different numbers of samples and neighborhood parameters: (top, left) $N = 1000$, (top, middle) $N = 3000$, (top, right) $N = 5000$; (bottom, left) $(\Deltapos, \Deltaori)= (1.5, 1 - \cos(\tfrac{\pi}{3})$, (bottom, middle) $(\Deltapos, \Deltaori)= (3, 1 - \cos(\tfrac{\pi}{2}))$, (bottom, right) $(\Deltapos, \Deltaori)= (6, 1 - \cos(\pi))$.}
\label{fig.influence_of_neighborhood_size_number_of_samples}
\vspace{-3.5mm}
\end{figure}

\subsection{Influence of Neighborhood Size and Number of Samples}

Two other important factors that affect the quality of optimal sampling-based planning are the neighborhood size for local optimal rewiring of the motion graph in \refalg{alg.optimal_unicycle_motion_planning}(lines 10-21) and the number of samples, $\NumSample$, which determines the optimization iterations with new samples.
In \reffig{fig.influence_of_neighborhood_size_number_of_samples}, we present example numerical simulations using the $(\Deltapos,\Deltaori)$-neighborhood with Euclidean translation and cosine orientation distances. 
In our numerical studies, we observe that a similar neighborhood size can be achieved with different unicycle pose distances by adjusting $\Deltapos$ and $\Deltaori$. 
Thus, while the specific choice of a unicycle pose distance for neighborhood determination is not significantly influential in optimal motion planning, the neighborhood size (e.g., volume) is essential. 
As expected and shown in \reffig{fig.influence_of_neighborhood_size_number_of_samples}, the quality of the optimal unicycle motion plan (with dual-headway translation and rotation distances, where $\alpha = 1$, $\beta = 10$) improves as the neighborhood size and number of samples increase, with increasing computational cost. 
Consequently, one can compute more for smoother and more optimal motion trajectories.

\subsection{Informed Sampling and Pruning in  Motion Planning}

The Voronoi bias of randomized motion planning for exploration poses a challenge for motion planning in high-dimensional, complex, and cluttered configuration spaces \cite{lavalle_kuffner_IJRR2001, karaman_frazzoli_IJRR2011}, such as unicycle pose planning around obstacles.
Informed sampling and pruning with heuristics \cite{gammell_srivinasa_barfoot_IROS2014} enable effective and efficient exploration of complex environments in sampling-based optimal unicycle feedback motion planning, as illustrated in \reffig{fig.informed_motion_planning}.
In particular, we use informed sampling and pruning with heuristics to eliminate sample unicycle poses from a randomized motion graph through which no optimal path exists.
We perform informed sampling by rejecting sample unicycle poses whose travel cost, via their nearest neighbor,  to the start pose (plus their heuristic cost to the global pose) is larger than the travel cost from the start pose to the goal pose in a motion graph. \footnote{In informed sampling, we reject a new sample unicycle pose without performing lines 15-21 in \refalg{alg.optimal_unicycle_motion_planning} if $\mathrm{mincost} + \mathrm{heuristic}(\uniposenew, \uniposegoal) > \cost_{\graph}(\uniposestart, \uniposegoal)$. Similarly, we perform informed pruning of a node $\unipose \in \vertexset$ of a motion graph $\graph=(\vertexset, \edgeset)$ by checking if $\cost_{\graph}(\uniposestart, \unipose) + \mathrm{heuristic}(\unipose, \uniposegoal) > \cost_{\graph}(\uniposestart, \uniposegoal)$ after each node insertion.}
As an admissible heuristic bounding the total travel cost between unicycle poses ($\unipose=(\pos, \ori)$ and $\unipose'=(\pos', \ori')$) from below, we consider the zero heuristic (i.e., $\mathrm{heuristic}(\unipose, \unipose') = 0$) and the Euclidean heuristic (i.e., $\mathrm{heuristic}(\unipose, \unipose') = \norm{\pos - \pos'}$).
As seen in \reffig{fig.informed_motion_planning}, once a path between the start and goal is found, informed sampling and pruning focus exploration in a smaller region of the configuration space to effectively and efficiently find an optimal path.
As expected, the informativeness of the heuristic plays a significant role in the computational gain with informed sampling and pruning. 
Although the Euclidean distance is not a proper optimization objective for minimizing total travel distance and turning effort in optimal unicycle motion planning (see \reffig{fig.influence_of_optimization_objective}), as a lower bound on the dual-headway unicycle distance, we observe in \reffig{fig.informed_motion_planning} that the Euclidean distance is an informative and useful heuristic for sampling-based planning of optimal unicycle motion using the dual-headway distance.

\begin{figure}[t]
\centering
\begin{tabular}{@{}c@{\hspace{0.5mm}}c@{\hspace{0.5mm}}c@{}}
\includegraphics[width = 0.33\columnwidth]{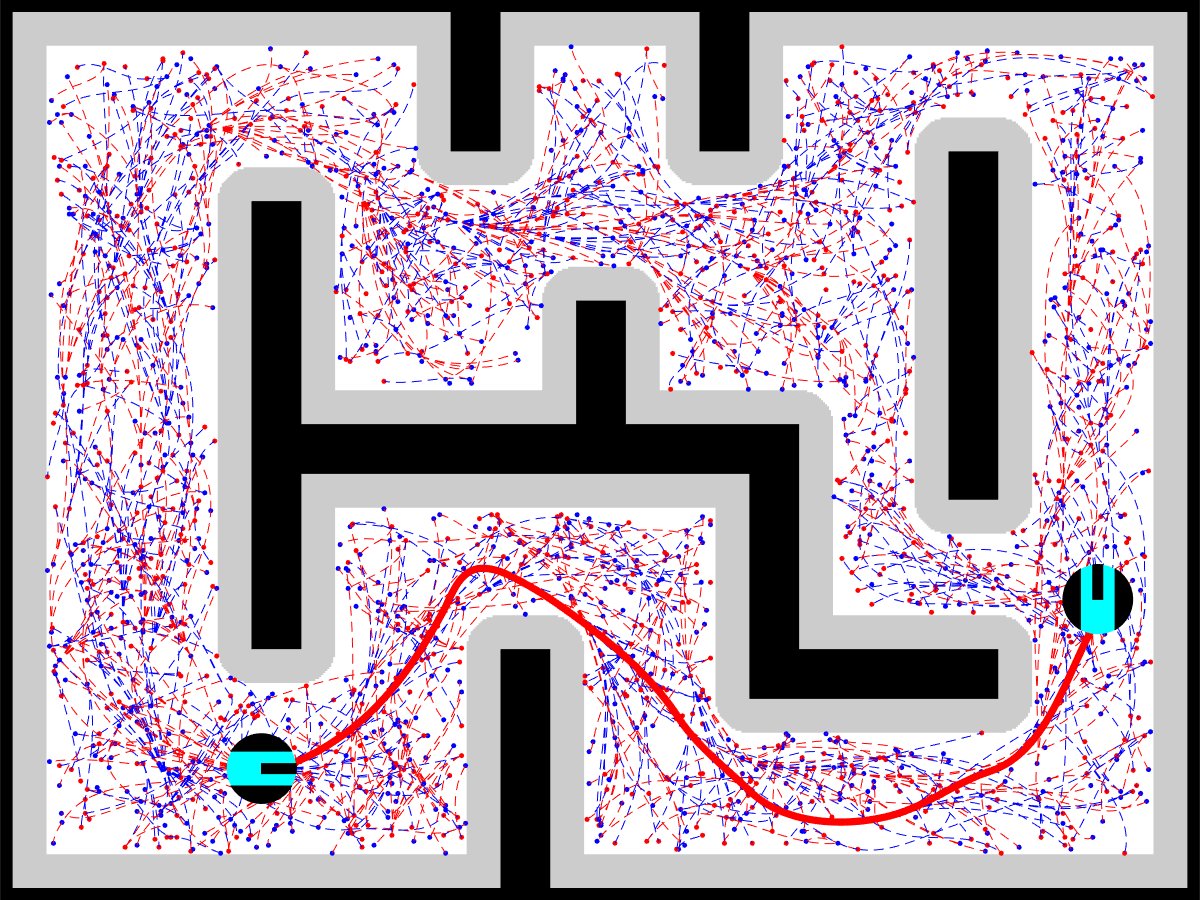} &
\includegraphics[width = 0.33\columnwidth]{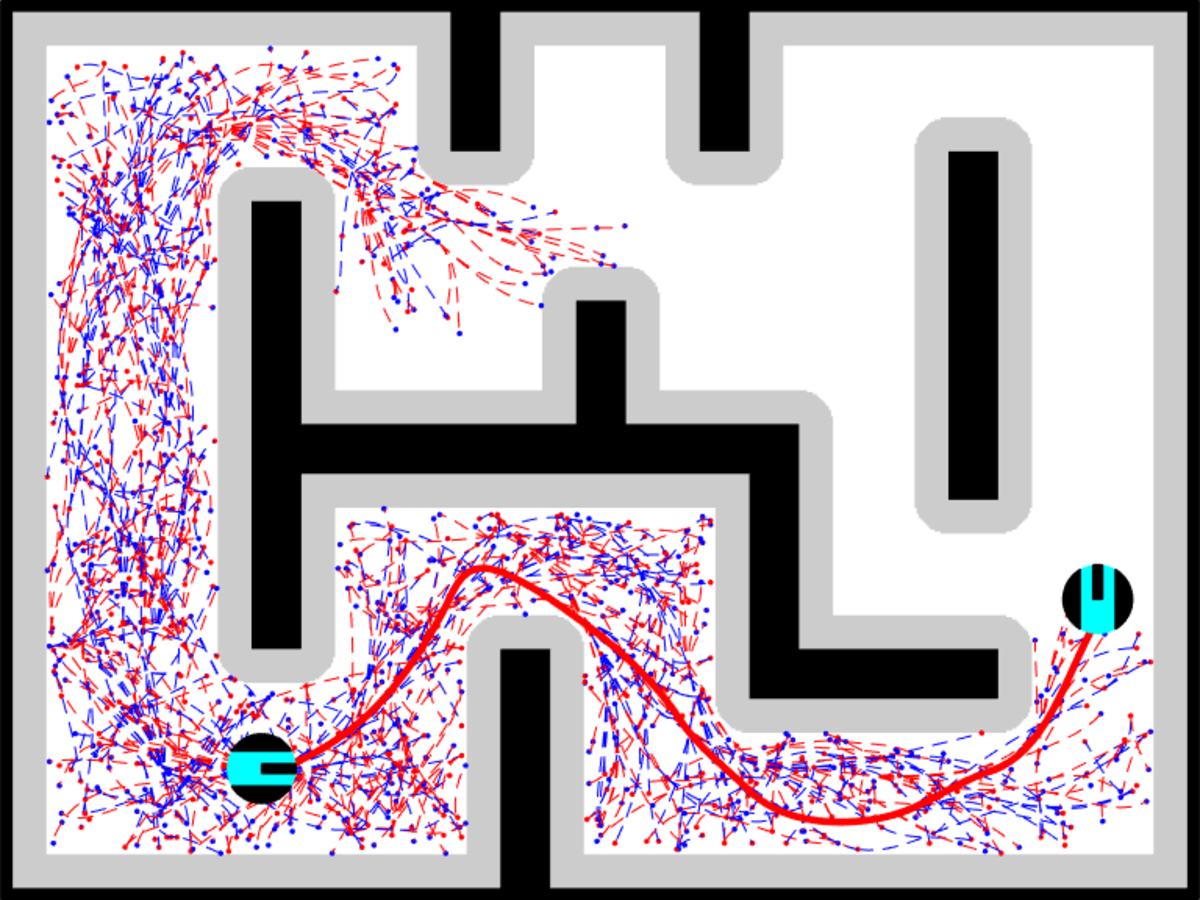} &
\includegraphics[width = 0.33\columnwidth]{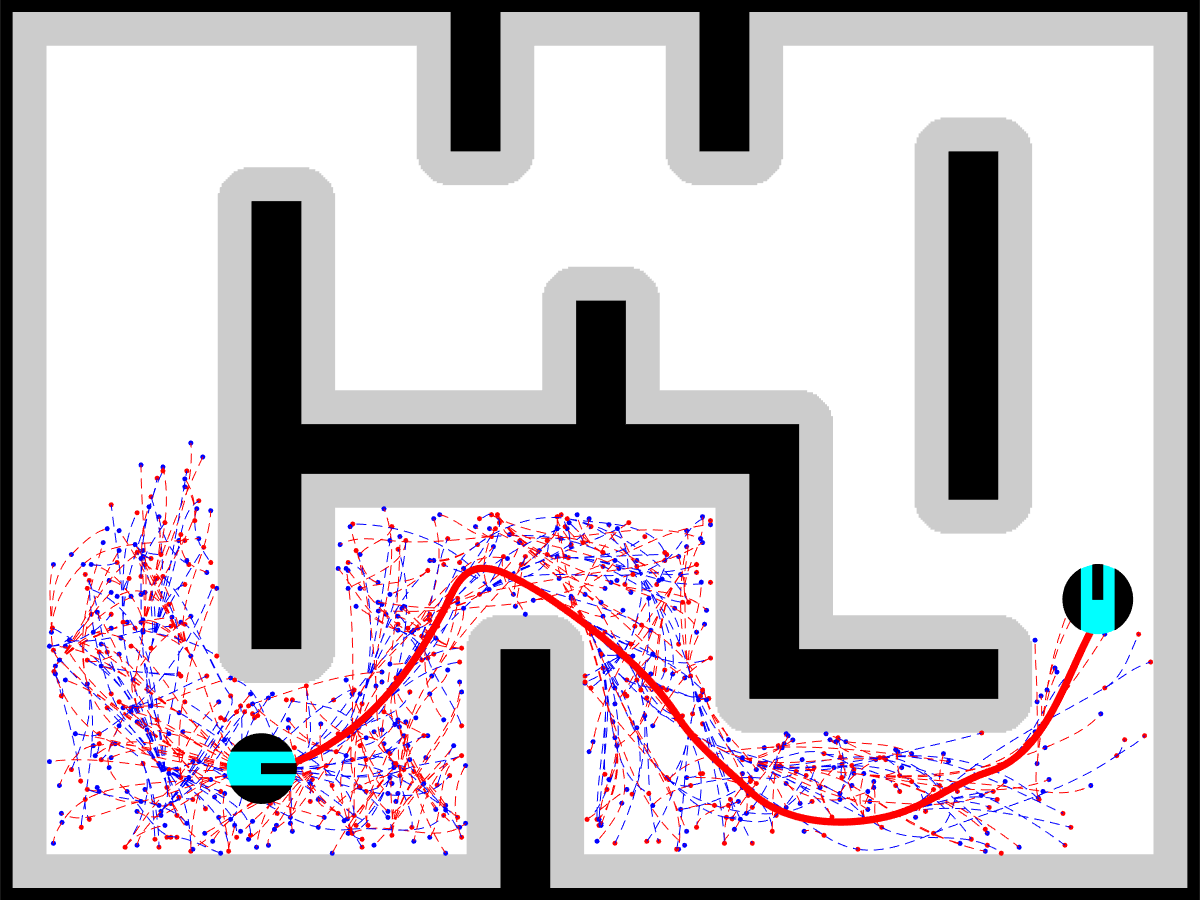}
\end{tabular}
\vspace{-3mm}
\caption{Informed sampling and pruning with heuristics enable effective and efficient exploration of complex environments in sampling-based optimal unicycle feedback motion planning.
(Left) Planning without informed sampling, (middle) planning with informed sampling and zero heuristic, (right) planning with informed sampling and pruning using a Euclidean heuristic.}
\label{fig.informed_motion_planning}
\vspace{-3mm}
\end{figure}

\section{Conclusions}
\label{sec.conclusions}

In this paper, we present a simple and intuitive geometric dual-headway unicycle pose control approach to asymptotically steer a unicycle robot to almost any given position and orientation by asymptotically bringing the adaptively placed headway point of the unicycle pose to the tailway point of the goal pose using feedback linearization.
By systematic analysis of geometric properties of dual-headway unicycle pose control, we build a positively inclusive feedback motion prediction bound on the closed-loop unicycle motion trajectory for safety verification, represented as the convex hull of the unicycle position, goal position, and their respective headway and tailway points. 
We introduce dual-headway translation and orientation distances to accurately measure travel and turning effort. 
The dual-headway translation measure is the travel distance from the start to the goal via headway and tailway points, which serves as an upper bound on the actual travel distance. 
The ratio of the dual-headway translation distance to the Euclidean distance, a lower bound on the actual travel distance, gives rise to the dual-headway orientation distance.
We demonstrate an example application of dual-headway unicycle pose control, motion prediction, and pose distance measures for optimal sampling-based feedback motion planing for minimizing total travel and turning effort.
In numerical simulations, we observe that tightly integrating planning and control via the dual-headway approach results in smoother and more effective motion patterns, compared to the standard decoupled Euclidean translation and cosine orientation distance measures.           

We are currently working on generalizing the tailway control approach for unicycles and nonholonomic systems (e.g., drones) in general, combined with geometric nonlinear control design methods, for path-following and tracking control. 
Another promising research direction is unicycle pose projection using dual-headway distances for effective local steering in sampling-based motion planning, as well as for measuring unicycle travel costs over semantic costmaps.





%

%


\bibliographystyle{IEEEtran}
\bibliography{references}


\appendices 

\section{Proofs}

\subsection{Proof of \reflem{lem.headway_tailway_distance}}
\label{app.lem.headway_tailway_distance}
\begin{proof}
The results follows from the triangle inequality (i.e., $\norm{\vect{a} + \vect{b}} \leq \norm{\vect{a}} + \norm{\vect{b}}$ for any $\vect{a}, \vect{b} \in \R^{n}$) as
\begin{align*}
\norm{\pos - \goalpos} &= \nlist{\headpos - \goaltailpos - \norm{\pos - \goalpos}\!\plist{\!\headcoef \ovectsmall{\ori} \!+\! \goaltailcoef \ovectsmall{\goalori}\!}\!}
\\
& \leq \norm{\headpos - \goaltailpos} + (\headcoef + \goaltailcoef)\norm{\pos - \goalpos}
\end{align*}
which  is equivalent, for $\headcoef + \goaltailcoef < 1$, to  
\begin{align*}
\hspace{20mm} \norm{\pos - \goalpos} \leq \tfrac{1}{1 - \headcoef - \goaltailcoef} \norm{\headpos - \goaltailpos}.  \hspace{20mm} \qedhere
\end{align*}
\end{proof}

\subsection{Proof of \refprop{prop.headway_tailway_distance_decay}}
\label{app.headway_tailway_distance_decay}

\begin{proof}
The rate of change of $\norm{\headpos - \goaltailpos}$ can be bounded as 
{\small
\begin{align*}
\frac{\diff}{\diff t}\norm{\headpos - \goaltailpos}^2 &= 2 \tr{\plist{\headpos - \goaltailpos}}\plist{\headposdot - \goaltailposdot}
\\
& \leq  -2 \refcoef \norm{\headpos - \goaltailpos}^2 \plist{\!1 - \frac{\goaltailcoef \absval{\tfrac{\tr{\goalpos - \pos}}{\norm{\goalpos - \pos}}\ovectsmall{\ori}}}{1 - \headcoef \tfrac{\tr{\goalpos - \pos}}{\norm{\goalpos - \pos}}\ovectsmall{\ori} }\!} 
\end{align*}
}%
as the changes of the headway and tailway points satisfy
{\footnotesize
\begin{align*}
\tr{\plist{\headpos - \goaltailpos}} \headposdot &= -\refcoef\norm{\headpos - \goaltailpos}^2,
\\
\tr{\plist{\headpos - \goaltailpos}} \goaltailposdot & = \goaltailcoef \tfrac{\tr{\plist{\goalpos - \pos}}}{\norm{\goalpos - \pos}}\ovectsmall{\ori}\tr{\plist{\headpos - \goaltailpos}} \ovectsmall{\goalori}\linvel_{\goalpos, \goalori}(\pos, \ori) 
\\
& \hspace{-15mm} = - \refcoef\tr{\plist{\headpos - \goaltailpos}}\ovectsmall{\ori}\tr{\plist{\headpos - \goaltailpos}} \ovectsmall{\goalori} \frac{\goaltailcoef \tfrac{\tr{\plist{\goalpos - \pos}}}{\norm{\goalpos - \pos}}\ovectsmall{\ori}}{1\! -\! \headcoef \tfrac{\tr{\plist{\goalpos - \pos}}}{\norm{\goalpos - \pos}}\ovectsmall{\ori} }
\\
& \hspace{-15mm}\geq -\refcoef \norm{\headpos - \goaltailpos}^2 \frac{\tailcoef \absval{\tfrac{\tr{\goalpos - \pos}}{\norm{\goalpos - \pos}}\ovectsmall{\ori}}}{1 - \headcoef \tfrac{\tr{\plist{\goalpos - \pos}}}{\norm{\goalpos - \pos}}\ovectsmall{\ori} } > -\refcoef\norm{\headpos - \goaltailpos}^2
\end{align*}
}%
where \mbox{$\goaltailcoef \absval{\!\tfrac{\tr{\plist{\goalpos \! - \pos}\!}\!}{\norm{\goalpos \!- \pos}}\!\ovectsmall{\ori}\!} \! <\! 1 \!- \headcoef \tfrac{\tr{\plist{\goalpos \!- \pos}\!}\!}{\norm{\goalpos \!- \pos}}\!\ovectsmall{\ori}\!$ for $\headcoef \!+ \goaltailcoef \!\!<\! 1$.} 
\end{proof}

\subsection{Proof of \refprop{prop.forward_motion_in_finite_time}}
\label{app.forward_motion_in_finite_time}

\begin{proof}
 The linear velocity control $\linvel_{\goalpos, \goalori}(\pos, \ori)$  in \refeq{eq.dual_headway_control} satisfies 
{\small
\begin{align*}
\tr{(\goaltailpos \! - \headpos)} \ovectsmall{\ori} \! \leq\! 0 \Longrightarrow \linvel_{\goalpos, \goalori} (\pos, \ori) \! \leq\! \refcoef\tr{(\goaltailpos \! - \headpos)\!} \!\ovectsmall{\ori} \! \leq 0.
\end{align*} 
}%
for $\headcoef < 1$. 
Hence, when \mbox{$-1 < \frac{\tr{(\goaltailpos \! - \headpos)}}{\norm{\goaltailpos - \headpos}} \ovectsmall{\ori} \leq 0$}, the time rate of change of $\tr{(\goaltailpos \! - \!\headpos)\!} \ovectsmall{\ori}$ is always strictly positive and lower bounded as 
{\footnotesize
\begin{align*}
\tfrac{\diff}{\diff t} \tr{(\goaltailpos \! - \headpos)\!}\! \ovectsmall{\ori} &= \tr{(\goaltailposdot \! - \headposdot)\!} \!\ovectsmall{\ori} \!+\! \tr{(\goaltailpos \! - \headpos)\!} \! \nvectsmall{\ori} \angvel_{\goalpos, \goalori}(\pos, \ori)
\\
& \hspace{-15mm}= - \linvel_{\goalpos, \goalori}(\pos, \ori)\! \underbrace{\plist{\!1\! + \!\tfrac{\tr{(\pos - \goalpos)\!}}{\norm{\pos - \goalpos}} \! \ovectsmall{\ori} \plist{\! \headcoef \! + \goaltailcoef \tr{\ovectsmall{\goalori}}\ovectsmall{\ori}}\!\!}\!}_{\geq (1 - \headcoef - \goaltailcoef) > 0}
\\
& \hspace{10mm}+ \tfrac{\refcoef \plist{\!\tr{(\goaltailpos \! - \headpos)\!} \! \nvectsmall{\ori}}^{2}}{\headcoef \norm{\pos - \goalpos}}
\\
&\hspace{-20mm} \geq - \refcoef(1-\headcoef -\goaltailcoef)  \tr{(\goaltailpos \! - \headpos)\!}\! \ovectsmall{\ori} + \underbrace{\tfrac{\refcoef \plist{\!\tr{(\goaltailpos \! - \headpos)\!}\! \nvectsmall{\ori}}^{2}}{\headcoef \norm{\pos - \goalpos}}}_{>0}
\\
& \hspace{-20mm} > - \refcoef(1-\headcoef -\goaltailcoef)  \tr{(\goaltailpos \! - \headpos)} \ovectsmall{\ori}.
\end{align*}
}%
Since its rate of increase is strictly faster than that of an exponentially increasing function to zero, due to the comparison lemma \cite{khalil_NonlinearSystems2001}, $\tr{(\goaltailpos \! - \! \headpos)\!}\! \ovectsmall{\ori}$ becomes positive in finite time, which corresponds to positive linear velocity and forward motion. 
\end{proof}

\subsection{Proof of \refprop{prop.persistent_forward_motion}}
\label{app.prop.persistent_forward_motion}

\begin{proof}
Since $1 + \headcoef \frac{\tr{(\pos - \goalpos)}}{\norm{\pos - \goalpos}}\ovectsmall{\ori} > 0$ for $\headcoef > 1$ and $\pos \neq \goalpos$, the dual-headway unicycle linear velocity input satisfies
\begin{align*}
\linvel_{\goalpos, \goalori}(\pos, \ori) \geq 0 \Longleftrightarrow \tr{(\goaltailpos - \headpos)} \ovectsmall{\ori} \geq 0
\end{align*}
where $\linvel_{\goalpos, \goalori}(\pos, \ori) = 0$ if and only if $\tr{(\goaltailpos - \headpos)} \ovectsmall{\ori} = 0$.
Accordingly, one can observe for any $\forall \pos \neq \goalpos$ that 
{\small
\begin{align*}
& \hspace{-8mm}\frac{\diff}{\diff t}\tr{\plist{\goaltailpos \!-\! \headpos}\!}\!\ovectsmall{\ori} \Big |_{\tr{\plist{\goaltailpos - \headpos}\!}\!\ovectsmall{\ori} = 0} \nonumber \\
& = \!\! \!\!\underbrace{\tr{\plist{\goaltailposdot \!-\! \headposdot}\!} \! \ovectsmall{\ori}}_{= 0 \text{  since $\linvel_{\goalpos\!, \goalori\!}(\pos, \ori) \!= \!0$}} \!\!\!\!\! + \tr{\plist{\goaltailpos \!-\! \headpos}\!}\!\nvectsmall{\ori}\angvel_{\goalpos, \goalori}(\pos, \ori) \!\!\!\!\!
\\
&=   \frac{\refcoef \plist{\!\!\tr{\plist{\goaltailpos \!-\! \headpos}\!}\!\nvectsmall{\ori}\!}^2\!}{\headcoef \norm{\goalpos \!- \pos}} =\frac{\refcoef\norm{\goaltailpos \!- \headpos}^2}{\headcoef \norm{\goalpos\! - \pos}} > 0  \!\!\!\!\!
\end{align*}
}%
where the last equality and so the result follow from that $\tr{(\goaltailpos \!-\! \headpos)\!}\! \ovectsmall{\ori} \!=\! 0$ implies $\plist{\!\!\tr{\plist{\goaltailpos \!-\! \headpos}\!}\!\nvectsmall{\ori}\!}^{\!2} \!\!\!=\! \norm{\goaltailpos \!-\! \headpos}^2$.
\end{proof}

\subsection{Proof of \refprop{prop.distance_to_goal_decay}}
\label{app.distance_to_goal_decay}

\begin{proof}
For $\goaltailcoef \leq \headcoef < 1$, forward unicycle motion under the dual-headway motion control in \refeq{eq.dual_headway_control} implies forward orientation alignment towards the goal position, as follows:
{\small
\begin{align*}
& \linvel_{\goalpos, \goalori}(\pos, \ori) > 0 \Longleftrightarrow \tr{(\goaltailpos \!-\! \headpos)\!}\!\ovectsmall{\ori} > 0
\\
& \quad \Longrightarrow \tr{(\goalpos \!- \!\pos)\!}\!\ovectsmall{\ori} > \norm{\goalpos\! - \pos} \plist{\!\headcoef\! + \goaltailcoef \tr{\ovectsmall{\goalori}\!}\!\ovectsmall{\ori}\!}\! \geq 0
\end{align*}
}%
which decreases the unicycle position distance to the goal as 
\begin{align*}
\frac{\diff}{\diff t} \norm{\pos \!-\! \goalpos}^2 = - 2 \tr{(\pos\!-\! \goalpos)\!}\ovectsmall{\ori}   \linvel_{\goalpos, \goalori}(\pos, \ori) < 0  
\end{align*}
away from the goal, i.e., $\pos \neq \goalpos$. Thus, the result holds.
\end{proof}

\subsection{Proof of \refprop{prop.goal_alignment_under_forward_motion}}
\label{app.goal_alignment_under_forward_motion}

\begin{proof}
One can verify that 
{\small
\begin{align*}
\tfrac{\diff}{\diff t} \tfrac{\tr{(\goaltailpos - \headpos)}}{\norm{\goaltailpos - \headpos}} \ovectsmall{\goalori} & = \alpha  \tr{\ovectsmall{\goalori}}\plist{\mat{I} -  \tfrac{(\goaltailpos - \headpos)}{\norm{\goaltailpos - \headpos}} \tfrac{\tr{(\goaltailpos - \headpos)}}{\norm{\goaltailpos - \headpos}}} \ovectsmall{\goalori}
\\
\tfrac{\diff}{\diff t} \tfrac{\tr{(\goaltailpos - \headpos)}}{\norm{\goaltailpos - \headpos}} \ovectsmall{\ori} & = \alpha \tr{\ovectsmall{\goalori}}\plist{\mat{I} -  \tfrac{(\goaltailpos - \headpos)}{\norm{\goaltailpos - \headpos}} \tfrac{\tr{(\goaltailpos - \headpos)}}{\norm{\goaltailpos - \headpos}}} \ovectsmall{\ori}
\\
 & \quad + \beta \tr{\ovectsmall{\ori}}\plist{\mat{I} -  \tfrac{(\goaltailpos - \headpos)}{\norm{\goaltailpos - \headpos}} \tfrac{\tr{(\goaltailpos - \headpos)}}{\norm{\goaltailpos - \headpos}}} \ovectsmall{\ori}
\end{align*}
}%
where
{\small
\begin{align*}
\alpha &= \refcoef  \underbrace{\tfrac{\tr{(\goaltailpos - \headpos)}}{\norm{\tailpos - \headpos}} \ovectsmall{\ori}}_{\in [0,1]} \underbrace{\tfrac{\tailcoef \tfrac{\tr{(\goalpos - \pos)}}{\norm{\goalpos - \pos}}  \ovectsmall{\ori}}{1 - \headcoef \tfrac{\tr{(\goalpos - \pos)}}{\norm{\goalpos - \pos}}\ovectsmall{\ori}}}_{\in [0,1)} \in [0, \refcoef)
\\
\beta &= \tfrac{\refcoef}{\headcoef} \underbrace{\tfrac{\norm{\goaltailpos - \headpos}}{\norm{\goalpos - \pos}}}_{\geq 1 - \headcoef - \goaltailcoef}  \geq \refcoef \tfrac{1 - \headcoef - \goaltailcoef}{\headcoef} \geq \refcoef
\end{align*}
}%
for $\tr{(\goaltailpos - \headpos)} \ovectsmall{\ori} \geq 0$ and $2\headcoef + \goaltailcoef < 1$, where the lower bound on $\beta$ is due to \reflem{lem.headway_tailway_distance}.
Therefore, it follows from their quadratic form and $0\leq \alpha < \beta$ that we have for  $\tfrac{\tr{\plist{\goaltailpos - \headpos}}}{\norm{\goaltailpos - \headpos}} \!\ovectsmall{\ori} \geq 0$, $ \tfrac{\tr{\plist{\goaltailpos - \headpos}}}{\norm{\goaltailpos - \headpos}} \!\ovectsmall{\goalori} \neq \pm 1$, and $2\headcoef + \goaltailcoef \!<\! 1$ that
{\small
\begin{align*}
\tfrac{\diff}{\diff t} \tfrac{\tr{\plist{\goaltailpos - \headpos}}}{\norm{\goaltailpos - \headpos}} \ovectsmall{\goalori} > 0, \quad  \tfrac{\diff}{\diff t} \tfrac{\tr{\plist{\goaltailpos - \headpos}}}{\norm{\goaltailpos - \headpos}} \plist{\!\ovectsmall{\ori} \!+\! \ovectsmall{\goalori}\!}\! > 0
\end{align*}
}%
which completes the proof.
\end{proof}

\subsection{Proof of \refprop{prop.forward_unicycle_motion_prediction}}
\label{app.forward_unicycle_motion_prediction}
\begin{proof}
For any unicycle pose $(\pos, \ori) \in \fwddomain_{\goalpos, \goalori}$, the forward dual-headway controller ensures persistent forward motion (\refprop{prop.persistent_forward_motion}) and decreases the Euclidean distance to the goal (\refprop{prop.distance_to_goal_decay}). 
Therefore, the tailway point $\goaltailpos$ continuously gets closer to the goal position $\goalpos$ as 
\begin{align*}
\goaltailpos(t') \in \blist{\goaltailpos(t'), \goalpos} \subseteq  \blist{\goaltailpos(t), \goalpos} \quad \forall t' \geq t.
\end{align*}  
Accordingly, since the headway point $\headpos$ continuously moves straight towards the tailway point $\goaltailpos$, as imposed by the headway reference dynamics \refeq{eq.headway_reference_dynamics}, the headway trajectory $\headpos(t)$ can also be bounded for $t' \geq t$ as
\begin{align*}
\headpos(t') \in \conv(\headpos(t'), \goaltailpos(t'), \goalpos) \subseteq  \conv(\headpos(t), \goaltailpos(t), \goalpos), 
\end{align*} 
where the positive invariance of the motion bound follows from Nagumo's theorem of the subtangentiality condition, where the headway velocity $\headposdot$ always points inside or is tangent to the bounding motion set $\conv\plist{\headpos, \tailpos, \goalpos}$ \cite{blanchini_Automatica1999}.  
Similarly, from Nagumo's theorem of the subtangentiality condition of positive invariance, the unicycle position trajectory $\pos(t)$ is bounded for $t' \geq t$ as
{\small
\begin{align*}
\pos(t') \! \in\! \conv(\pos(t'), \headpos(t'), \goaltailpos\!(t'), \goalpos\!) \! \subseteq \! \conv(\pos(t), \headpos(t), \goaltailpos\!(t), \goalpos\!), 
\end{align*}
}%
since the unicycle moves forward towards its headway point $\headpos$ and the uncycle velocity $\posdot$ is inside or tangent to the motion bound set $\conv(\pos, \headpos, \goaltailpos, \goalpos)$. 

The circular trajectory bound $\pos(t') \!\in \! \ball(\goalpos, \norm{\pos(t') - \goalpos}) \subseteq \ball(\goalpos, \norm{\pos(t) - \goalpos})$ is simply due to the decreasing Euclidean distance to the goal under forward motion (\refprop{prop.distance_to_goal_decay}). Thus, the result follows.
\end{proof}

\end{document}

%% file: packages.tex
\usepackage{fancyhdr}

\usepackage{xcolor}
\usepackage{comment}

\usepackage{cite}

\usepackage[mathcal]{euscript}

\usepackage[cmex10]{amsmath}
\usepackage{amssymb}

\usepackage{amsthm} 
  
\usepackage{dsfont} 

\usepackage{epsfig} 

\usepackage{multirow}
\usepackage{enumerate}

\usepackage[ruled, vlined, linesnumbered]{algorithm2e}
\SetAlCapFnt{\small}
\SetNlSty{}{}{} 
\SetInd{0.5em}{0.2em}
\SetAlgoSkip{9.5pt}

%% file: commands.tex
\newtheoremstyle{plain}
	  {}
	  {}
	  {\itshape}
	  {}
	  {\bfseries}
	  {}
	  {5pt plus 1pt minus 1pt}
	  {}

\newtheoremstyle{definition}
  	  {}
	  {}
	  {\normalfont}
	  {}
	  {\bfseries}
	  {}
	  {5pt plus 1pt minus 1pt}
	  {}
  
\theoremstyle{plain}

\newtheorem{lemma}{Lemma}
\newtheorem{proposition}{Proposition}

\theoremstyle{definition}

\newcommand{\refeq}[1]			{(\ref{#1})} 
\newcommand{\reffig}[1]			{Fig. \ref{#1}} 
\newcommand{\refsec}[1]			{Section \ref{#1}}
\newcommand{\refapp}[1]			{Appendix \ref{#1}}
\newcommand{\reftab}[1]			{Table \ref{#1}}

\newcommand{\refprop}[1]		{Proposition \ref{#1}}

\newcommand{\reflem}[1]			{Lemma \ref{#1}}

\newcommand{\refalg}[1]			{Algorithm \ref{#1}}

%% file: notation.tex
\newcommand{\unisample}            	{\mathrm{sample}} 
\newcommand{\unipose}{\mathrm{p}}
\newcommand{\uniposestart}{\mathrm{p}_{\mathrm{start}}}
\newcommand{\uniposegoal}{\mathrm{p}_{\mathrm{goal}}}
\newcommand{\uniposerand}{\mathrm{p}_{\mathrm{rand}}}

\newcommand{\uniposeneighbor}{\mathrm{P}_{\mathrm{near}}}
\newcommand{\uniposenear}{\mathrm{p}_{\mathrm{near}}}
\newcommand{\uniposenearest}{\mathrm{p}_{\mathrm{best}}}
\newcommand{\uniposenew}{\mathrm{p}_{\mathrm{new}}}
\newcommand{\uniposeparent}{\mathrm{p}_{\mathrm{parent}}}
\newcommand{\uniposemin}{\mathrm{p}_{\mathrm{min}}}
\newcommand{\parent}{\mathrm{parent}}

\newcommand{\issafe}{\mathrm{issafe}}

\newcommand{\NumSample} 	{N} 
\newcommand{\graph}           {G}
\newcommand{\edgeset}       {E}
\newcommand{\vertexset}     {V}

\newcommand{\uninearest}{\mathrm{nearest}}
\newcommand{\unineighbor}{\mathrm{neighbor}}
\newcommand{\unidist}	{\mathrm{unidist}}
\newcommand{\translate}{\mathrm{trans}}
\newcommand{\orient}{\mathrm{orient}}
\newcommand{\euclidean} {\mathrm{euc}}
\newcommand{\cosine}    {\mathrm{cos}}
\newcommand{\euclideancosine}{\mathrm{euccos}}
\newcommand{\headtail}  {\mathrm{headtail}}
\newcommand{\headtailori} {\overline{\mathrm{headtail}}}
\newcommand{\dualhead}  {\mathrm{dualhead}}
\newcommand{\dualheadori}{\overline{\mathrm{dualhead}}}
\newcommand{\uniproj}    {\mathrm{project}}
\newcommand{\unisampleset}{\mathcal{U}}
\newcommand{\cost}{\mathrm{cost}}
\newcommand{\localcost} {\mathrm{localcost}}

\newcommand{\neighbor}{N}

\newcommand{\R}  	{\mathbb{R}} 

\newcommand{\radius} 	{\rho}

\newcommand{\conv}      {\mathrm{conv}} 

\newcommand{\ovect}[1] {\begin{bmatrix}\cos #1\\ \sin #1 \end{bmatrix}}
\newcommand{\ovectsmall}[1] {\scalebox{0.8}{$\ovect{#1}$}}

\newcommand{\nvect}[1] {\begin{bmatrix}-\sin #1\\ \cos #1 \end{bmatrix}}
\newcommand{\nvectsmall}[1] {\scalebox{0.8}{$\nvect{#1}$}}
\newcommand{\nvecT}[1] {\tr{\nvect{#1}\!}}
\newcommand{\nvecTsmall}[1] {\scalebox{0.85}{$\nvecT{#1}$}}

\newcommand{\freespace}	{\mathcal{F}} 
\newcommand{\workspace}	{\mathcal{W}} 
\newcommand{\obstspace}	{\mathcal{O}} 
\newcommand{\ball}      {\mathrm{B}} 

\newcommand{\pos}{\vect{x}} 
\newcommand{\ori}{\theta} 
\newcommand{\goal}[1]{{#1}^{*}} 
\newcommand{\goalpos}{\vect{x}^*} 
\newcommand{\goalori}{\theta^*} 

\newcommand{\head}[1]{{#1}_{\mathrm{h}}} 
\newcommand{\tail}[1]{{#1}_{\mathrm{t}}} 
\newcommand{\headpos}{\pos_{\mathrm{h}}} 
\newcommand{\tailpos}{\pos_{\mathrm{t}}} 
\newcommand{\goalheadpos}{\goalpos_{\mathrm{h}}} 
\newcommand{\goaltailpos}{\goalpos_{\mathrm{t}}} 
\newcommand{\posdot}{\dot{\pos}} 
\newcommand{\headposdot}{\head{\dot{\pos}}} 
\newcommand{\tailposdot}{\tail{\dot{\pos}}} 
\newcommand{\goalheadposdot}{\head{\goal{\dot{\pos}}}} 
\newcommand{\goaltailposdot}{\goal{\tail{\dot{\pos}}}} 
\newcommand{\headcoef}{\kappa_{\mathrm{h}}} 
\newcommand{\tailcoef}{\kappa_{\mathrm{t}}} 
\newcommand{\goalheadcoef}{\kappa_{\mathrm{h}}^{*}} 
\newcommand{\goaltailcoef}{\kappa_{\mathrm{t}}^{*}} 
\newcommand{\headtailcoef}{\kappa} 
\newcommand{\refcoef}{\kappa_{\mathrm{r}}} 
\newcommand{\deltapos}{\delta \pos}
\newcommand{\deltaori}{\delta \ori}
\newcommand{\deltaradius}{\delta r}
\newcommand{\Deltapos}{\Delta \pos}
\newcommand{\Deltaori}{\Delta \ori}
\newcommand{\Deltaradius}{\Delta r}

\newcommand{\linvel}     {v}  
\newcommand{\angvel}     {\omega} 
\newcommand{\fwdlinvel}  {\overrightarrow{\linvel}} 
\newcommand{\fwdangvel}  {\overrightarrow{{\angvel}}} 
\newcommand{\bwdlinvel}  {\overleftarrow{\linvel}} 
\newcommand{\bwdangvel}  {\overleftarrow{{\angvel}}} 
\newcommand{\fwddomain}  {\overrightarrow{\mathcal{D}}} 
\newcommand{\bwddomain}  {\overleftarrow{\mathcal{D}}} 

\newcommand{\fwdctrl}       {\overrightarrow{\vect{u}}} 			
\newcommand{\bwdctrl}       {\overleftarrow{\vect{u}}} 			

\newcommand{\ctrl}      {\vect{u}} 			

\let\originalleft\left
\let\originalright\right
\renewcommand{\left}{\mathopen{}\mathclose\bgroup\originalleft}
\renewcommand{\right}{\aftergroup\egroup\originalright}
	
\newcommand{\plist}[1] 	{\left(#1\right)} 
\newcommand{\blist}[1]	{\left[ #1 \right]} 
\newcommand{\clist}[1]	{\left\{#1\right\}} 
\newcommand{\nlist}[1]   {\left\|#1\right\|} 

\newcommand{\vect}[1]   {\mathrm{#1}}
\newcommand{\mat}[1]    {\mathbf{#1}}
\newcommand{\tr}[1] {{#1}^{\mathrm{T}}} 
\newcommand{\norm}[1]  {\|#1\|}
\newcommand{\absval}[1]{\left|#1 \right|} 

\newcommand{\argmin}{\operatornamewithlimits{arg\ min}} 
\newcommand{\diff} {\mathrm{d}}

\makeatletter
\DeclareRobustCommand\widecheck[1]{{\mathpalette\@widecheck{#1}}}
\def\@widecheck#1#2{%
    \setbox\z@\hbox{\m@th$#1#2$}%
    \setbox\tw@\hbox{\m@th$#1%
       \widehat{%
          \vrule\@width\z@\@height\ht\z@
          \vrule\@height\z@\@width\wd\z@}$}%
    \dp\tw@-\ht\z@
    \@tempdima\ht\z@ \advance\@tempdima2\ht\tw@ \divide\@tempdima\thr@@
    \setbox\tw@\hbox{%
       \raise\@tempdima\hbox{\scalebox{1}[-1]{\lower\@tempdima\box
\tw@}}}%
    {\ooalign{\box\tw@ \cr \box\z@}}}
\makeatother

%% file: manuscript_report.bbl
\begin{thebibliography}{10}
\providecommand{\url}[1]{#1}
\csname url@rmstyle\endcsname
\providecommand{\newblock}{\relax}
\providecommand{\bibinfo}[2]{#2}
\providecommand\BIBentrySTDinterwordspacing{\spaceskip=0pt\relax}
\providecommand\BIBentryALTinterwordstretchfactor{4}
\providecommand\BIBentryALTinterwordspacing{\spaceskip=\fontdimen2\font plus
\BIBentryALTinterwordstretchfactor\fontdimen3\font minus
  \fontdimen4\font\relax}
\providecommand\BIBforeignlanguage[2]{{%
\expandafter\ifx\csname l@#1\endcsname\relax
\typeout{** WARNING: IEEEtran.bst: No hyphenation pattern has been}%
\typeout{** loaded for the language `#1'. Using the pattern for}%
\typeout{** the default language instead.}%
\else
\language=\csname l@#1\endcsname
\fi
#2}}

\bibitem{kim_etal_RAM2009}
M.~Kim, S.~Kim, S.~Park, M.-T. Choi, M.~Kim, and H.~Gomaa, ``Service robot for
  the elderly,'' \emph{IEEE Robotics \& Automation Magazine}, vol.~16, no.~1,
  pp. 34--45, 2009.

\bibitem{jones_MRA2006}
J.~L. Jones, ``Robots at the tipping point: the road to irobot roomba,''
  \emph{IEEE Robotics Automation Magazine}, vol.~13, no.~1, pp. 76--78, 2006.

\bibitem{ackerman_Spectrum2022}
E.~Ackerman, ``A robot for the worst job in the warehouse: Boston dynamics'
  stretch can move 800 heavy boxes per hour,'' \emph{IEEE Spectrum}, vol.~59,
  no.~1, pp. 50--51, 2022.

\bibitem{renan_nascimento_RAS2021}
{\'{I}}.~R. da~Costa~Barros and T.~P. Nascimento, ``Robotic mobile fulfillment
  systems: A survey on recent developments and research opportunities,''
  \emph{Robot. Auton. Syst.}, vol. 137, p. 103729, 2021.

\bibitem{gonzalez_perez_milanes_mashashibi_TITS2016}
D.~González, J.~Pérez, V.~Milanés, and F.~Nashashibi, ``A review of motion
  planning techniques for automated vehicles,'' \emph{IEEE Trans. Intell.
  Transp.}, vol.~17, no.~4, pp. 1135--1145, 2016.

\bibitem{paden_cap_yong_yershov_frazzoli_TIV2016}
B.~{Paden}, M.~{Čáp}, S.~{Yong}, D.~{Yershov}, and E.~{Frazzoli}, ``A survey
  of motion planning and control techniques for self-driving urban vehicles,''
  \emph{IEEE Trans. Intell. Veh.}, vol.~1, no.~1, pp. 33--55, 2016.

\bibitem{philippsen_siegwart_ICRA2003}
R.~Philippsen and R.~Siegwart, ``Smooth and efficient obstacle avoidance for a
  tour guide robot,'' in \emph{IEEE International Conference on Robotics and
  Automation}, vol.~1, 2003, pp. 446--451.

\bibitem{snape_etal_IROS2010}
J.~Snape, J.~van~den Berg, S.~J. Guy, and D.~Manocha, ``Smooth and
  collision-free navigation for multiple robots under differential-drive
  constraints,'' in \emph{IEEE/RSJ International Conference on Intelligent
  Robots and Systems}, 2010, pp. 4584--4589.

\bibitem{chakravarthy_debasish_TSM1998}
A.~Chakravarthy and D.~Ghose, ``Obstacle avoidance in a dynamic environment: a
  collision cone approach,'' \emph{IEEE Trans. Syst. Man Cybern. Part A},
  vol.~28, pp. 562--574, 1998.

\bibitem{fiorini_shiller_IJRR1998}
P.~Fiorini and Z.~Shiller, ``Motion planning in dynamic environments using
  velocity obstacles,'' \emph{The International Journal of Robotics Research},
  vol.~17, no.~7, pp. 760--772, 1998.

\bibitem{arslan_koditschek_IJRR2019}
{\"O}.~Arslan and D.~E. Koditschek, ``Sensor-based reactive navigation in
  unknown convex sphere worlds,'' \emph{The International Journal of Robotics
  Research}, vol.~38, no. 2-3, pp. 196--223, 2019.

\bibitem{isleyen_arslan_RAL2022}
A.~{\.I}{\c{s}}leyen, N.~van~de Wouw, and {\"O}.~Arslan, ``From low to high
  order motion planners: Safe robot navigation using motion prediction and
  reference governor,'' \emph{IEEE Robot. Autom. Lett.}, vol.~7, no.~4, pp.
  9715--9722, 2022.

\bibitem{brockett_DGCT1983}
R.~W. Brockett, ``Asymptotic stability and feedback stabilization,'' in
  \emph{Differential Geometric Control Theory}, 1983, pp. 181--191.

\bibitem{astolfi_JDSMC1999}
A.~Astolfi, ``Exponential stabilization of a wheeled mobile robot via
  discontinuous control,'' \emph{Journal of Dynamic Systems, Measurement, and
  Control}, vol. 121, no.~1, pp. 121--126, 1999.

\bibitem{lee_etal_IROS2000}
S.-O. Lee, Y.-J. Cho, M.~Hwang-Bo, B.-J. You, and S.-R. Oh, ``A stable
  target-tracking control for unicycle mobile robots,'' in \emph{IEEE/RSJ
  Inter. Conf. on Intelligent Robots and Systems}, 2000, pp. 1822--1827 vol.3.

\bibitem{dandrea-novel_campion_bastin_IJRR1995}
B.~d'Andr{\'e}a Novel, G.~Campion, and G.~Bastin, ``Control of nonholonomic
  wheeled mobile robots by state feedback linearization,'' \emph{The Inter.
  Journal of Robotics Research}, vol.~14, no.~6, pp. 543--559, 1995.

\bibitem{isleyen_vandewouw_arslan_IROS2023}
A.~{\.I}{\c{s}}leyen, N.~van~de Wouw, and {\"O}.~Arslan, ``Feedback motion
  prediction for safe unicycle robot navigation,'' in \emph{IEEE/RSJ Inter.
  Conf. on Intelligent Robots and Systems}, 2023, pp. 10\,511--10\,518.

\bibitem{isleyen_vandewouw_arslan_CDC2023}
------, ``Adaptive headway motion control and motion prediction for safe
  unicycle motion design,'' in \emph{IEEE Conference on Decision and Control},
  2023, pp. 6942--6949.

\bibitem{tarshahani_isleyen_arslan_ECC2024}
A.~Tarshahani, A.~{\.I}{\c{s}}leyen, and {\"O}.~Arslan, ``Total turning and
  motion range prediction for safe unicycle control,'' in \emph{European
  Control Conference}, 2024, pp. 2760--2767.

\bibitem{aicardi_etal_RAM1995}
M.~Aicardi, G.~Casalino, A.~Bicchi, and A.~Balestrino, ``Closed loop steering
  of unicycle like vehicles via {L}yapunov techniques,'' \emph{IEEE Robotics \&
  Automation Magazine}, vol.~2, no.~1, pp. 27--35, 1995.

\bibitem{murray_sastry_TAC1993}
R.~Murray and S.~Sastry, ``Nonholonomic motion planning: steering using
  sinusoids,'' \emph{IEEE Transactions on Automatic Control}, vol.~38, no.~5,
  pp. 700--716, 1993.

\bibitem{samson_IJRR1993}
C.~Samson, ``Time-varying feedback stabilization of car-like wheeled mobile
  robots,'' \emph{Int. J. Robot. Res.}, vol.~12, no.~1, pp. 55--64, 1993.

\bibitem{samson_TAC1995}
C.~{Samson}, ``Control of chained systems application to path following and
  time-varying point-stabilization of mobile robots,'' \emph{IEEE Transactions
  on Automatic Control}, vol.~40, no.~1, pp. 64--77, 1995.

\bibitem{canny_ComplexityMotionPlanning1988}
J.~Canny, \emph{The complexity of robot motion planning}.\hskip 1em plus 0.5em
  minus 0.4em\relax MIT Press, 1988.

\bibitem{lavalle_kuffner_IJRR2001}
S.~M. LaValle and J.~J. Kuffner, ``Randomized kinodynamic planning,''
  \emph{Int. J. Robot. Res.}, vol.~20, no.~5, pp. 378--400, 2001.

\bibitem{karaman_frazzoli_IJRR2011}
S.~Karaman and E.~Frazzoli, ``Sampling-based algorithms for optimal motion
  planning,'' \emph{The International Journal of Robotics Research}, vol.~30,
  no.~7, pp. 846--894, 2011.

\bibitem{kuwata_teo_fiore_karaman_frazzoli_how_TCST2009}
Y.~Kuwata, J.~Teo, G.~Fiore, S.~Karaman, E.~Frazzoli, and J.~P. How,
  ``Real-time motion planning with applications to autonomous urban driving,''
  \emph{IEEE Trans. Control Syst. Technol.}, vol.~17, no.~5, pp. 1105--1118,
  2009.

\bibitem{palmieri_arras_IROS2014}
L.~Palmieri and K.~O. Arras, ``A novel {RRT} extend function for efficient and
  smooth mobile robot motion planning,'' in \emph{IEEE/RSJ International
  Conference on Intelligent Robots and Systems}, 2014, pp. 205--211.

\bibitem{park_kuipers_IROS2015}
J.~J. Park and B.~Kuipers, ``Feedback motion planning via non-holonomic rrt*
  for mobile robots,'' in \emph{IEEE/RSJ International Conference on
  Intelligent Robots and Systems}, 2015, pp. 4035--4040.

\bibitem{arslan_berntorp_tsiotras_ICRA2017}
O.~Arslan, K.~Berntorp, and P.~Tsiotras, ``Sampling-based algorithms for
  optimal motion planning using closed-loop prediction,'' in \emph{IEEE Inter.
  Conf. on Robotics and Automation}, 2017, pp. 4991--4996.

\bibitem{danielson_berntorp_cairano_weiss_ACC2020}
C.~Danielson, K.~Berntorp, S.~D. Cairano, and A.~Weiss, ``Motion-planning for
  unicycles using the invariant-set motion-planner,'' in \emph{American Control
  Conference}, 2020, pp. 1235--1240.

\bibitem{khalil_NonlinearSystems2001}
H.~K. Khalil, \emph{Nonlinear Systems}.\hskip 1em plus 0.5em minus 0.4em\relax
  Prentice Hall, 2001.

\bibitem{arslan_tiemessen_TRO2022}
O.~Arslan and A.~Tiemessen, ``Adaptive b\'ezier degree reduction and splitting
  for computationally efficient motion planning,'' \emph{IEEE Transactions on
  Robotics}, vol.~38, no.~6, pp. 3655--3674, 2022.

\bibitem{arslan_pacelli_koditschek_IROS2017}
O.~Arslan, V.~Pacelli, and D.~E. Koditschek, ``Sensory steering for
  sampling-based motion planning,'' in \emph{IEEE/RSJ International Conference
  on Intelligent Robots and Systems}, 2017, pp. 3708--3715.

\bibitem{burridge_rizzi_koditschek_IJRR1999}
R.~R. Burridge, A.~A. Rizzi, and D.~E. Koditschek, ``Sequential composition of
  dynamically dexterous robot behaviors,'' \emph{The International Journal of
  Robotics Research}, vol.~18, no.~6, pp. 535--555, 1999.

\bibitem{gammell_srivinasa_barfoot_IROS2014}
J.~D. Gammell, S.~S. Srinivasa, and T.~D. Barfoot, ``Informed {RRT*}: Optimal
  sampling-based path planning focused via direct sampling of an admissible
  ellipsoidal heuristic,'' in \emph{IEEE/RSJ International Conference on
  Intelligent Robots and Systems}, 2014, pp. 2997--3004.

\bibitem{blanchini_Automatica1999}
F.~Blanchini, ``Set invariance in control,'' \emph{Automatica}, vol.~35,
  no.~11, pp. 1747 -- 1767, 1999.

\end{thebibliography}
